\documentclass{article}
     \PassOptionsToPackage{numbers, compress}{natbib}
     
 \usepackage[preprint]{neurips_2022}
\usepackage[utf8]{inputenc} 
\usepackage[T1]{fontenc}    
\usepackage{microtype}
\usepackage{graphicx}
\usepackage{subfigure}
\usepackage{setspace}
\usepackage{booktabs} 
\usepackage{amsthm}
\usepackage{amsmath}
\usepackage{xcolor,wrapfig,epsfig}
\usepackage{algorithm}
\usepackage{algpseudocode}
\usepackage{graphics}
\usepackage{url}
\usepackage[toc,page,header]{appendix}
\usepackage{minitoc}

\usepackage{hyperref}

\definecolor{mygrey}{rgb}{0.5,0.5,0.5}

\newcommand{\blue}[1]{{\color{blue}{#1}}}
\newcommand{\red}[1]{{\color{red}{#1}}}
\newcommand{\green}[1]{{\color{green}{#1}}}
\newcommand{\orange}[1]{{\color{orange}{#1}}}
\newcommand{\EE}{\mathbb{E}}

\newtheorem{lemma}{Lemma}

\newtheorem{defi}{Definition}

\newtheorem{theorem}{Theorem}

\newtheorem{assumption}{Assumption}

\usepackage{amssymb}
\newcommand\numberthis{\addtocounter{equation}{1}\tag{\theequation}}
 \newcommand{\FullTitle}{Lazy Queries Can Reduce Variance in Zeroth-order Optimization}
\algrenewcommand{\algorithmiccomment}[1]{\hskip0em$\triangleright$ #1}

\title{Lazy Queries Can Reduce Variance in\\ Zeroth-order Optimization}

 \author{%
\begin{minipage}[t]{0.33\textwidth}
\centering
\textbf{Quan Xiao}
\\[2pt]
\textnormal{
Department of ECSE\\
Rensselaer Polytechnic Institute\\ 
xiaoq5@rpi.edu}
\end{minipage}
\begin{minipage}[t]{0.33\textwidth}
\centering
Qing Ling 
\\[2pt]
\textnormal{School of Data and Comp Sci.\\
Sun Yat-Sen University\\ 
lingqing556@mail.sysu.edu.cn}
\end{minipage}
 \begin{minipage}[t]{0.33\textwidth}
\centering
\textbf{Tianyi Chen}
\\[2pt]
\textnormal{Department of ECSE\\
Rensselaer Polytechnic Institute\\ 
chentianyi19@gmail.com}
 \end{minipage}
}
\begin{document}

\maketitle
\doparttoc 
\faketableofcontents 

\begin{abstract}
A major challenge of applying zeroth-order (ZO) methods is the high query complexity, especially when  queries are costly. 
We propose a novel gradient estimation technique for ZO methods based on adaptive lazy queries that we term as LAZO. Different from the classic one-point or two-point gradient estimation methods, LAZO develops two alternative ways to check the usefulness of old queries from  previous iterations, and then adaptively reuses them to construct the low-variance gradient estimates. We rigorously establish that through judiciously reusing the old queries, LAZO can reduce the variance of stochastic gradient estimates so that it not only saves queries per iteration but also achieves the regret bound for the symmetric two-point method. We evaluate the numerical performance of LAZO, and demonstrate the low-variance property and the performance gain of LAZO in both regret and query complexity relative to several existing ZO methods. The idea of LAZO is general, and can be applied to other variants of ZO methods. 
\end{abstract}
\vspace{-0.2cm}
\section{Introduction}\label{intro}
\vspace{-0.1cm}
Zeroth-order (ZO) optimization (also known as gradient-free optimization, or
bandit optimization), is useful in complex tasks when the analytical forms of loss functions are not available, only permitting evaluations of function values but not gradients. ZO methods have already been applied to reinforcement learning \cite{salimans2017evolution}, adversarial machine learning \cite{chen2017zoo,kariyappa2021maze} and meta-learning \cite{ruan2019learning,song2019maml}.

However, a major challenge of applying ZO methods to benefit these practical problems is their high query complexity, especially when the queries are costly. 
To this end, we aim to develop a new ZO method that can inherit the merits of existing ones but also \emph{reduce the query complexity}. 

To illustrate our method, we consider the online convex optimization (OCO) setting \cite{zinkevich2003online,hazan2007logarithmic}. OCO can be viewed as a repeated game between a learner and the nature. Consider the time indexed by $t$. 
Per iteration $t$, a learner selects an action $x_t\in \mathcal{X}$ and subsequently the nature chooses a loss function $f_t$, through which the learner incurs a loss $f_{t}(x_t)$. We consider the setting that, after the decision $x_t$ is made, only the value of the loss function $f_t$ at $x_t$ is revealed; the gradient is unavailable. Our goal is to minimize the static regret 
\begin{equation}\label{eq.regret}
{\cal R}_T({\cal A})\triangleq\mathbb{E}\left[\sum_{t=0}^{T} f_{t}(x_t)-\min _{x \in \mathcal{X}} \sum_{t=0}^{T} f_{t}(x)\right]
\end{equation}
where ${\cal A}$ is an algorithm, {\small$\mathcal{X} \subseteq \mathbb{R}^{d}$} is a convex set, $\{f_t\}$ are convex functions and the expectation is taken over the  random queries. This is a commonly used performance measure of OCO, which measures the difference between the cumulative loss of online decisions generated by an algorithm ${\cal A}$ and the loss of applying the best fixed decision chosen in hindsight \cite{hazan2016}.

While the loss function (or its gradient) is unknown, ZO methods approximate the gradient through values of the loss function based on various
gradient estimation techniques. 
The work of \cite{flaxman2004online} has first applied a ZO method to the OCO problem by querying a \emph{single function value} at each time (so-called  one-point gradient estimate). 
Going beyond one-point ZO, \cite{agarwal2010optimal,duchi2015optimal,nesterov2017random,shamir2017optimal} leveraged multiple queries at each iteration (so-called multi-point gradient estimate) to achieve an enhanced performance (i.e., regret) relative to the one-point methods. 

However, there is an essential trade-off between \emph{the number of queries per iteration} and \emph{the number of iterations} needed to achieve a target accuracy. The target accuracy is often measured by the average regret per iteration. 
Intuitively, querying fewer points per iteration may increase the variance of gradient estimation, and thus boost the the number of iterations needed to achieve a target average regret \cite{agarwal2010optimal,shamir2017optimal}. Consequently, it may increase the query complexity that depends on both of the two quantities. 
In this context, a natural yet important question is 
\begin{center}
\vspace{-0.2cm}
\emph{Can we develop ZO methods that save the number of queries per iteration without sacrificing  regret?} 
\vspace{-0.2cm}
\end{center}
At the first glance, it seems counter-intuitive that such a method does exist, achieving the best of two worlds. Nevertheless, we will provide an affirmative answer in this paper. The key idea that we leverage is to adaptively use the delayed queries that we call \emph{lazy queries}.

\vspace{-0.2cm}
\subsection{Related works} 
\vspace{-0.1cm}
To put our work in context, we review prior contributions that we group in the three categories. 

\textbf{One-point ZO.} ZO methods based on one-point gradient estimation can be traced back to the control literature, where one of the early approaches is the simultaneous perturbation stochastic approximation method \cite{spall1997one}. In the context of OCO, algorithms with one-point feedback have been developed in \cite{flaxman2004online, kleinberg2004nearly}. Building upon this element and the effective one-point gradient estimation scheme, the multi-agent ZO has developed for a game-theoretic model in \cite{heliou2020gradient}. See also a recent survey \cite{slivkins2019introduction} and references therein. For all the aforementioned one-point ZO methods, the regret bound is still ${\cal O}(T^{\frac{3}{4}})$, which is much worse than the ${\cal O}(T^{\frac{1}{2}})$ regret bound of their full information counterparts \cite{zinkevich2003online}. 
The one-point ZO methods achieving the ${\cal O}(T^{\frac{1}{2}})$ regret are the kernel-based \cite{bubeck2017kernel} and ellipsoid method \cite{lattimore2021improved}, but they use rather sophisticated gradient estimation techniques and have the ${\cal O}(d^{9.5})$ and ${\cal O}(d^{4.5})$ dimension dependence, respectively, making them inefficient in practice. 

\textbf{ZO with delayed feedback.}
The recent literature on bandits with delayed feedback is also related to this work, e.g., \cite{joulani2013online,quanrud2015online,li2019bandit,thune2019nonstochastic,zimmert2020optimal,vernade2020linear,ito2020delay,manegueu2020stochastic}.
However, these methods \emph{passively} receives delayed feedback, in which delays generally sublinearly increase the regret, while our method \emph{actively} leverages delayed feedback and uses delay to save queries and reduce the regret. Prior-guided ZO \cite{cheng2021convergence} is also relevant, but they only consider the time-invariant case.   
The work most relevant to ours is the residual one-point ZO method \cite{zhang2020boosting}, which \emph{nonadaptively} augments the one-point query with a delayed query to construct a two-point gradient estimator. Theoretically, its regret order is still ${\cal O}(T^{\frac{3}{4}})$, which is the same as the vanilla one-point ZO method. Very recently, a control-theoretical ZO approach has been developed in \cite{chen2021improve} that significantly improves the dimension dependence in the regret of residual one-point ZO method \cite{zhang2020boosting} in the time-invariant case via high-pass and low-pass filters. However, the suboptimal dependence on the time $T$ still remains. The idea of using delayed queries has also been used in saving communication resources in distributed learning \cite{chen2018lag}, but the LAG algorithm and the analysis there are very different from those in the present paper.

\textbf{Multi-point ZO.} ZO methods based on two or multiple function value evaluations enjoy better ${\cal O}(T^{\frac{1}{2}})$ regret. They have been independently developed from both online learning and optimization communities \cite{agarwal2010optimal,duchi2012randomized, shamir2017optimal, nesterov2017random,duchi2015optimal,ghadimi2013stochastic}. Recent advances in this direction include  the sign-based ZO method \cite{liu2018signsgd}, the proximal gradient extension of two-point ZO method \cite{huang2019faster}, adaptive momentum ZO \cite{chen2019zo}, autoencoder-based ZO \cite{tu2019autozoom} and adaptive sampling ZO by leveraging the inherent sparsity of the gradient \cite{wang2018stochastic,golovin2019iclr,cai2020zeroth}. For solving finite-sum optimization problems, variance-reduced ZO methods have also been developed by using various state-of-the-art variance reduction techniques \cite{liu2018zeroth,ji2019improved}. Our work is relevant but complementing to these multi-point ZO works. In fact, they all suggest promising future directions by applying our lazy query-based gradient estimation technique to them.

\vspace{-0.2cm}
\subsection{Our contributions} 
\vspace{-0.1cm}
In this context, our paper puts forward a new ZO gradient estimation method that leverages an adaptive  condition to parsimoniously query the loss function at one or two points. Our contributions can be summarized as follows.
\vspace{-0.2cm}
\begin{itemize}\setlength\itemsep{-0.01em}
\item [\bf C1)] We propose a lazy query-based ZO gradient estimation method that we term LAZO. Different from the classic one- or two-point methods, LAZO is a hybrid version of one- and two-point methods. It adaptively reuses the old yet informative queries from previous iterations to construct low-variance gradient estimates. 
\item [\bf C2)] We apply LAZO to the stochastic gradient descent (SGD) algorithm and obtain a new SGD algorithm. We rigorously establish the regret of LAZO-based SGD. Surprisingly, through judiciously reusing old queries, LAZO can reduce the variance of gradient estimation so that it not only saves queries per iteration but also achieves the ${\cal O}(\sqrt{dT})$ regret bound. 
\item [\bf C3)] We evaluate the numerical perfomance of LAZO on various tasks and show that LAZO maintains low variance and has performance gain in terms of regret and query complexity relative to popular methods. We also provide the multi-point extension of LAZO. 
\end{itemize}

\vspace{-0.2cm}
\section{Lazy Query for Zeroth-Order SGD}\label{lazy}
\vspace{-0.1cm}
In this section, we present our new ZO method that adaptively queries the loss function at either one or two points in each iteration, which gives its name \textbf{LA}zy \textbf{Z}eroth-\textbf{O}rder gradient (\textbf{LAZO}) method.

\vspace{-0.1cm}
\subsection{Preliminaries}
\vspace{-0.1cm}
Before we present our new algorithm, we review some basics of OCO and ZO. 
In the full information case, where all information of the loss function including the gradient is available, the ``workhorse'' OCO algorithm is the online gradient descent method \cite{zinkevich2003online}, given by $x_{t+1}=\Pi_{\mathcal{X}}\left(x_{t}-\eta \nabla f_t(x_t)\right)$,
where $\Pi_{\mathcal{X}}$ denotes the projection on to $\mathcal{X}$ and $\eta>0$ the stepsize. In the partial information setting, we only have access to the value of the loss function $f_t$ at $x_t$, rather than the gradient. 

Generically speaking, ZO gradient-based methods first query the function values  at one or multiple perturbed points $\{x_t+\delta u_t\}$, where $u_t$ is the unit perturbation vector and $\delta>0$ is a small perturbation factor and following \cite{duchi2015optimal,shamir2017optimal}, we assume that one can query $f_t$ at any $x_t+\delta u_t$.\footnote{Otherwise, a standard technique in \cite{agarwal2010optimal,flaxman2004online} that running the algorithm on a smaller set $(1-\gamma)\mathcal{X}=\{(1-\gamma)x: x\in\mathcal{X}\}$ can be applied since choosing $\gamma$ sufficiently small guarantees $x_t+\delta u_t\in \mathcal{X}$.}; construct a stochastic gradient estimate $\tilde{g}_t(x_t)$ using these function values (see exact forms of $\tilde{g}_t(x_t)$ in Section \ref{obv}); and then plug it into the gradient iteration \cite{flaxman2004online}
\begin{equation}
x_{t+1}=\Pi_{\mathcal{X}}\left(x_{t}-\eta \tilde{g}_{t}(x_{t})\right)
\label{eq:sgd}
\end{equation}
Different from the most commonly used SGD, the stochastic gradient estimate in ZO gradient-based methods is usually biased in the sense that $    \mathbb{E}\left[\tilde{g}_t(x_t)\right]\neq \nabla f_t(x_t)$.

The rationale behind \eqref{eq:sgd} is that $\tilde{g}_t(x_t)$ is an unbiased estimator for the gradient of a smoothed version of $f_t$ at $x_t$. Specifically, defining a smoothed version of $f_t$ as $f_{\delta,t}(x)\triangleq\mathbb{E}_{v_t\sim U(\mathbb{B})}[f_t(x+\delta v_t)]$,
where $v_t\sim U(\mathbb{B})$ denotes the uniform sampling $v_t$ from the ball $\mathbb{B}=\left\{x \in \mathbb{R}^{d}|\| x \| \leq 1\right\}$, we have 
\begin{equation}\label{eq.unbias}
    \mathbb{E}\left[\tilde{g}_t(x)\right]=\nabla f_{\delta,t}(x)\approx \nabla f_t(x).
\end{equation}

In OCO, we make the following basic assumptions.  

\begin{assumption}[Lipschitz continuity] 
\label{as1}
For all $t$, $f_t(x)$ is $L_t$-Lipschitz continuous, i.e. $\forall x, y \in \mathcal{X}$, $ |f_t(x)-f_t(y)| \leq L_t \|x-y\|$.
Moreover, we define $L\triangleq\max_{t=0,\cdots, T} L_t$.
\end{assumption}

\begin{assumption}[Bounded set] 
\label{as2}
There exist constants $R$ such that $\mathcal{X}\subset R \mathbb{B}$.
\end{assumption}

\begin{assumption}[Convexity] 
\label{as3}
For all $t$, $f_t(x)$ is convex. 
\end{assumption}
Assumptions \ref{as1}--\ref{as3} are common in OCO with both full and partial information feedback \cite{hazan2007logarithmic, flaxman2004online,agarwal2010optimal,shamir2017optimal}.

\subsection{Observation: A delicate trade-off}\label{obv}
For the biased SGD iteration \eqref{eq:sgd}, the performance has been well-studied in literature. 
To present the connection between the regret and the quality of gradient estimation, we need the following lemma.\footnote{Lemma 1 is not new, one can see \cite{shamir2017optimal} for similar result. We present it for narration convenience. } 

\begin{lemma}\label{lm3}
If Assumptions \ref{as1}--\ref{as3} hold, and running the biased SGD (BSGD) iteration \eqref{eq:sgd} with a generic $\tilde{g}_t(x_t)$ satisfying \eqref{eq.unbias}, then for any $x\in \mathcal{X}$, we have that
\begin{equation}
\label{eq.lm3}
\mathbb{E}\left[{\cal R}_T({\texttt{BSGD}})\right]\leq \frac{\left\|x_{0}-x^*\right\|^{2}}{2 \eta}+2 L \delta T\nonumber+\frac{\eta}{2} \sum_{t=0}^{T}\mathbb{E}\left[\left\|\tilde{g}_{t}\left(x_{t}\right)\right\|^{2}\right] 
\end{equation}
where $x^*\in\arg\min _{x \in \mathcal{X}} \sum_{t=0}^{T} f_{t}(x)$.
\end{lemma}

The first term in the right-hand side (RHS) of \eqref{eq.lm3} is the initial distance to $x^*$; the second term shows the impact of the perturbation  $\delta$, which is due to the \emph{bias} of the gradient estimator (cf. \eqref{eq.unbias}); and the third term relates to both the bias and variance of $\tilde{g}_{t}$. Lemma \ref{lm3} implies that the regret of \eqref{eq:sgd} relative to the best fixed decision $x^*$ critically depends on the second moment bound of $\tilde{g}_{t}\left(x_{t}\right)$.

\vspace{-0.3cm}
\paragraph{Second moment of the gradient estimator $\tilde{g}_{t}\left(x_{t}\right)$.} We discuss the second moment bounds with respect to the one-point \cite{flaxman2004online}, one-point residual \cite{zhang2020boosting} and two-point ZO methods \cite{nesterov2017random}.

\textbf{ C1)}  The \textbf{classic one-point} gradient estimator \cite{flaxman2004online} can be written as 
\begin{equation}
\tilde{g}^{(0)}_t(x_t)\triangleq\frac{du_t}{\delta}f_t(x_t+\delta u_t)
\label{eq:flax}
\end{equation}
where $d$ is the dimension of $x_t$, $u_t\sim U(\mathbb{S})$ is a random vector from the unit sphere $\mathbb{S}=\left\{x \in \mathbb{R}^{d}|\| x \|=1\right\}$ centered around the origin. Its second moment satisfies $\mathbb{E}\left[\|\tilde{g}^{(0)}_{t}\left(x_{t}\right)\|^{2}\right]\leq \frac{d^2G^2}{\delta^2}$,
where $G$ is defined as $G\triangleq \max_{x\in{\cal X}}\max_t f_t(x)$.

\textbf{ C2)}  The \textbf{one-point residual} gradient estimator in \cite{zhang2020boosting} can be written as
\begin{equation}
\tilde{g}^{(1)}_t(x_t)\triangleq\frac{du_t}{\delta}\left(f_t(x_t+\delta u_t)-f_{t-1}(x_{t-1}+\delta u_{t-1})\right).
\label{eq:1e}
\end{equation}
Define the variation as $V_{\rm f}\triangleq \max_t\max_{x\in{\cal X}} |f_t(x)-f_{t-1}(x)|$. The second moment is bounded by
\begin{equation}\label{eq.tg_1}
\small
\EE\left[\left\|\tilde{g}^{(1)}_t(x_t)\right\|^{2}\right] \leq \max \Big\{\left\|\tilde{g}^{(1)}_0(x_0)\right\|^{2}, d^2 L^{2}+\frac{d^2V_{\rm f}^{2}}{\delta^{2}}  \Big\}
\end{equation}
where Lipschitz constant $L$ is defined in Assumption \ref{as1}.

\textbf{ C3)} The \textbf{asymmetric two-point} gradient estimator and its second moment are respectively \cite{nesterov2017random}  
\begin{align}
\small
 &\tilde{g}^{(2)}_t(x_t)\triangleq\frac{du_t}{\delta}\left(f_t(x_t+\delta u_t)-f_{t}(x_{t})\right),~~~~~ \mathbb{E}\left[\big\|\tilde{g}^{(2)}_{t}\left(x_{t}\right)\big\|^{2}\right]\leq d^2L^2
\label{eq:2e}
\end{align}
and, the \textbf{symmetric two-point} gradient estimator and its bound on the second moment are respectively \cite{shamir2017optimal}
\begin{align}\label{eq.tg_2}
\small
 &\tilde{g}^{(2)}_t(x_t)\triangleq\frac{du_t}{2\delta}\left(f_t(x_t+\delta u_t)-f_{t}(x_{t}-\delta u_t)\right),~~~~~ \mathbb{E}\left[\big\|\tilde{g}^{(2)}_{t}\left(x_{t}\right)\big\|^{2}\right]\leq dL^2.
\end{align}

\vspace{-0.2cm}
\textbf{Performance trade-off.}
Plugging the bounds C1)--C3) into Lemma \ref{lm3}, we observe that the parameter $\delta$ would play a crucial role on the regret. If we use \emph{more queries} per iteration (e.g., two-point ZO in C3)), $\delta$ does not appear in the bound of $\mathbb{E}\big[\left\|\tilde{g}_{t}\left(x_{t}\right)\right\|^{2}\big]$, and one can simply choose an arbitrarily small $\delta$ to minimize the regret bound. 
Hence, two-point ZO methods can reach ${\cal O}(T^\frac{1}{2})$ regret \cite{nesterov2017random,shamir2017optimal}.
On the other hand, if we use \emph{fewer queries} per iteration (e.g., one-point ZO in C1)--C2)), $\delta$ does appear in the denominator of the bound on $\mathbb{E}\big[\left\|\tilde{g}_{t}\left(x_{t}\right)\right\|^{2}\big]$. A trade-off thus emerges between the bias and variance of the gradient estimator since reducing $\delta$ (e.g., bias) will increase the variance bound on $\mathbb{E}\big[\left\|\tilde{g}_{t}\left(x_{t}\right)\right\|^{2}\big]$ in \eqref{eq.lm3}. 
Thus, the one-point ZO methods \cite{flaxman2004online,zhang2020boosting} only achieve ${\cal O}(T^\frac{3}{4})$ regret. 

\subsection{Key idea: Two lazy query rules}
Motivated by this delicate trade-off, we will develop a low-variance ZO method that achieves the best of one- and two-point ZO. 
Our key idea is to use the adaptive combination of the one-point residual and two-point gradient estimators. In this way, the algorithm will query new points only when one point estimator suffers high variance. 

Recalling the second moment bound of the one-point residual method in \eqref{eq.tg_1}, we notice the degrading term is $\frac{d^2}{\delta^2}V_{\rm f}$, which will reduce to that in \eqref{eq:2e} only when $V_{\rm f}={\cal O} (\delta^2)$, so that the regret bound can approach  $O(T^\frac{1}{2})$ correspondingly. We gauge that the requirement on $V_{\rm f}$ is stringent because it characterizes the maximum function changes at all points and iterations. Intuitively, we can still reuse old queries when the function variation at particular point and iteration is small. 

 
\vspace{-0.3cm}
\paragraph{A valuable revisit.} 
We carefully re-analyze the second moment bound for the one-point residual estimator $g_t^{(1)}(x_t)$ in \eqref{eq:1e}, and get the following instance-dependent bound (with probability $1$)
\begin{align}
\label{eq:Qt}
\small
\|\tilde{g}^{(1)}_t(x_t)\|^2\leq & \frac{|f_t(w_t)-f_{t-1}(w_{t-1})|^2}{\|w_t-w_{t-1}\|^2}\left(8d^2+\frac{2d^2\eta^2}{\delta^2}\|\tilde{g}_{t-1}^{(1)}(x_{t-1})\|^2\right)
\end{align}
where the perturbed point is defined as $w_t \triangleq x_t+\delta u_t$.

The proof of the instance-dependent bound \eqref{eq:Qt} can be found in Section \ref{appendix_a.2}. The bound implies that if the instance-wise function value variation $|f_t(w_t)-f_{t-1}(w_{t-1})|$ is small relative to $\|w_t-w_{t-1}\|$, 
the second moment bound of $\tilde{g}_t^{(1)}(x_t)$ can be even smaller than the bound for the two-point ZO method in \eqref{eq:2e}.  
We substantiate this instance-dependent analysis next. 
\begin{lemma}[Reduced norms]
\label{lm:temp}
Under Assumptions \ref{as1}--\ref{as2}, we run \eqref{eq:sgd} with $\tilde{g}_t^{(1)}(x_t)$ and set $\eta=\frac{R}{L\sqrt{dT}}$, $\delta=R\sqrt{\frac{d}{T}}$. For a given iteration $t$, if $\|\tilde{g}_{t-1}^{(1)}(x_{t-1})\|^2\leq d^2L^2$ and
\begin{equation}\label{lemma2_as}
\small
\frac{|f_t(w_t)-f_{t-1}(w_{t-1})|^2}{\|w_t-w_{t-1}\|^2}<\frac{L^2}{10d} 
\end{equation}
then the second moment of the one-point residual gradient estimator \eqref{eq:1e} satisfies $\|\tilde{g}_t^{(1)}(x_t)\|^{2}< dL^2$. 
\end{lemma}
\vspace{-0.1cm}
Compared with the second moment bound in \eqref{eq.tg_1} and \eqref{eq.tg_2}, Lemma \ref{lm:temp} quantitatively shows that if the temporal variation, in the sense of \eqref{lemma2_as}, is small relative to $L^2/d$, then reusing the delayed query can reduce the variance, and thus benefit the regret. 

Inspired by this, we introduce the following two definitions of the function variation. 

\begin{defi}[Temporal variation]
    \label{df:temp}
  For  $\forall x,y\in\mathcal{X}$, define the two temporal variation at $t$ as
    \begin{align}\label{eq.df:temp}
    \small
 D_t^a(x,y)\!\triangleq\!\frac{|f_t(x)\!-\!f_{t-1}(y)|}{\|x-y\|}~,~~~~~~~~
D_t^b(x,y)\!\triangleq\!\frac{|f_t(x)\!-\!f_{t-1}(y)|}{\eta L}.
    \end{align}
\end{defi}

The two definitions in Definition \ref{df:temp} differ in that: the temporal variation $D_t^a(x,y)$ normalizes the function values by the variation in terms of the query points, and the temporal variation $D_t^b(x,y)$ approximates the variation in terms of the query points by the stepsize $\eta$ since $\|w_t-w_{t-1}\|\approx{\cal O}(\eta L)$. 

\begin{wrapfigure}{R}{0.55\textwidth}
   \vspace{-0.85cm}
  \begin{minipage}[t]{7.8cm}
\begin{algorithm}[H]
   \caption{LAZO: Lazy query for ZO gradient method: \colorbox{red!30}{red part} is run only by \red{\bf LAZOa}; \colorbox{blue!30}{blue part} is implemented only by \blue{\bf LAZOb}; not both at the same time.}
   \label{alg:LAZO}
\begin{algorithmic}[1]
   \State {\bfseries Input:} $x_0\in \mathbb{R}^b$; $T, \delta,\eta, \gamma, D>0$.
   \State Sample $u_0\sim U(\mathbb{S})$.
   \State Query $f_0(w_0)$ and $f_0(x_0-\delta u_0)$.\quad\Comment{$w_0=x_0+\delta u_0$}
   \State $\tilde{g}_{0}(x_{0})=\frac{du_0}{2\delta}(f_t(w_0)-f_{t}(x_{0}-\delta u_0))$

   \For{$t=1$ {\bfseries to} $T$}
   \State Sample $u_t\sim U(\mathbb{S})$.
   \State Query $f_t(w_t)$. \qquad\qquad\qquad~\Comment{$w_t=x_t+\delta u_t$}
   \State Compute $\tilde{g}_{t}(x_{t})=\frac{du_t}{\delta}(f_t(w_t)-f_{t-1}(w_{t-1}))$
   \If{\colorbox{red!30}{$D_t^a(w_t,w_{t-1})$} or \colorbox{blue!30}{$D_t^b(w_t,w_{t-1})\!>\! D$}}
   \State Query $f_t(x_t-\delta u_t)$.
    \State $\tilde{g}_{t}(x_{t})=\frac{du_t}{2\delta}(f_t(w_t)-f_t(x_t-\delta u_t))$
   \EndIf
   \State Update $x_{t+1}=\Pi_{\mathcal{X}}\left(x_{t}-\eta \tilde{g}_{t}(x_{t})\right)$.
   \EndFor
\end{algorithmic}
\end{algorithm}
\end{minipage}
  \vspace{-1cm}
 \end{wrapfigure}

Building upon these two definitions, we design two alternative rules that check whether the temporal variations exceed a threshold $D$ to decide whether to query one or two points. Specifically, using the same ZO iteration \eqref{eq:sgd}, we propose the \textbf{LAZOa/b} gradient estimator as 
\begin{equation}
\small
\tilde{g}_{t}^{a/b}(x_{t})\!=\!\left\{\begin{array}{ll}\tilde{g}^{(1)}_t(x_t), &\! \!\!\text{If } D_t^{a/b}(w_t,w_{t-1})\!\leq\! D \\ \tilde{g}^{(2)}_t(x_t), &\!\!\! \text{else,} \end{array}\right.
\label{eq:clipping}
\end{equation}

We summarize the complete  LAZOa and LAZOb algorithms in Algorithm \ref{alg:LAZO}. More intuition can be seen in Section \ref{sec:intuition}.

LAZOa and LAZOb have trade-off between \emph{performance} and \emph{computation}. As defined in Definition \ref{df:temp}, $D_t^b$ can be viewed as an approximation to $D_t^a$, so LAZOa performs better through a more accurate temporal variation estimation. This is also shown in the experiments where LAZOa has faster convergence or achieves lower loss given a fixed iteration $T$. However, LAZOb has lower computation overhead since its rule can be equivalently written as $\|\tilde{g}^{(1)}_t(x_t)\|\!\leq\! \tilde D$, which is verified in Section \ref{comp}. 

\vspace{-0.2cm} 
\section{Theoretical Analysis}\label{theory}

This section provides theoretical guarantee for LAZO. Before we proceed, we first highlight the technical challenge of analyzing the regret of LAZO. 
 
\vspace{-0.1cm}
\subsection{Challenge: Lazy query introduces bias}\label{chan}

In a high level, the difficulty comes from that the lazy query breaks down the unbiasedness property of the ZO gradient estimator. 
For simplicity, we use the LAZOb estimator as an example to illustrate this, and the same argument also holds for the LAZOa estimator.

Let $A_t\triangleq \{u_t|D_t^b(w_t, w_{t-1})\leq D\}$ denote the region where LAZOb uses the one-point query and $\bar{A}_t$ denote the complementary set of $A_t$. 
To see the potential bias, we condition on the iterate $x_t$ and take the conditional expectation of the LAZOb estimator, given by
\begin{align}
\label{lazoexpect}
\small
&~~~~~~\mathbb{E}_{u_t}\Big[\tilde{g}_t^b(x_t)\Big|x_t\Big]\nonumber=\mathbb{E}_{u_t}\left[\tilde{g}_t^{(1)}(x_t)\mathbf{1}_{A_t}+\tilde{g}_t^{(2)}(x_t)\mathbf{1}_{\bar{A}_t}\Big|x_t\right]\nonumber\\
&\overset{(a)}{=}\mathbb{E}_{u_t}\left[\tilde{g}_t^{(1)}(x_t)\mathbf{1}_{A_t^s}+\tilde{g}_t^{(2)}(x_t)\mathbf{1}_{\bar{A}_t^s}\Big|x_t\right]+\mathbb{E}_{u_t}\Big[\tilde{g}_t^{(1)}(x_t)\mathbf{1}_{A_t\backslash A_t^s}-\tilde{g}_t^{(2)}(x_t)\mathbf{1}_{{\bar{A}_t^s\backslash \bar{A}_t}} \Big|x_t\Big]\nonumber\\
&\overset{(b)}{=}\underbrace{\mathbb{E}_{u_t}\!\left[\tilde{g}_t^{(1)}(x_t)\mathbf{1}_{A_t^s}+\tilde{g}_t^{(2)}(x_t)\mathbf{1}_{\bar{A}_t^s}\Big|x_t\right]}_{=\nabla f_{\delta,t}(x_t)}+\mathbb{E}_{u_t}\!\Big[\!\underbrace{\left(\tilde{g}_t^{(1)}(x_t)\!-\!\tilde{g}_t^{(2)}(x_t)\right)\mathbf{1}_{A_t\backslash A_t^s}}_{b_t\triangleq}\Big|x_t\Big]\!\! 
\end{align}
where (a) holds since $A_t^s\subseteq A_t$ denotes the largest symmetric subset of $A_t$, i.e., $A_t^s=\sup\{A|\mathbf{1}_{A|x_t}(u)=\mathbf{1}_{A|x_t}(-u), A\subseteq A_t \}$ so that $\mathbf{1}_{A_t}=\mathbf{1}_{A_t^s}+\mathbf{1}_{A_t\backslash A_t^s}$, and $\bar{A}_t^s$ denotes the complementary set of $A_t^s$ so that $\bar{A}_t\subseteq \bar{A}_t^s$ and $\mathbf{1}_{\bar{A}_t}=\mathbf{1}_{\bar{A}_t^s}-\mathbf{1}_{\bar{A}_t^s\backslash \bar{A}_t}$; (b) is due to $\mathbf{1}_{\bar{A}_t^s\backslash \bar{A}_t}=\mathbf{1}_{A_t\backslash A_t^s}$. 

For the two terms in \eqref{lazoexpect}, taking expectations of $\tilde{g}_t^{(1)}(x_t)$ and $\tilde{g}_t^{(2)}(x_t)$ over a symmetric distribution will give $\nabla f_{\delta,t}(x_t)$, and  the remaining term will be treated as the bias; see details in the supplementary material. 
This way of decomposition ensures that when the region of using one-point query is symmetric, the bias will be none; and the bias will diminish as $A_t$ is close to a symmetric set. 
This type of asymmetric bias also emerges in gradient clipping \cite{chen2020understanding} and signSGD \cite{bernstein2018signsgd}, where a symmetric gradient distribution is often assumed. We make a similar but weaker assumption in ZO below. 

\begin{assumption}
\label{as_sym1}
Denote $\mathbf{1}$ as the indicator function and $A|x_t$ as $A$ conditioned on $x_t$.
Let $A_t^s=\sup\{A|\mathbf{1}_{A|x_t}(u)=\mathbf{1}_{A|x_t}(-u), A\subseteq A_t \}$ be the maximum symmetric space in $A_t$. Assume that the nonasymmetric area satisfies $\sum_{t=0}^{T}\mathbb{P}(A_t\backslash A_t^s)={\cal O}(\sqrt{{T}/{d}})$. 
\end{assumption}
\vspace{-0.3cm}
 
Assumption \ref{as_sym1} is satisfied even if the summation of the probability series does not converge. For example, if $\mathbb{P}(A_t\backslash A_t^s)\sim{\cal O}(\frac{1}{t})$, then $\sum_{t=0}^{T}\mathbb{P}(A_t\backslash A_t^s)\sim{\cal O}(\log{T})<{\cal O}(\sqrt{{T}/{d}})$.
 
\vspace{-0.1cm} 
 \noindent\textbf{Discussion on symmetricity.}
Assumption \ref{as_sym1} indicates that the active region of lazy rule $A_t$ is asymptotically symmetric, which turns out to nearly hold throughout our simulations. Since projection preserves symmetricity of a symmetric space. We gauge that if for all random projections, the projected areas are symmetric, the original $A_t$ is symmetric. 
We verify this via the linear quadratic regulator (LQR) and resource allocation experiment in Section \ref{experiment} and show in Figure \ref{fig:symm}. We randomly choose an iteration for each experiment, sample $4\times 10^{4}$ random perturbations $u_t$ and generate $4$ random $d\times 2$ Gaussian matrices to project those $u_t\in A_t$ onto a random 2-dimensional space and visualize them, from which we can see that all of the 8 projections are almost symmetric. More validations are included in Section  \ref{experiment_detail}.

\begin{figure}[htb]
    \centering
    \setlength{\tabcolsep}{-0.02cm}
    \begin{tabular}{cccccccc}
\includegraphics[width=.13\textwidth]{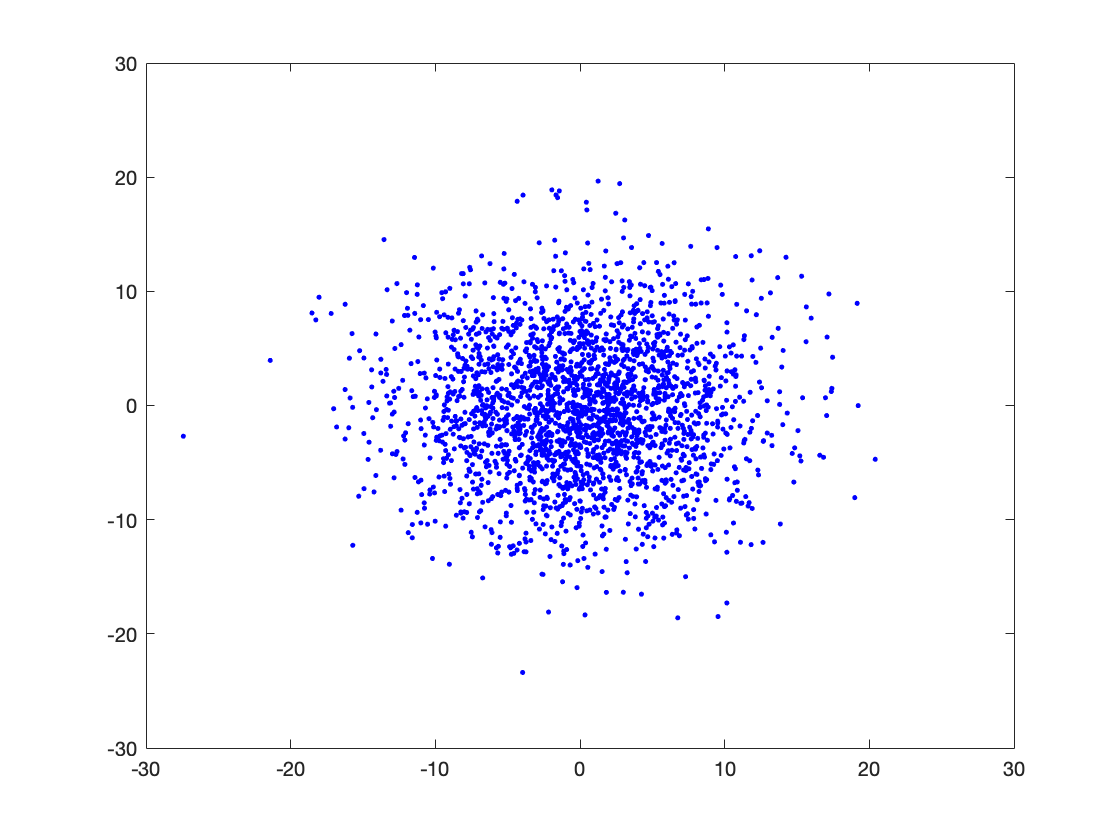} &
\includegraphics[width=.13\textwidth]{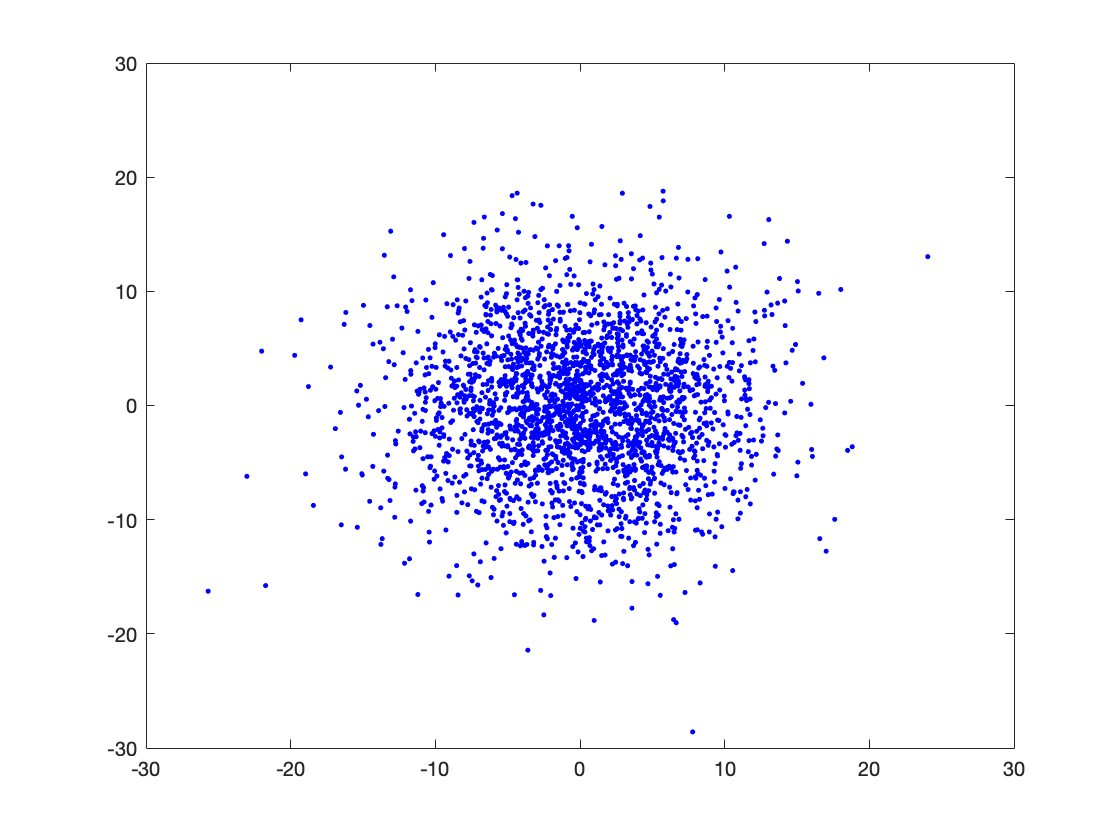} &
 \includegraphics[width=.13\textwidth]{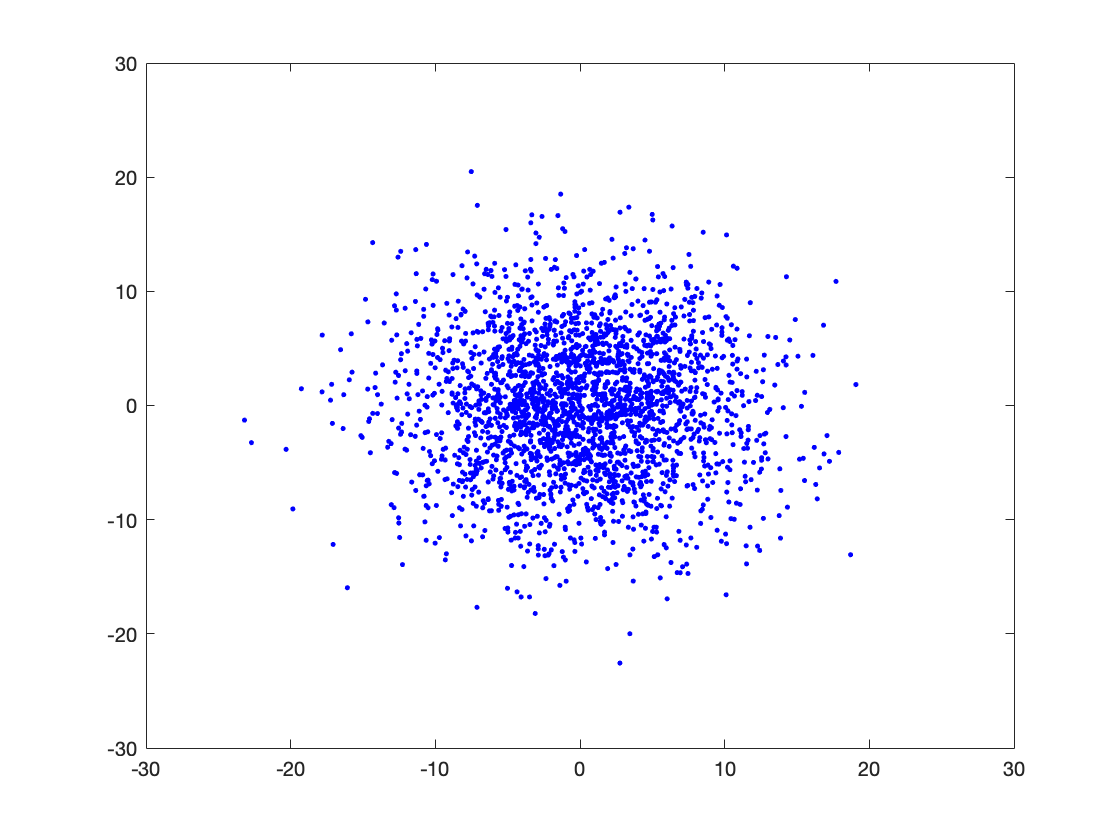}&
\includegraphics[width=.13\textwidth]{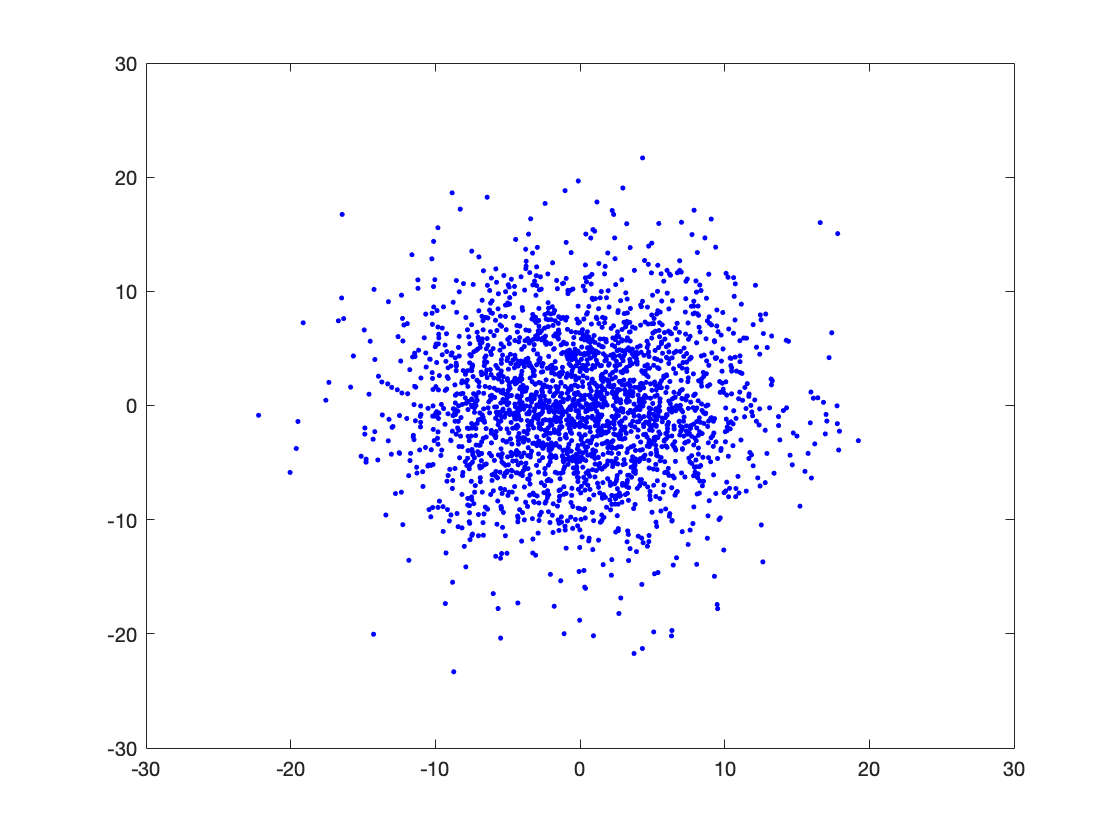}&
\includegraphics[width=.13\textwidth]{./symmetric/1.png} &
\includegraphics[width=.13\textwidth]{./symmetric/2.png} &
 \includegraphics[width=.13\textwidth]{./symmetric/3.png}&
\includegraphics[width=.13\textwidth]{./symmetric/4.png}\\
 \footnotesize(a) &\footnotesize(b) &

\footnotesize(c) &\footnotesize(d)&
 \footnotesize(e) &\footnotesize(f)  &

\footnotesize(g) &\footnotesize(h)\\
    \end{tabular}
       \vspace{-0.2cm}
    \caption{Distribution of $u_t\in A_t$ projected using 4 random matrices for LAZOb. (a) -- (d) are for LQR control and (e) -- (h) are for resource allocation. }
     \vspace{-0.3cm}
    \label{fig:symm}
\end{figure}


\subsection{Results in online convex optimization}
We first verify that the LAZO estimator is an asymptotically unbiased gradient estimator of the smoothed function $f_{\delta,t}$. 

\begin{lemma}
\label{lm:bt-bias}
Under Assumptions \ref{as1}--\ref{as_sym1}, the bias of the LAZO gradient estimator is 
\begin{align}
&\mathbb{E}\left[\tilde{g}_t^a(x_t)|x_t\right]=\nabla f_{\delta,t}(x_t)+\mathbb{E}\left[\tilde{b}_t\Big|x_t\right], ~~~~~~\mathbb{E}\left[\tilde{g}_t^b(x_t)|x_t\right]=\nabla f_{\delta,t}(x_t)+\mathbb{E}\left[b_t\Big|x_t\right]
\end{align}
where $b_t$ is defined in \eqref{lazoexpect}, while $\tilde{b}_t$ is defined similarly by replacing $A_t, A_t^s$ with $\tilde{A}_t\triangleq \{u_t|D_t^a(w_t, w_{t-1})\leq D\}, \tilde{A}_t^s=\sup\{A|\mathbf{1}_{A|x_t}(u)=\mathbf{1}_{A|x_t}(-u), A\subseteq \tilde{A}_t\}$.
Moreover, if we use $\eta=\frac{R}{L\sqrt{dT}}$, $\delta=R\sqrt{\frac{d}{T}}$, $D={\cal O}(\frac{L}{\sqrt{d}})<\frac{L}{\sqrt{10}}$ for LAZOa and $D={\cal O}(\sqrt{d}L)$ for LAZOb, then both $\sum_{t=0}^T\mathbb{E}\left[\|b_t\|\right]={\cal O}(\sqrt{dT})$ and $\sum_{t=0}^T\mathbb{E}[\|\tilde{b}_t\|]={\cal O}(\sqrt{dT})$. 
\end{lemma} 

Then we establish the second moment bound of the gradient estimator in LAZO.

\begin{lemma}
\label{clip2bound}
Under Assumptions \ref{as1}--\ref{as3}, the second moment bound of the gradient estimator $\tilde{g}_t^a(x_t)$ and $\tilde{g}_t^b(x_t)$ satisfy that there exists a constant $c={\cal O}(1)$ such that
\begin{subequations}\label{eq:clip2}
\begin{align}
    &\mathbb{E}\left[\|\tilde{g}_t^a(x_t)\|^2\Big|x_t\right]\leq   cdL^2+\frac{2d^2D^2\eta^2}{\delta^2}\|\tilde{g}_{t-1}^a(x_{t-1})\|^2+8d^2D^2; \label{eq:clip2a}\\
    &\mathbb{E}\left[\|\tilde{g}_t^b(x_t)\|^2\Big|x_t\right]\leq   \frac{D^2\eta^2d^2L^2}{\delta^2} +cdL^2.\label{eq:clip2b}
\end{align}
\end{subequations}
\end{lemma}

In \eqref{eq:clip2}, one can show that by properly choosing $D$ and $\eta$, the second moment bound is ${\cal O}(d)$, which is better than the bound of the one-point residual estimator \eqref{eq.tg_1} and the asymmetric two-point estimator \eqref{eq:2e} in terms of the $d$-dependence. In addition, one can choose $\eta,\delta={\cal O}(1/{\sqrt{T}})$, so comparing with the one-point and one-point residual estimator, the dependence on the term $\delta^{-2}$ will be cancelled out in \eqref{eq:clip2}. Thus, reusing old queries appropriately will 
reduce the variance of the gradient estimator. 

With the above two lemmas and the biased variant of Lemma \ref{lm3}, we can get the following regret bound of LAZO. The complete proof is presented in Appendix.  

\begin{theorem}[LAZO in the convex case]
\label{thm:lazoc}
Under Assumptions \ref{as1}--\ref{as_sym1}, we run LAZO for $T$ iterations with $\eta=\frac{R}{L\sqrt{dT}}$, and $\delta=R\sqrt{\frac{d}{T}}$. 
If we use $D={\cal O}(\frac{L}{\sqrt{d}})<\frac{L}{\sqrt{10}}$ for LAZOa and $D={\cal O}(\sqrt{d}L)$ for LAZOb, the regrets for LAZOa and LAZOb satisfy
\begin{align*}
\!\!\mathbb{E}[{\cal R }_T(\texttt{LAZOa})]= {\cal O}(\sqrt{dT}),~~~~~ \mathbb{E}[{\cal R }_T(\texttt{LAZOb})]= {\cal O}(\sqrt{dT}).
\end{align*}
\end{theorem}
\vspace{-0.2cm}
From Theorem \ref{thm:lazoc}, we know that thanks to the $\delta$-independent variance in \eqref{eq:clip2}, both LAZOa and LAZOb can achieve ${\cal O}(\sqrt{dT})$ regret, which improves the ${\cal O}(T^{3/4})$, ${\cal O}(d\sqrt{T})$ regret bounds of one-point (residual) methods in \cite{flaxman2004online,zhang2020boosting}, the asymmetric two-point gradient estimator in \cite{nesterov2017random}, respectively, and achieves the optimal ${\cal O}(\sqrt{dT})$ regret bound for the symmetric two-point ZO method in \cite{shamir2017optimal}.

\vspace{-0.3cm}
\subsection{Results in nonconvex stochastic optimization}
LAZO can be also applied to the stochastic optimization setting. Since the convex stochastic case is a special case of the OCO setting, we defer the results in Section \ref{nonconvex-proof}. 
Due to the popularity of nonconvex learning applications, we present the results of the nonconvex setting. 

Consider the function $F(x; \xi)$ that depends on the random variable $\xi$ and the stochastic problem $\min _{x \in \mathbb{R}^d}\, f(x)\triangleq\EE_\xi[F(x; \xi)]$.
In this case, instead of minimizing the regret \eqref{eq.regret}, the goal is to minimize the average gradient norm as ${\cal R}_T^{\rm nc}({\cal A})\triangleq \sum_{t=0}^T\mathbb{E}\left[\|\nabla f(x_t)\|^2\right]$.

Regarding algorithms, we can still implement LAZO in Algorithm \ref{alg:LAZO} by replacing $f_t(x)=F(x,\xi_t)$ and leave out the projection step in the projected ZO gradient descent since the feasible set is $\mathbb{R}^d$. We make the following assumptions in addition to Assumption \ref{as1} in the OCO setting.
\begin{assumption}
\label{as8}
Assume that $f(x)$ is $\mu$-smooth, i.e. $\forall x,y\in\mathcal{X}, \|\nabla f(x)-\nabla f(y)\|\leq \mu\|x-y\|$. 
\end{assumption}

\begin{theorem}[LAZO in the nonconvex case]\label{them_ncvx}
Under Assumptions \ref{as1}, \ref{as_sym1}, \ref{as8}, we run LAZO for $T$ iterations with $\eta=\frac{1}{L\sqrt{dT}}$, $\delta=\sqrt{\frac{d}{T}}$ and use $D={\cal O}(\frac{L}{\sqrt{d}})<\frac{L}{\sqrt{10}}$ for LAZOa and $D={\cal O}(\sqrt{d}L)$ for LAZOb.
The regrets for LAZOa and LAZOb satisfy 
\begin{align*}
\mathbb{E}[{\cal R }_T^{\rm nc}(\texttt{LAZOa})]\leq {\cal O}(\sqrt{dT}), \mathbb{E}[{\cal R }_T^{\rm nc}(\texttt{LAZOb})] \leq {\cal O}(\sqrt{dT}).
\end{align*}
\end{theorem}
\vspace{-0.2cm}
Theorem \ref{them_ncvx} shows that to achieve $ \mathbb{E}\left[{\cal R }_T^{\rm nc}(\texttt{LAZO})\right]/T\leq\epsilon$, LAZOa and LAZOb require ${\cal O}(d\epsilon^{-2})$ iterations, and the iteration complexity is the same as that of two-point ZO SGD method in \cite{ghadimi2013stochastic}.

\vspace{-0.2cm}
\section{Extension to Multi-point LAZO  Rules}
\vspace{-0.1cm}
To further reduce the variance of gradient estimation, we can construct a $2K$-point LAZO gradient estimator with $K>1$, akin to \cite{shamir2017optimal,agarwal2010optimal}. In this case, we can apply our idea of lazy queries to existing multi-point ZO methods by expanding the reusing horizon from one round to $H$ rounds with $H>1$ to benefit the query complexity. 
For $\tau=1,\cdots H$, we can extend the definition of temporal variation in Definition \ref{df:temp} to $D_{t,\tau}^{a}(x,y)\triangleq |f_t(x)-f_{t-\tau}(y)|/\|x-y\|$ and $D_{t,\tau}^{b}(x,y)\triangleq |f_t(x)-f_{t-\tau}(y)|/(\eta L)$. 
Then we can check the informativeness of the $K$ queries at previous $H$ iterations based on the generalized notion of temporal variations, and query new points after reusing all appropriate old points to form $2K$-point estimator. Due to space limitation, we present the full details of this multi-point LAZO variant and its pseudo-code in Section \ref{exten_app} in Appendix.


\section{Numerical Experiments}\label{experiment}
\vspace{-0.1cm}

In this section, we empirically evaluate the performance of our LAZO and its multi-point variant on three applications: LQR control, resource allocation and generation of adversarial examples from a black-box deep neural network (DNN). Throughout this section, we compare LAZO with one-point residual algorithm \cite{zhang2020boosting} and two-point ZO gradient descent \cite{shamir2017optimal}. 
The detailed setup and the choice of parameters are in Section \ref{experiment_detail}. 

\begin{figure*}[htb]
    \setlength{\tabcolsep}{-0.07cm}
    \centering
    \vspace{-0.2cm}
    \begin{tabular}{ccc}
    \includegraphics[width=.34\textwidth]{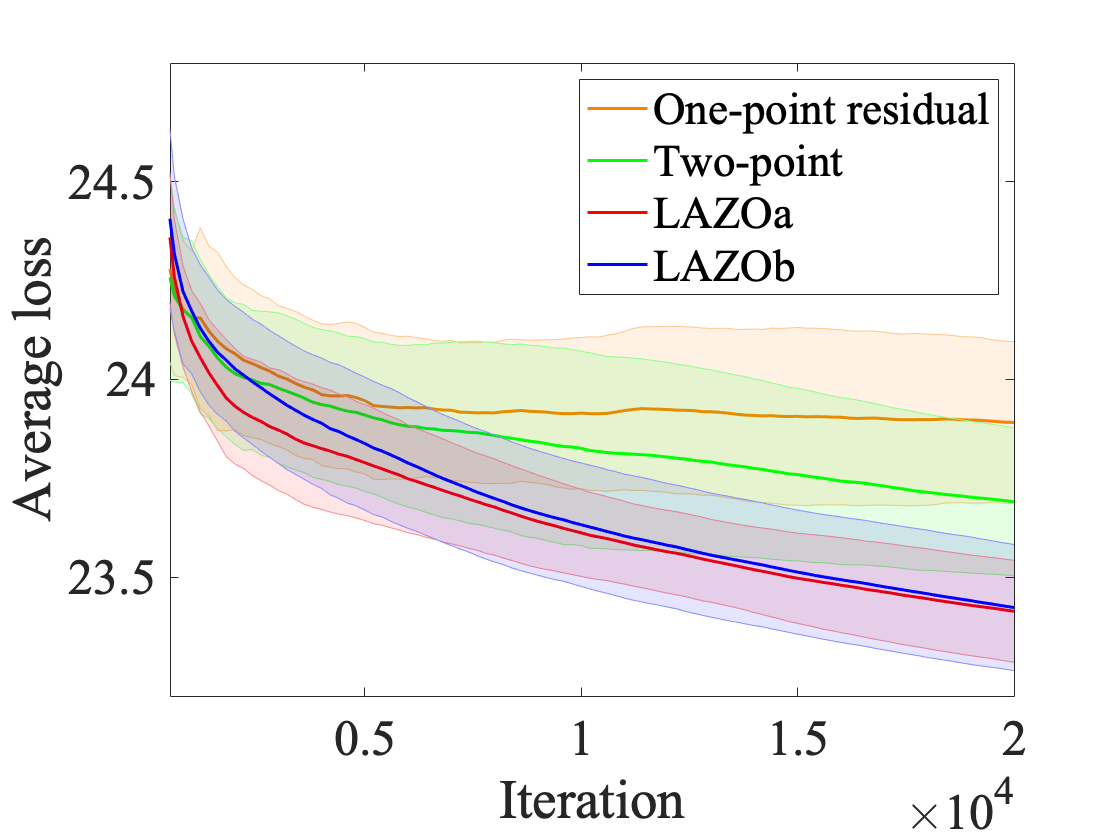}&
    \includegraphics[width=.34\textwidth]{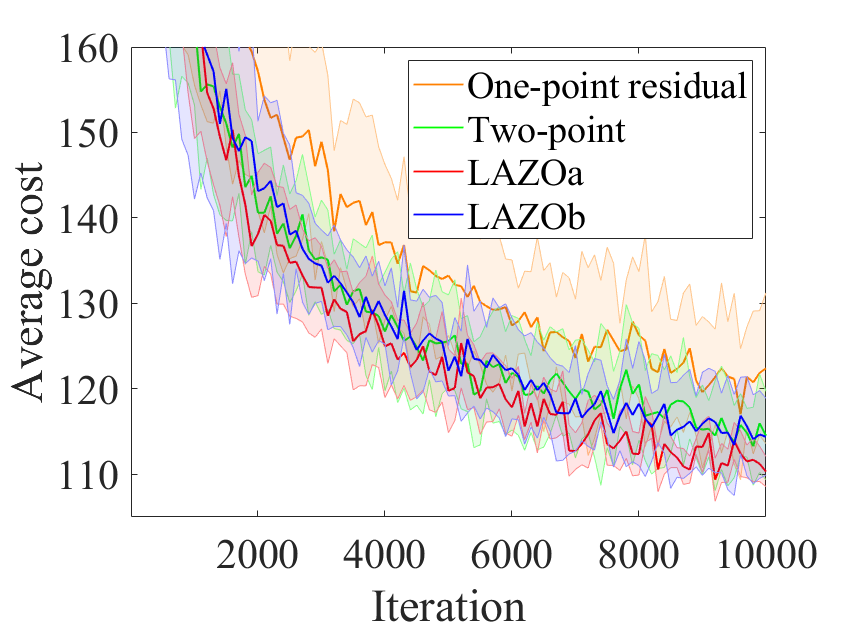}&\includegraphics[width=.34\textwidth]{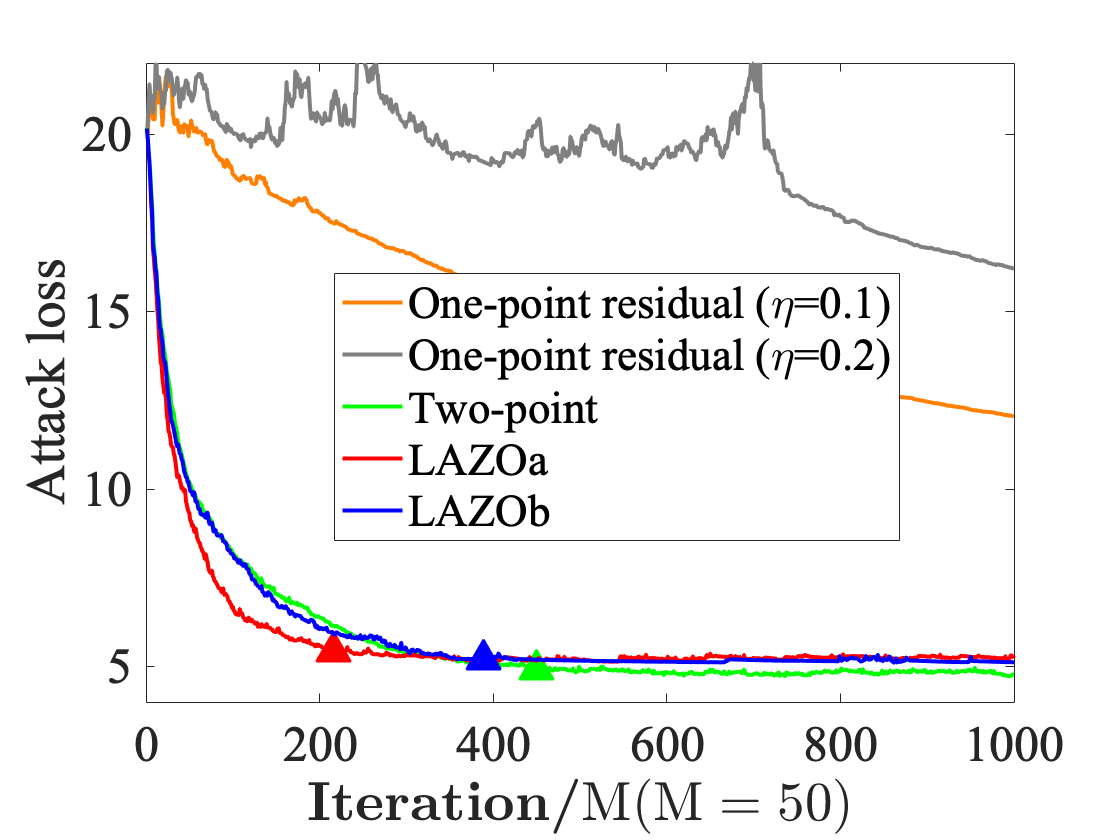}\\
    \footnotesize(a) Average loss v.s. iteration &\footnotesize(b) Average cost v.s. iteration&\footnotesize(c) Attack loss v.s. iteration\\
    \includegraphics[width=.34\textwidth]{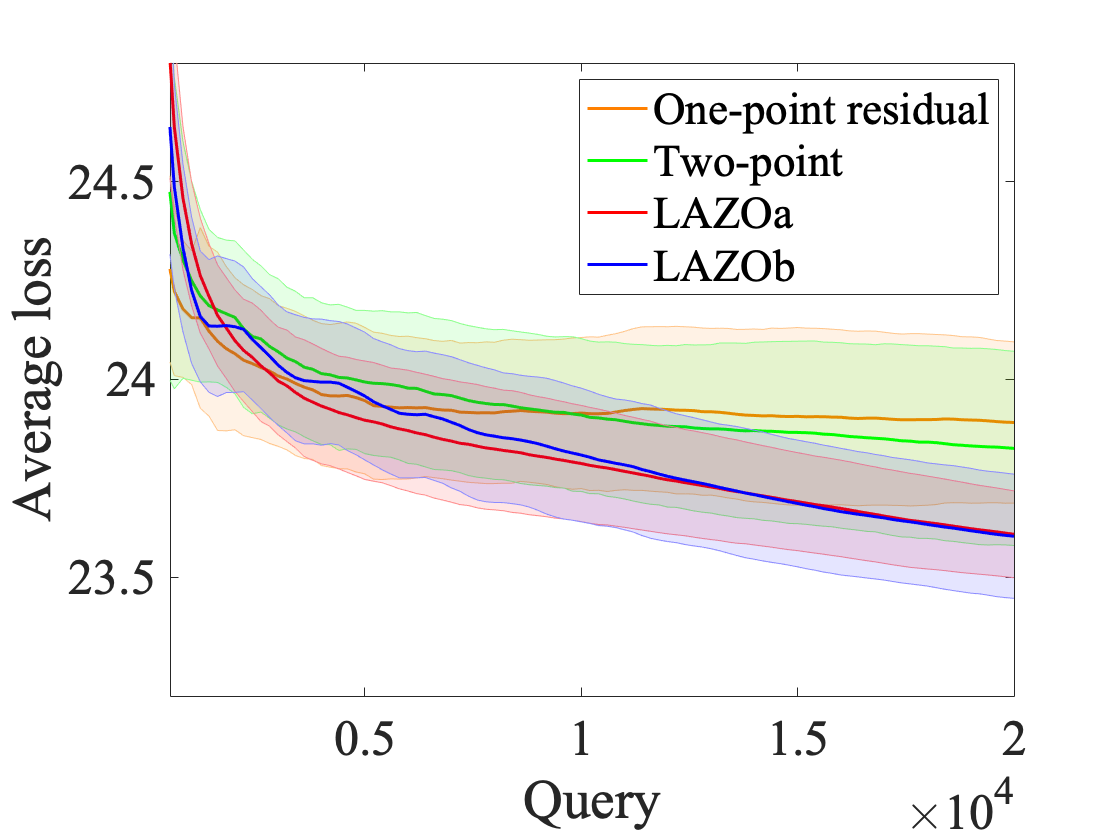}&
    \includegraphics[width=.34\textwidth]{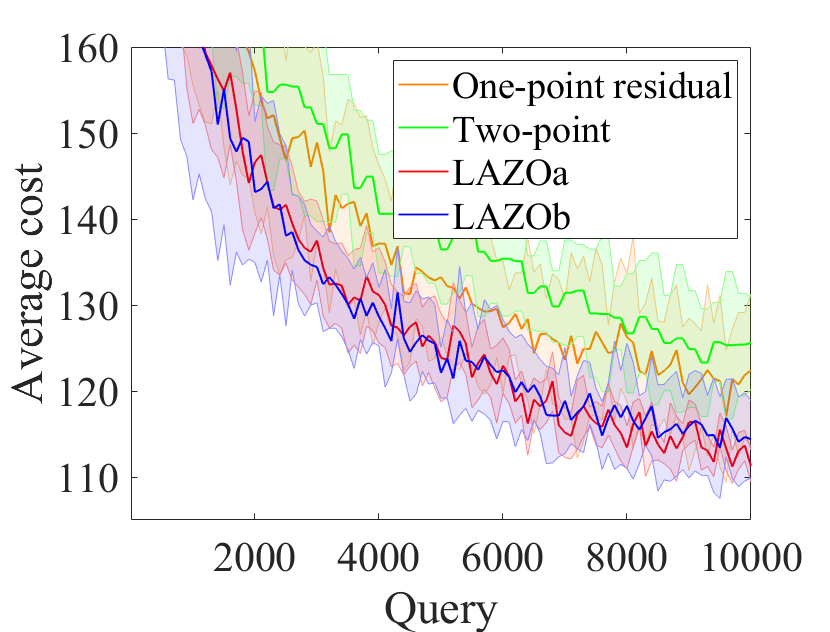}&\includegraphics[width=.34\textwidth]{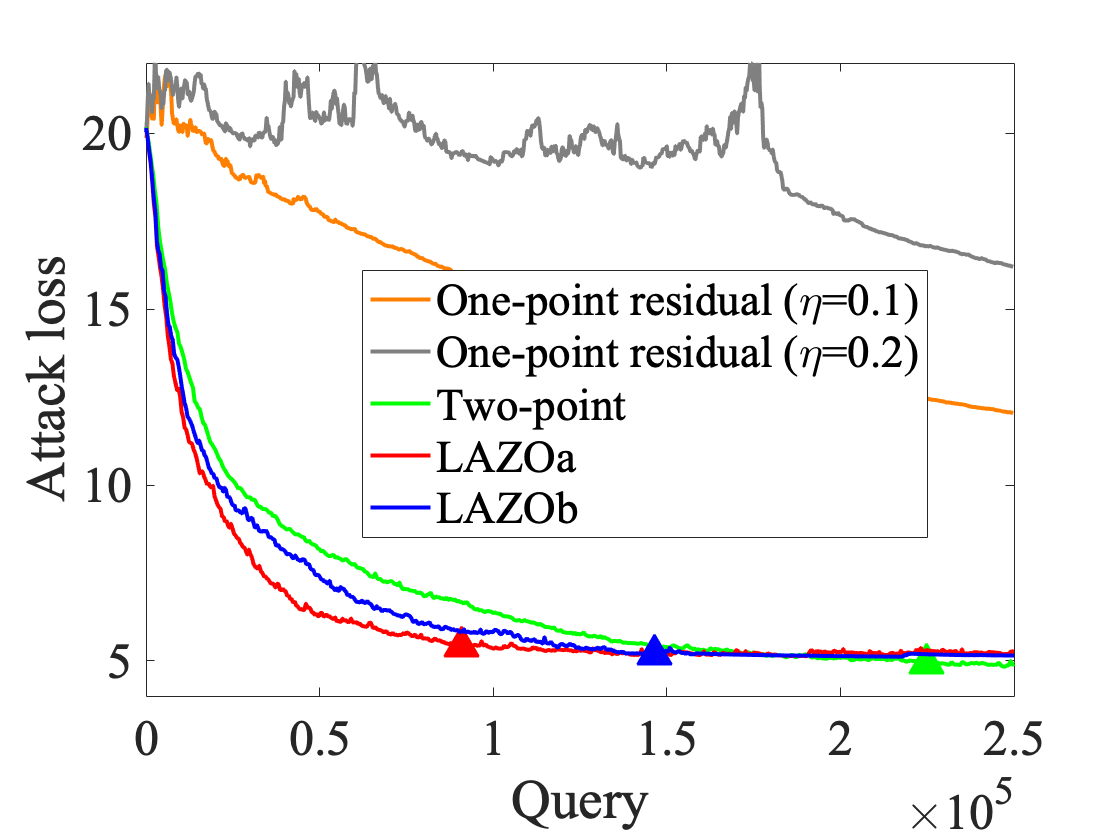}\\
    \footnotesize(d) Average loss v.s. query &\footnotesize(e) Average cost v.s. query&\footnotesize(f) Attack loss v.s. query\\
    \end{tabular}
    \caption{Comparative results of ZO with the one-point residual feedback (\orange{orange}), the two-point oracle (\green{green}), as well as proposed LAZOa (\red{red}) and LAZOb (\blue{blue}) on LQR (a, d), resource allocation (b, e) and black box attack (c, f). The solid line shows the average loss/cost over $10$ independent trials with random initialization, and the shaded region denotes the standard deviation of results over random trials. The triangle points indicate iteration and query numbers of the first successful attack.}
    \label{fig:result_LQR}
    \vspace{-0.3cm}
\end{figure*}


\begin{figure*}[htb]
    \centering
    \begin{tabular}{cc}
    \includegraphics[width=.34\textwidth]{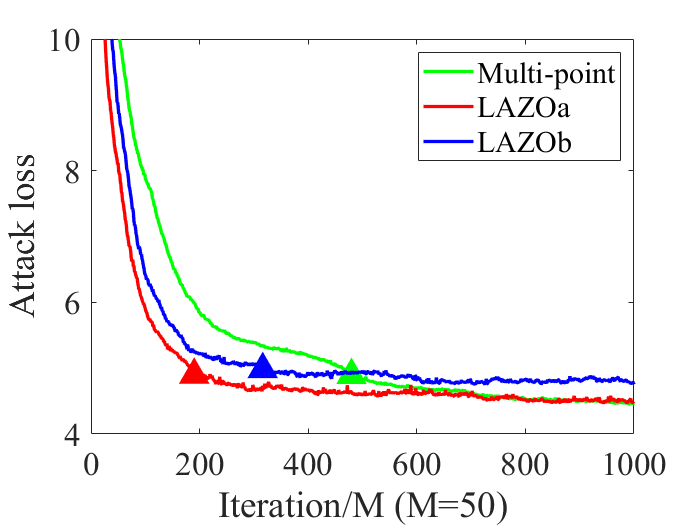}&
    \includegraphics[width=.34\textwidth]{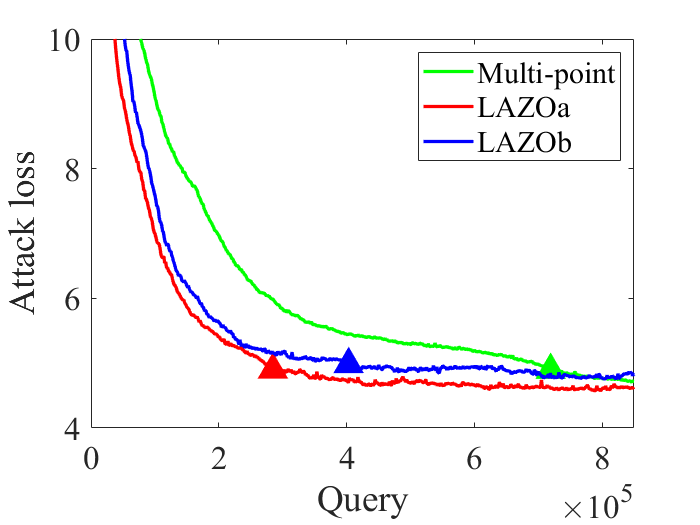}\\
    \footnotesize(a) Attack loss v.s. iteration &\footnotesize(b) Attack loss v.s. query\\
    \end{tabular}
    \caption{Comparative results of ZO with the multi-point oracle (\green{green}), as well as proposed LAZOa (\red{red}) and LAZOb (\blue{blue}) on the black-box attack task. The triangle points indicate iteration and query numbers of the first successful attack. }
    \label{fig:result_DNN_33}
    \vspace{-0.2cm}
\end{figure*}

\vspace{-0.2cm}
\subsection{Non-stationary LQR control}\label{lqr_ex}
\vspace{-0.1cm}
We study a non-stationary version of the LQR  problem \cite{fazel2018global} with the time-varying dynamics. At iteration $t$, consider the linear dynamic system described by the dynamic $x_{k+1}=A_tx_k+B_tq_k$, where $x_k\in \mathbb{R}^n$ is the state, $q_k\in\mathbb{R}^p$ is the control variable at step $k$, $A_t\in \mathbb{R}^{n\times n}$ and $B_t\in \mathbb{R}^{n\times p}$ are the dynamic matrices for iteration $t$. We optimize the control $q_k=K x_k$ that linearly depends on the current state $x_k$, where $K \in \mathbb{R}^{p\times n}$ is the policy. The loss function will be 
\begin{align*}
\small
\min_K &~~ f_t(K)=\mathbb{E}\left[\frac{1}{H} \sum_{k=1}^{H} \beta^k x_{k}^{\top} (Q+K^{\top} RK) x_{k}\right],~~\text {s.t. } ~~ x_{k+1}=(A_t+B_t K)x_k. 
\end{align*}

\vspace{-0.3cm}
We set $n=p=6$, $\beta=0.5$, $\delta=0.01$, and stepsize $\eta=10^{-5}$ for all  methods. We generate $A_t$, $B_t$ to mimic the situations where the losses encounter intermittent changes.

\begin{wrapfigure}{r}{2.45in}
\def\epsfsize#1#2{0.5#1}
\vspace{-0.2cm}
    \includegraphics[width=.46\textwidth]{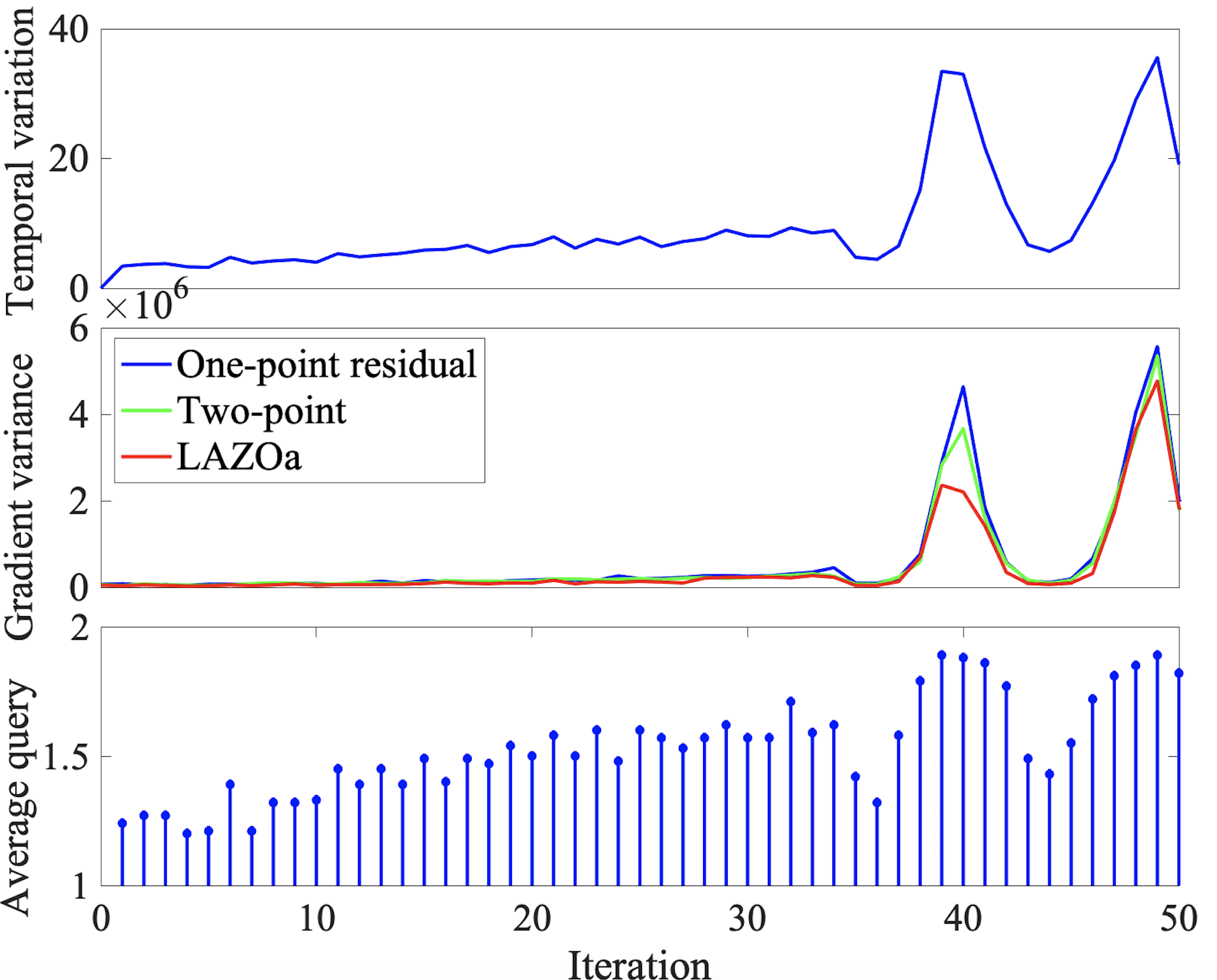}
        \vspace{-0.6cm}
     \caption{Monitoring the adaptive query condition in LQR at first $50$ iterations. The upmost plot records the average temporal variation over 100 trials. The gradient estimation variance for one-point residual method (\blue{blue}), two-point method (\green{green}) and LAZOa (\red{red}) is shown in the middle plot. The lowermost plot shows 
     the average queries in LAZOa. }
     \label{fig:zo-monitor}
\vspace{-0.2cm}
\end{wrapfigure}
We monitor the average temporal variation, gradient variance and average queries per iteration over $50$ iterations and $100$ trials in Figure \ref{fig:zo-monitor} for LAZOa. We observe that when the loss function varies slowly (e.g., $t\leq 35$), the temporal variation is also small and thus, the average query for LAZOa is relatively small; when the loss function changes rapidly (e.g., $t\geq 35$), the temporal variation is also large and as a result, LAZOa needs more average queries. Note that the actual upper bound of two-point ZO's variance depends on $L_t$ and the way we generate $A_t,B_t$ will affect not only the temporal variation but also $L_t$, resulting in the change of variance of two-point ZO. Figure \ref{fig:zo-monitor} also indicates that thanks to the lazy query, the variance of the LAZOa gradient estimator keeps the lowest. In Figure \ref{fig:result_LQR}, we report the cost versus iteration and query of the four methods. Here we choose $D=1$ for LAZOa and $D=100/L$ for LAZOb to optimize the performance for them. 
In Figure \ref{fig:result_LQR}(a), LAZOa and LAZOb yield the best convergence performance and LAZOa has the smallest errorbar over random trials. 
Regarding query complexity in Figure \ref{fig:result_LQR}(d), 
LAZO still outperforms the other two methods. Besides, we provide the runtime comparison for one-point residual method, two-point method and LAZOa/b in Section \ref{comp}, which implies the adaptive rules of LAZOa/b will not bring huge computational overhead and are even slightly time-saving compared to the two-point ZO method.

\vspace{-0.2cm}
\subsection{Non-stationary resource allocation}
\vspace{-0.2cm}

We consider a resource allocation problem with $16$ agents connected by a ring graph. 
Figure \ref{fig:result_LQR}(b) and (e) present the cumulative cost versus iteration and query. LAZOa outperforms the other methods in terms of iteration while both LAZOa and LAZOb improve the query complexity. Besides, we can see that both two LAZO variants have narrower errorbar compared with the other methods, and LAZOa is more stable than LAZOb over different trials. 

\vspace{-0.3cm}
\subsection{Black-box adversarial attacks}\label{black-attack} 
\vspace{-0.2cm}

In the nonconvex stochastic setting, we study generating adversarial examples from an image classifier given by a black-box DNN on the MNIST dataset \cite{mnistdata}. Simulation settings are in Section \ref{experiment_detail}. In the experiment, since the one-point residual method is unstable, we choose a smaller stepsize $\eta=0.1$ and bigger $\delta=0.5$ for it to ensure stability. To show $\eta=0.1$ is a reasonable choice for one-point residual method, we also report the result when $\eta=0.2$.
For the other three methods, we pick the optimal $\eta\in \{0.1, 1,2,\cdots,10\}$ and $\delta\in\{0.5,0.1,0.05,0.01\}$. Besides, we set $D=1$ for LAZOa and $D=1000/L$ for LAZOb. 
Figure \ref{fig:result_LQR}(c) and (f) show the attacking loss versus iteration and query for the four methods. LAZOa and LAZOb outperform one-point residual and two-point ZO methods in terms of both iteration and query. In addition, compared to two-point ZO methods, LAZOa and LAZOb only requires nearly $40\%$ and $67\%$ queries to achieve the first successful attack, which is the first iteration or query number when the $\ell_2$ distortion loss begins to drop \cite{chen2017zoo,liu2018signsgd}, respectively. 

We also compare multi-point LAZO with multi-point ZO method \cite{shamir2017optimal} and the comparison results for $H=3,K=3$ setting are shown in Figure \ref{fig:result_DNN_33}. In Figure \ref{fig:result_DNN_33}, both multi-point LAZOa and multi-point LAZOb outperform multi-point ZO in terms of iteration complexity and query complexity. Moreover, multi-point LAZOa and LAZOb requires nearly $38\%$ and $55\%$ queries to achieve the first successful attack, respectively, which further improve the results for $H=1,K=1$ in Figure \ref{fig:result_LQR}(c) and (f).

\section{Conclusions}\label{conclusion}

This paper proposes a novel ZO gradient estimation method based on a lazy query condition. Different from the classic ZO methods, LAZO monitors the informativeness of old queries, and then adaptively reuses them to construct the low-variance stochastic gradient estimates. We rigorously establish that through judiciously reusing the old queries, LAZO not only saves queries but also achieves the regrets of the symmetric two-point ZO methods \cite{shamir2017optimal}. 
Future research includes: i) extending LAZO to the decentralized setting and ii) extending LAZO to various ZO variants.

\section*{Acknowledgments}
This work was supported by National Science Foundation CAREER Award 2047177, and the Rensselaer-IBM AI Research Collaboration (\url{http://airc.rpi.edu}), part of the IBM AI Horizons Network (\url{http://ibm.biz/AIHorizons}).

\bibliography{lazo_aistats}
\bibliographystyle{plain}

\clearpage

\clearpage
\onecolumn
\begin{center}
{\large \bf Supplementary Material for\\
``\FullTitle"}
\end{center}
\appendix

\vspace{-1.5cm}
\addcontentsline{toc}{section}{} 
\part{} 
\parttoc 

\vspace{0.5cm}

\section{Proofs of the results in online convex optimization}\label{oco_proof}

\subsection{Proof of Lemma \ref{lm3}}

Lemma \ref{lm3} is a standard result of biased SGD. To be self-contained, we provide its proof here. 

\begin{proof}

Since $f_{\delta,t}$ is convex for all $t$, we have that
\begin{equation}
\label{eq:32}
f_{\delta, t}\left(x_{t}\right)-f_{\delta, t}(x) \leq\left\langle\nabla f_{\delta, t}\left(x_{t}\right), x_{t}-x\right\rangle, \text { for all } x \in \mathcal{X}.
\end{equation}
Using $\mathbb{E}\left[\tilde{g}_t(x_t)|x_t\right]=\nabla f_{\delta,t}(x_t)$ in \eqref{eq.unbias} and taking expectation over both sides, we get for all $x\in\mathcal{X}$,
\begin{align}
\mathbb{E}\left[f_{\delta, t}\left(x_{t}\right)-f_{\delta, t}(x)\right] &\leq\mathbb{E}\left[\left\langle\mathbb{E}\left[\tilde{g}_t(x_t)|x_t\right], x_{t}-x\right\rangle\right]\nonumber\\
&=\mathbb{E}\left[\mathbb{E}\left[\left\langle\tilde{g}_t(x_t), x_{t}-x\right\rangle|x_t\right]\right]=\mathbb{E}\left[\left\langle\tilde{g}_t(x_t), x_{t}-x\right\rangle\right].
\label{eq:29}
\end{align}
Since $x_{t+1}=\Pi_{\mathcal{X}}\left[x_{t}-\eta \tilde{g}_t\left(x_{t}\right)\right]$, for any $x\in\mathcal{X}$ we have that
\begin{align*}
 \left\|x_{t+1}-x\right\|^{2} &=\left\|\Pi_{\mathcal{X}}\left[x_{t}-\eta \tilde{g}_t\left(x_{t}\right)\right]-\Pi_{\mathcal{X}}[x]\right\|^{2} \\ & \leq\left\|x_{t}-\eta \tilde{g}_t\left(x_{t}\right)-x\right\|^{2} \\ &=\left\|x_{t}-x\right\|^{2}-2 \eta\left\langle\tilde{g}_{t}\left(x_{t}\right), x_{t}-x\right\rangle+\eta^{2}\left\|\tilde{g}_{t}\left(x_{t}\right)\right\|^{2}.  \numberthis
\label{eq:30}
\end{align*}
where the inequality follows the non-expansive property of the projection.

Rearranging the terms in inequality \eqref{eq:30} yields
\begin{equation}
\left\langle\tilde{g}_{t}\left(x_{t}\right), x_{t}-x\right\rangle \leq \frac{1}{2 \eta}\left(\left\|x_{t}-x\right\|^{2}-\left\|x_{t+1}-x\right\|^{2}\right)+\frac{\eta}{2}\left\|\tilde{g}_{t}\left(x_{t}\right)\right\|^{2}.
\label{eq:31}
\end{equation}
Taking expectations on both sides of inequality \eqref{eq:31} with respect to $u_t$, substituting the resulting bound into \eqref{eq:29} and summing from $t=0$ to $T$, we obtain that
\begin{equation}
\label{eq:37}
\mathbb{E}\left[\sum_{t=0}^{T} f_{\delta, t}\left(x_{t}\right)-\sum_{t=0}^{T} f_{\delta,t}(x)\right] \leq \frac{1}{2 \eta}\left\|x_{0}-x\right\|^{2}+\frac{\eta}{2} \mathbb{E}\left[\sum_{t=0}^{T}\left\|\tilde{g}_{t}\left(x_{t}\right)\right\|^{2}\right].
\end{equation}
Since $f_t$ is Lipschitz for all $t$, we conclude that
$$|f_{\delta,t}(x)-f_{t}(x)|=|\mathbb{E}_{v\in \mathbb{B}}[f_t(x+\delta v)-f_t(x)]|
\leq L\mathbb{E}[\|v\|]\leq\delta L.$$

Then we obtain that
\begin{align*}
\label{eq:38}
&\mathbb{E}\left[\sum_{t=0}^{T} f_{t}\left(x_{t}\right)-\sum_{t=0}^{T} f_{t}(x)\right]\\
=&\mathbb{E}\left[\sum_{t=0}^{T} f_{\delta, t}\left(x_{t}\right)-\sum_{t=0}^{T} f_{\delta, t}(x) +\sum_{t=0}^{T}\left(f_{t}(x_{t})-f_{\delta, t}(x_{t})\right)-\sum_{t=0}^{T}\left(f_{t}(x)-f_{\delta, t}(x)\right)\right]\\ \leq & \frac{1}{2 \eta}\left\|x_{0}-x\right\|^{2}+\frac{\eta}{2} \mathbb{E}\left[\sum_{t=0}^{T}\left\|\tilde{g}_{t}\left(x_{t}\right)\right\|^{2}\right]+2 L \delta T.\numberthis
\end{align*}
Then plugging $x^*=\arg\min_{x\in\mathcal{X}} \sum_{t=0}^T f_t(x)$ into \eqref{eq:38} and according to the definition of regret in \eqref{eq.regret}, we can reach the conclusion. 
\end{proof}

\subsection{Derivation of \eqref{eq:Qt}}\label{appendix_a.2}
\begin{proof}
\allowdisplaybreaks
From the definition of $\tilde{g}_{t}^{(1)}(x_{t})$, we have
\begin{align*}
    \|\tilde{g}_{t}^{(1)}(x_{t})\|^{2}&=\frac{d^2}{\delta^2}\left(f_t(w_t)-f_{t-1}(w_{t-1})\right)^2\|u_t\|^2\\
    &\leq \frac{d^2}{\delta^2}\frac{\left(f_t(w_t)-f_{t-1}(w_{t-1})\right)^2}{\|w_t-w_{t-1}\|^2}\|w_t-w_{t-1}\|^2\\
    &\leq \frac{2d^2}{\delta^2}\frac{\left(f_t(w_t)-f_{t-1}(w_{t-1})\right)^2}{\|w_t-w_{t-1}\|^2}(\|x_t-x_{t-1}\|^2+\delta^2\|u_t-u_{t-1}\|^2)\\
    &\leq \frac{2d^2}{\delta^2}\frac{\left(f_t(w_t)-f_{t-1}(w_{t-1})\right)^2}{\|w_t-w_{t-1}\|^2}(\eta^2\|\tilde{g}_{t-1}^{(1)}(x_{t-1})\|^2+4\delta^2)\\
    &=2d^2\frac{\left(f_t(w_t)-f_{t-1}(w_{t-1})\right)^2}{\|w_t-w_{t-1}\|^2}\left(4+\frac{\eta^2}{\delta^2}\|\tilde{g}_{t-1}^{(1)}(x_{t-1})\|^2\right)
    \numberthis
\end{align*}
where the first inequality is because $\|u_t\|=1$, the second inequality comes from the fact that
\begin{align}\label{eq.pf.18-1}
        \|w_t-w_{t-1}\|^2&=\|x_t-x_{t-1}+\delta(u_t-u_{t-1})\|^2\nonumber\\
        &\leq2(\|x_t-x_{t-1}\|^2+\delta^2\|u_t-u_{t-1}\|^2)
\end{align}
while the third inequality follows from the update in \eqref{eq:sgd} and the relation $\|u_t-u_{t-1}\|^2\leq4$. 
\end{proof}

\subsection{Proof of Lemma \ref{lm:temp}}
\allowdisplaybreaks

\begin{proof}
\allowdisplaybreaks
With the choice of $\eta$ and $\delta$, we have that $\frac{\eta^2}{\delta^2}=\frac{1}{d^2L^2}$. Then using \eqref{eq:Qt}, we obtain that
\begin{align*}
    \|\tilde{g}_{t}^{(1)}(x_{t})\|^{2} &\leq2d^2\frac{|f_t(w_t)-f_{t-1}(w_{t-1})|^2}{\|w_t-w_{t-1}\|^2}\left(4+\frac{\eta^2}{\delta^2}\|\tilde{g}_{t-1}^{(1)}(x_{t-1})\|^2\right)\\
    &< 2d^2\frac{L^2}{10d}\left(4+\frac{1}{d^2L^2}\|\tilde{g}_{t-1}^{(1)}(x_{t-1})\|^2\right)\\
    &\leq \frac{4}{5}dL^2+\frac{1}{5d}\|\tilde{g}_{t-1}^{(1)}(x_{t-1})\|^2\\
    &\leq \frac{4}{5}d L^2+\frac{1}{5}d L^2=d L^2
    \numberthis
\end{align*}
where the second inequality comes from the condition. 
\end{proof}

Next, we prove Lemma \ref{lm:bt-bias}, Lemma \ref{clip2bound} and Theorem \ref{thm:lazoc} by dividing them into Section \ref{A4_laozb} for LAZOb and Section \ref{A5_lazoa} for LAZOa. 

\subsection{LAZOb estimator}\label{A4_laozb}
In this section, we present the proof for LAZOb estimator.

\subsubsection{Bias for LAZO$\mathrm{b}$}
To be self-contained, we restate the LAZOb part of Lemma \ref{lm:bt-bias} as follows.

\begin{lemma}
\label{lm:bt-bias_re}
Under Assumptions \ref{as1}--\ref{as_sym1}, the bias of the LAZOb gradient estimator is 
\begin{equation}
\mathbb{E}\left[\tilde{g}_t^b(x_t)|x_t\right]=\nabla f_{\delta,t}(x_t)+\mathbb{E}\left[b_t\Big|x_t\right]
\end{equation}
where $b_t$ is defined in \eqref{lazoexpect}.
Moreover, if we use $D={\cal O}(\sqrt{d}L)$, then $\sum_{t=0}^T\mathbb{E}[\|b_t\|]={\cal O}(\sqrt{dT})$. 
\end{lemma}

\begin{proof}
Recall $A_t\triangleq \{u_t|D_t^b(w_t, w_{t-1})\leq D\}$ denote the region where LAZOb uses the one-point query and $\bar{A}_t$ denote the complementary set of $A_t$ where LAZOb uses the two-point query. 
We have
\begin{align}
\nonumber
    &\mathbb{E}\left[\tilde{g}_t^b(x_t)\Big|x_t\right]\\
    =&\mathbb{E}\left[\frac{du_t}{\delta}(f_t(x_t+\delta u_t)-f_{t-1}(x_{t-1}+\delta u_{t-1}))\mathbf{1}_{A_t}\Big|x_t\right]\\\nonumber
    &+\mathbb{E}\left[\frac{du_t}{2\delta}(f_t(x_t+\delta u_t)-f_t(x_t-\delta u_t))\mathbf{1}_{\bar{A}_t}\Big|x_t\right]\\\nonumber
    =&\mathbb{E}\left[\frac{du_t}{\delta}(f_t(x_t+\delta u_t)-f_{t-1}(x_{t-1}+\delta u_{t-1}))\mathbf{1}_{A_t^s}\Big|x_t\right]\\\nonumber
    &+\mathbb{E}\left[\frac{du_t}{2\delta}(f_t(x_t+\delta u_t)-f_t(x_t-\delta u_t))\mathbf{1}_{\bar{A}_t^s}\Big|x_t\right]\\\nonumber
    &+\mathbb{E}\left[\frac{du_t}{2\delta}(f_t(x_t+\delta u_t)-2f_{t-1}(x_{t-1}+\delta u_{t-1})+f_t(x_t-\delta u_t))\mathbf{1}_{A_t\backslash A_t^{s}}\Big|x_t\right]\\\nonumber
    =&\mathbb{E}\left[\frac{du_t}{\delta}f_t(x_t+\delta u_t)\mathbf{1}_{A_t^s}\Big|x_t\right]+\mathbb{E}\left[\frac{du_t}{\delta}f_t(x_t+\delta u_t)\mathbf{1}_{\bar{A}_t^s}\Big|x_t\right]\\\nonumber
    &+\mathbb{E}\left[\frac{du_t}{2\delta}(f_t(x_t+\delta u_t)-2f_{t-1}(x_{t-1}+\delta u_{t-1})+f_t(x_t-\delta u_t))\mathbf{1}_{A_t\backslash A_t^{s}}\Big|x_t\right]\\
    =&\nabla f_{\delta,t}(x_t)+\mathbb{E}\left[\frac{du_t}{2\delta}(f_t(x_t+\delta u_t)-2f_{t-1}(x_{t-1}+\delta u_{t-1})+f_t(x_t-\delta u_t))\mathbf{1}_{A_t\backslash A_t^{s}}\Big|x_t\right]
    \label{eq:bias30}
\end{align}
where $\bar{A}_t, \bar{A}_t^s$ and $A_t\backslash A_t^s$ denote the complementary set of $A_t$, $A_t^{s}$ and the difference of sets $A_t$ and $A_t^s$. The first two terms in the third equality are due to the symmetricity of the sets $A_t^s$ and $\bar{A}_t^s$.

From equation \eqref{eq:bias30}, we can get that 
\begin{align*}
\label{btbound}
    \mathbb{E}\left[\|b_t\|\right]\leq&\frac{d}{2\delta}\mathbb{E}\left[\left|2(f_t(x_t+\delta u_t)-f_{t-1}(x_{t-1}+\delta u_{t-1}))+f_t(x_t-\delta u_t)-f_t(x_t+\delta u_t)\right|\mathbf{1}_{A_t\backslash A_t^s}\right]\\
    \leq&\frac{d}{\delta}\mathbb{E}\left[\left|f_t(x_t+\delta u_t)-f_{t-1}(x_{t-1}+\delta u_{t-1})\right|\mathbf{1}_{A_t\backslash A_t^s}\right]\\
    &+\frac{d}{2\delta}\mathbb{E}\left[|f_t(x_t+\delta u_t)-f_t(x_t-\delta u_t)|\mathbf{1}_{A_t\backslash A_t^s}\right]\\
    \numberthis
    \leq& (\frac{D\eta dL}{\delta}+dL)\mathbb{P}(A_t\backslash A_t^s)
\end{align*}
where the first inequality is derived from adding and subtracting $f_t(x_t+\delta u_t)$; the second inequality is because $\mathbb{E}\left[|X+Y|\right]\leq\mathbb{E}\left[|X|\right]+\mathbb{E}\left[|Y|\right]$; the first term in the third inequality is given by the fact that if $u\in A_t\backslash A_t^s\subseteq A_t$, then  $\left|f_t(x_t+\delta u_t)-f_{t-1}(x_{t-1}+\delta u_{t-1})\right|\leq D\eta L$ and the second term is due to the Lipschitz condition.

Thus, plugging $\eta=\frac{R}{L\sqrt{dT}}$, $\delta=R\sqrt{\frac{d}{T}}$, $D={\cal O}(\sqrt{d}L)$ and $\sum_{t=0}^T\mathbb{P}(A_t\backslash A_t^s)={\cal O}(\sqrt{T}/\sqrt{d})$ to \eqref{btbound}, we can get that $\sum_{t=0}^T\mathbb{E}\left[\|b_t\|\right]={\cal O}(\sqrt{dT})$. 
\end{proof}

\subsubsection{The second moment bound for LAZOb}

\begin{lemma}[{\citep[Lemma 9]{shamir2017optimal}}]
For any function $h : \mathbb{R}^d \rightarrow \mathbb{R}$ which is $L$-Lipschitz continuous, it holds that if $u \in \mathbb{R}^d$ is uniformly distributed on the Euclidean unit sphere, then there exists a constant $c={\cal O}(1)$ such that 
\begin{equation}
    \label{eq:concentration}
    \sqrt{\mathbb{E}\left[(h(u)-\mathbb{E}[h(u)])^{4}\right]} \leq c \frac{L^{2}}{d}.
\end{equation}
\label{sham}
\end{lemma}
\begin{lemma}
\label{lm:con}
Under Assumption \ref{as1}, for any symmetric set $A \subseteq \mathbb{R}^d$ with respect to $u \in \mathbb{R}^d$, any given $x \in \mathbb{R}^d$ and any given $t$, we have that 
\begin{equation}
    \mathbb{E}_u\left[\frac{d^2}{4\delta^2}(f_t(x+\delta u)-f_t(x-\delta u))^2\mathbf{1}_{A}\right]\leq cdL^2\sqrt{\mathbb{P}(A)}.
\end{equation}
\end{lemma}

\begin{proof}
First, it follows that for any $\alpha\in\mathbb{R}$,
\begin{align*}
    &\mathbb{E}_u\left[\frac{d^2}{4\delta^2}(f_t(x+\delta u)-f_t(x-\delta u))^2\mathbf{1}_{A}\right]\\
    =&\mathbb{E}_u\left[\frac{d^2}{4\delta^2}(f_t(x+\delta u)-\alpha-f_t(x-\delta u)+\alpha)^2\mathbf{1}_{A}\right]\\
    \leq &\frac{d^2}{2\delta^2}\mathbb{E}_u\left[((f_t(x+\delta u)-\alpha)^2+(f_t(x-\delta u)-\alpha)^2)\mathbf{1}_A\right]\\
    \leq &\frac{d^2}{2\delta^2}\mathbb{E}_u\left[(f_t(x+\delta u)-\alpha)^2\mathbf{1}_A\right]+\frac{d^2}{2\delta^2}\mathbb{E}_u\left[(f_t(x+\delta u)-\alpha)^2\mathbf{1}_A\right]\\
    =&\frac{d^2}{\delta^2}\mathbb{E}_u\left[(f_t(x+\delta u)-\alpha)^2\mathbf{1}_A\right]\\
    \leq& \frac{d^2}{\delta^2}\sqrt{\mathbb{E}_u\left[(f_t(x+\delta u)-\alpha)^4\right]}\sqrt{\mathbb{E}_u\left[\mathbf{1}_A\right]}\\
    \leq&\frac{d^2}{\delta^2}c\frac{\delta^2L^2}{d}\sqrt{\mathbb{P}(A)}=cdL^2\sqrt{\mathbb{P}(A)}
\end{align*}
where the first inequality is due to $(X+Y)^2\leq 2(X^2+Y^2)$; the second inequality is due to the symmetric distribution of $u$ and the symmetricity of $A$; the third inequality is because $\mathbb{E}\left[XY\right]\leq\sqrt{\mathbb{E}\left[X^2\right]}\sqrt{\mathbb{E}\left[Y^2\right]}$. If we choose $\alpha=\mathbb{E}_u\left[f_t(x_t+\delta u)\right]$ and apply Lemma \ref{sham}, then the last inequality holds since $f_t(x+\delta u)$ is $L\delta$-Lipschitz. This completes the proof. 
\end{proof}

To be self-contained, we restate the LAZOb part of Lemma \ref{clip2bound} by the following lemma and prove it.

\begin{lemma}
\label{clip2bound_re}
The second moment bound of the gradient estimator $\tilde{g}_t^b(x_t)$ satisfies that there exists some constant $c$ such that
\begin{align}
\label{eq:clip2_supp}
    \mathbb{E}\left[\|\tilde{g}_t^b(x_t)\|^2\Big|x_t\right]\leq \frac{D^2\eta^2d^2L^2}{\delta^2} +cdL^2.
\end{align}
\end{lemma}
\begin{proof}
Using the definition of LAZOb estimator in \eqref{eq:clipping}, we have
\begin{align}
\nonumber
    \mathbb{E}\left[\|\tilde{g}_t^b(x_t)\|^2\Big|x_t\right]=&\mathbb{E}\left[\frac{d^2}{\delta^2}(f_t(x_t+\delta u_t)-f_{t-1}(x_{t-1}+\delta u_{t-1}))^2\mathbf{1}_{A_t}\Big|x_t\right]\\\nonumber
    &+\mathbb{E}\left[\frac{d^2}{4\delta^2}(f_t(x_t+\delta u_t)-f_t(x_t-\delta u_t))^2\mathbf{1}_{\bar{A}_t}\Big|x_t\right]\\\nonumber
    \leq&\mathbb{E}\left[\frac{d^2}{\delta^2}(f_t(x_t+\delta u_t)-f_{t-1}(x_{t-1}+\delta u_{t-1}))^2\mathbf{1}_{A_t}\Big|x_t\right]\\\nonumber
    &+\mathbb{E}\left[\frac{d^2}{4\delta^2}(f_t(x_t+\delta u_t)-f_t(x_t-\delta u_t))^2\Big|x_t\right]\\
    \leq& \frac{D^2\eta^2d^2L^2}{\delta^2} +cdL^2
    \label{eq:LAZOb_second}
\end{align}
where the first inequality is due to $\|u_t\|=1$ and $\mathbf{1}_{A_t}\times\mathbf{1}_{\bar{A}_t}=0$; the second term in the first inequality is due to $\mathbf{1}_{\bar{A}_t}\leq 1$; and the last term is derived from Lemma \ref{lm:con} and $\mathbb{P}(A_t)\leq 1$. 
\end{proof}

\subsubsection{The regret bound for LAZO$\mathrm{b}$}
To be self-contained, we restate the LAZOb part of Theorem \ref{thm:lazoc} as follows.

\begin{theorem}
\label{thm:lazoc_re}
Under Assumptions \ref{as1}--\ref{as_sym1}, we run LAZOb for $T$ iterations with $\eta=\frac{R}{L\sqrt{dT}}$, and $\delta=R\sqrt{\frac{d}{T}}$. 
If we use $D={\cal O}(\sqrt{d}L)$ for LAZOb, then the regret for LAZOb satisfy
\begin{align*}
\mathbb{E}[{\cal R }_T(\texttt{LAZOb})]= {\cal O}(\sqrt{dT}).
\end{align*}
\end{theorem}

\begin{proof}
 
Using $\mathbb{E}\left[\tilde{g}_t^b(x_t)|x_t\right]=\nabla f_{\delta,t}(x_t)+\mathbb{E}\left[b_t\big|x_t\right]$ in \eqref{eq.unbias} and taking expectation of both sides of \eqref{eq:32}, we get for all $x\in\mathcal{X}$,
\begin{align}
\mathbb{E}\left[f_{\delta, t}\left(x_{t}\right)-f_{\delta, t}(x)\right] &\leq\mathbb{E}\left[\left\langle\mathbb{E}\left[\tilde{g}_t^b(x_t)-b_t|x_t\right], x_{t}-x\right\rangle\right]\nonumber\\
&=\mathbb{E}\left[\mathbb{E}\left[\left\langle\tilde{g}_t^b(x_t)-b_t, x_{t}-x\right\rangle|x_t\right]\right]\nonumber\\
&=\mathbb{E}\left[\left\langle\tilde{g}_t^b(x_t), x_{t}-x\right\rangle\right]-\mathbb{E}\left[\langle b_t, x_{t}-x\rangle\right]\nonumber\\
&\leq\mathbb{E}\left[\left\langle\tilde{g}_t^b(x_t), x_{t}-x\right\rangle\right]+R\mathbb{E}\left[\|b_t\|\right]\nonumber\\
&\leq\mathbb{E}\left[\left\langle\tilde{g}_t^b(x_t), x_{t}-x\right\rangle\right]+4(\frac{D\eta dL}{\delta}+dL)R\mathbb{P}(A_t\backslash A_t^s).
\label{eq:btbias}
\end{align}
Taking expectations on both sides of inequality \eqref{eq:31} with respect to $u_t$, substituting the resulting bound into \eqref{eq:btbias} and summing from $t=0$ to $T$, we obtain that
\begin{align}
\label{eq:37b}
\nonumber
\mathbb{E}\left[\sum_{t=0}^{T} f_{\delta, t}\left(x_{t}\right)-\sum_{t=0}^{T} f_{\delta,t}(x)\right] &\leq \frac{1}{2 \eta}\left\|x_{0}-x\right\|^{2}+\frac{\eta}{2} \mathbb{E}\left[\sum_{t=0}^{T}\left\|\tilde{g}_{t}^b\left(x_{t}\right)\right\|^{2}\right]\\
&+4(\frac{D\eta dL}{\delta}+dL)R\sum_{t=0}^T\mathbb{P}(A_t\backslash A_t^s).
\end{align}
Then similar to \eqref{eq:38}, we obtain that
\begin{align*}
\label{eq:gtc_regret0}
&\mathbb{E}\left[\sum_{t=0}^{T} f_{t}\left(x_{t}\right)-\sum_{t=0}^{T} f_{t}(x)\right]\\
=&\mathbb{E}\left[\sum_{t=0}^{T} f_{\delta, t}\left(x_{t}\right)-\sum_{t=0}^{T} f_{\delta, t}(x) +\sum_{t=0}^{T}\left(f_{t}(x_{t})-f_{\delta, t}(x_{t})\right)-\sum_{t=0}^{T}\left(f_{t}(x)-f_{\delta, t}(x)\right)\right]\\ \leq &\frac{1}{2 \eta}\left\|x_{0}-x\right\|^{2}+\frac{\eta}{2} \mathbb{E}\left[\sum_{t=0}^{T}\left\|\tilde{g}_{t}^b\left(x_{t}\right)\right\|^{2}\right]+4(\frac{D\eta dL}{\delta}+dL)R\sum_{t=0}^T\mathbb{P}(A_t\backslash A_t^s)+2 L \delta T.\numberthis
\end{align*}
Plugging $\eta=\frac{R}{L\sqrt{dT}}$, $\delta=R\sqrt{\frac{d}{T}}$ to the second moment bound in Lemma \ref{clip2bound}, then we can get that
\begin{align}
\label{lazobsc}
    \mathbb{E}\left[\|\tilde{g}_t^b(x_t)\|^2\right]\leq D^2+cdL^2.
\end{align}
Then plugging $x^*=\arg\min_{x\in\mathcal{X}} \sum_{t=0}^T f_t(x)$ and the second moment bound for the $\tilde{g}_t^b(x_t)$ into \eqref{eq:gtc_regret}, and according to the definition of regret in \eqref{eq.regret}, we can reach the conclusion that
\begin{align*}
\label{regret_lazob}
    \!\!\mathbb{E}[{\cal R }_T(\texttt{LAZOb})]&\leq\frac{R^{2}}{2 \eta}+\frac{\eta}{2}(D^2+cdL^2)T+4(D+dL)R\sum_{t=0}^T\mathbb{P}(A_t\backslash A_t^s)+2 L \delta T.\numberthis
\end{align*}
Then plugging $\eta=\frac{R}{L\sqrt{dT}}$, $\delta=R\sqrt{\frac{d}{T}}$, $D={\cal O}(\sqrt{d}L)$ and $\sum_{t=0}^T\mathbb{P}(A_t\backslash A_t^s)={\cal O}(\sqrt{T}/\sqrt{d})$ into \eqref{regret_lazob}, we can get that
\begin{align*}
    \!\!\mathbb{E}[{\cal R }_T(\texttt{LAZOb})]= {\cal O}(R\sqrt{dT}).\numberthis
\end{align*}
This completes the proof.
\end{proof}
\vspace{0.05cm}
\subsection{LAZOa estimator}\label{A5_lazoa}
\vspace{0.1cm}
\allowdisplaybreaks
Similar to the derivation for LAZOb, we can denote $\tilde{A}_t$ in the subsection to be the the region triggering one-point in LAZOa, that is $\tilde{A}_t\triangleq \{u_t|D_t^a(w_t, w_{t-1})\leq D\}$. 

First, we prove that under some conditions, the LAZOa gradient estimator is bounded.

\begin{lemma}
\label{lazoabound}
Under Assumptions \ref{as1}--\ref{as2}, we run \eqref{eq:sgd} with $\tilde{g}_t(x_t)$ and set $\eta=\frac{R}{L\sqrt{dT}}$, $\delta=R\sqrt{\frac{d}{T}}$. If $D<\frac{L}{\sqrt{10}}$, then for all $t$, $\|\tilde{g}^a_t(x_t)\|\leq d^2L^2$. 
\end{lemma}
\begin{proof}
Since the first step of LAZOa is using the two-point estimator, then $\|\tilde{g}^a_0(x_0)\|\leq d^2L^2$. 
With the choice of $\eta$ and $\delta$, we have that $\frac{\eta^2}{\delta^2}=\frac{1}{d^2L^2}$. Then using \eqref{eq:Qt}, and assuming $\|\tilde{g}_{t}^a(x_{t})\|^{2}\leq d^2L^2$ holds for any $t<m$, then 
\begin{align*}
    \|\tilde{g}_{m}^a(x_{m})\|^{2}&=\|\tilde{g}_{m}^{(1)}(x_{m})\mathbf{1}_{\tilde{A}_t}+\tilde{g}_{m}^{(2)}(x_{m})\mathbf{1}_{\bar{\tilde{A}}_t}\|^{2}\\
    &=\|\tilde{g}_{m}^{(1)}(x_{m})\|^{2}\mathbf{1}_{\tilde{A}_t}+\|\tilde{g}_{m}^{(2)}(x_{m})\|^{2}\mathbf{1}_{\bar{\tilde{A}}_t}\\ &\leq2d^2D^2\left(4+\frac{\eta^2}{\delta^2}\|\tilde{g}_{m-1}^a(x_{m-1})\|^2\right)\mathbf{1}_{\tilde{A}_t}+d^2L^2\mathbf{1}_{\bar{\tilde{A}}_t}\\
    &< 2d^2\frac{L^2}{10}\left(4+\frac{1}{d^2L^2}\|\tilde{g}_{m-1}^a(x_{m-1})\|^2\right)\mathbf{1}_{\tilde{A}_t}+d^2L^2\mathbf{1}_{\bar{\tilde{A}}_t}\\
    &\leq d^2L^2\mathbf{1}_{\tilde{A}_t}+d^2L^2\mathbf{1}_{\bar{\tilde{A}}_t}=d^2L^2.
    \numberthis
\end{align*}
Thus we can arrive at the conclusion using induction.
\end{proof}
\vspace{0.05cm}
\subsubsection{Bias for LAZO$\mathrm{a}$}
\vspace{0.1cm}
To be self-contained, we restate the LAZOa part of Lemma \ref{lm:bt-bias} as follows.
\begin{lemma}
If $\eta=\frac{R}{L\sqrt{dT}}$, $\delta=R\sqrt{\frac{d}{T}}$ and $D={\cal O}(\frac{L}{\sqrt{d}})<\frac{L}{\sqrt{10}}$, then we have 
\begin{equation}
\sum_{t=0}^T\mathbb{E}\left[\|\tilde{b}_t\|\right]={\cal O}(\sqrt{dT}). 
\end{equation}
\end{lemma}
\begin{proof}
It is easy to see that \eqref{eq:bias30} holds for $\tilde{g}_t^a(x_t)$ if changing $A_t\backslash A_t^s$ to $\tilde{A}_t\backslash \tilde{A}_t^s$. 

Then according to the definition of $\tilde{A}_t$, we can get that 
\begin{align*}\label{bt2_bound}
    &\mathbb{E}\left[\|\tilde{b}_t\|\right]\\
    \leq&\frac{d}{2\delta}\mathbb{E}\left[\left|2(f_t(x_t+\delta u_t)-f_{t-1}(x_{t-1}+\delta u_{t-1}))+f_t(x_t-\delta u_t)-f_t(x_t+\delta u_t)\right|\mathbf{1}_{\tilde{A}_t\backslash \tilde{A}_t^s}\right]\\
    \leq& \frac{d}{\delta}\mathbb{E}\left[|f_t(x_t+\delta u_t)-f_{t-1}(x_{t-1}+\delta u_{t-1})|\mathbf{1}_{\tilde{A}_t\backslash \tilde{A}_t^s}\right]+\frac{d}{\delta}\mathbb{E}\left[|f_t(x_t+\delta u_t)-f_t(x_t-\delta u_t)|\mathbf{1}_{\tilde{A}_t\backslash \tilde{A}_t^s}\right]\\
    \leq& dL\mathbb{P}(\tilde{A}_t\backslash \tilde{A}_t^s)+\frac{d}{\delta}\mathbb{E}\left[|f_t(x_t+\delta u_t)-f_{t-1}(x_{t-1}+\delta u_{t-1})|\mathbf{1}_{\tilde{A}_t\backslash \tilde{A}_t^s}\right]\\
    \leq& dL\mathbb{P}(\tilde{A}_t\backslash \tilde{A}_t^s)+\frac{dD}{\delta}\mathbb{E}\left[\|x_t-x_{t-1}+\delta (u_t-u_{t-1})\|\mathbf{1}_{\tilde{A}_t\backslash \tilde{A}_t^s}\right]\\
    \leq& dL\mathbb{P}(\tilde{A}_t\backslash \tilde{A}_t^s)+2dD\mathbb{P}(\tilde{A}_t\backslash \tilde{A}_t^s)+\frac{dD\eta}{\delta}\mathbb{E}\left[\|\tilde{g}_{t-1}^a(x_{t-1})\|\mathbf{1}_{\tilde{A}_t\backslash \tilde{A}_t^s}\right]\\
    \leq& dL\mathbb{P}(\tilde{A}_t\backslash \tilde{A}_t^s)+2dD\mathbb{P}(\tilde{A}_t\backslash \tilde{A}_t^s)+\frac{d^2DL\eta}{\delta}\mathbb{P}(\tilde{A}_t\backslash \tilde{A}_t^s)
    \numberthis
\end{align*}
where the first inequality is due to $\mathbb{E}\left[|X+Y|\right]\leq \mathbb{E}\left[|X|\right]+\mathbb{E}\left[|Y|\right]$; the first term in the third inequality is because the Lipschitz condition; the fourth inequality is derived from the fact that if $u\in \tilde{A}_t\backslash \tilde{A}_t^s\subseteq \tilde{A}_t$, then $\left|f_t(x_t+\delta u_t)-f_{t-1}(x_{t-1}+\delta u_{t-1})\right|\leq D\|x_t-x_{t-1}+\delta (u_t-u_{t-1})\|$ and the second term is due to the Lipschitz condition; the second term in the fifth inequality is due to $\mathbb{E}\left[\|X+Y\|\right]\leq \mathbb{E}\left[|X|\right]+\mathbb{E}\left[|Y|\right]$, $\|u_t-u_{t-1}\|\leq 4$ and $\|x_t-x_{t-1}\|\leq \eta \|\tilde{g}_{t-1}^a(x_{t-1})\|$; the last inequality is according to Lemma \ref{lazoabound}.

Thus, plugging $\eta=\frac{R}{L\sqrt{dT}}$, $\delta=R\sqrt{\frac{d}{T}}$, $D={\cal O}(\frac{L}{\sqrt{d}})<\frac{L}{\sqrt{10}}$ and $\sum_{t=0}^T\mathbb{P}(\tilde{A}_t\backslash \tilde{A}_t^s)={\cal O}(\sqrt{T}/\sqrt{d})$ to \eqref{bt2_bound}, we can get that $\sum_{t=0}^T\mathbb{E}\left[\|\tilde{b}_t\|\right]={\cal O}(\sqrt{dT})$. 
\end{proof}

\subsubsection{The second moment bound for LAZO$\mathrm{a}$}
To be self-contained, we restate the LAZOa part of Lemma \ref{clip2bound} as follows. 
\begin{lemma}
\label{clip2bound_rere}
Under Assumptions \ref{as1}--\ref{as3}, the second moment bound of the gradient estimator $\tilde{g}_t^a(x_t)$ satisfy that there exists a constant $c={\cal O}(1)$ such that
\begin{align}
    \mathbb{E}\left[\|\tilde{g}_t^a(x_t)\|^2\Big|x_t\right]\leq   cdL^2+\frac{2d^2D^2\eta^2}{\delta^2}\|\tilde{g}_{t-1}^a(x_{t-1})\|^2+8d^2D^2.
\end{align}

\end{lemma}

\begin{proof}
Using the definition of $\tilde{g}_t^a(x_t)$, we have
\begin{align}
\nonumber
    \mathbb{E}\left[\|\tilde{g}_t^a(x_t)\|^2\Big|x_t\right]&=\mathbb{E}\left[\frac{d^2}{\delta^2}(f_t(x_t+\delta u_t)-f_{t-1}(x_{t-1}+\delta u_{t-1}))^2\mathbf{1}_{\tilde{A}_t}\Big|x_t\right]\\\nonumber
    &+\mathbb{E}\left[\frac{d^2}{4\delta^2}(f_t(x_t+\delta u_t)-f_t(x_t-\delta u_t))^2\mathbf{1}_{\bar{\tilde{A}}_t}\Big|x_t\right]\\\nonumber
    &\leq\mathbb{E}\left[\frac{d^2}{\delta^2}(f_t(x_t+\delta u_t)-f_{t-1}(x_{t-1}+\delta u_{t-1}))^2\mathbf{1}_{\tilde{A}_t}\Big|x_t\right]\\\nonumber
    &+\mathbb{E}\left[\frac{d^2}{4\delta^2}(f_t(x_t+\delta u_t)-f_t(x_t-\delta u_t))^2\mathbf{1}_{\bar{\tilde{A}}_t^s}\Big|x_t\right]\\
    &\leq cdL^2\sqrt{\mathbb{P}(\bar{\tilde{A}}_t^s\big|x_t)}+\mathbb{E}\left[\frac{d^2}{\delta^2}(f_t(x_t+\delta u_t)-f_{t-1}(x_{t-1}+\delta u_{t-1}))^2\mathbf{1}_{\tilde{A}_t}\Big|x_t\right]\nonumber\\
    &\leq cdL^2+\frac{d^2D^2}{\delta^2}\mathbb{E}\left[\|x_t-x_{t-1}+\delta(u_t-u_{t-1})\|^2\mathbf{1}_{\tilde{A}_t}\Big|x_t\right]\nonumber\\
    &\leq cdL^2+\frac{2d^2D^2}{\delta^2}\mathbb{E}\left[\|x_t-x_{t-1}\|^2\Big|x_t\right]+8d^2D^2\nonumber\\
    &\leq cdL^2+\frac{2d^2D^2\eta^2}{\delta^2}\|\tilde{g}_{t-1}^a(x_{t-1})\|^2+8d^2D^2
    \label{eq:LAZOa_second}
\end{align}
where the first inequality is due to $\|u_t\|=1$ and $\mathbf{1}_{\tilde{A}_t}\times\mathbf{1}_{\bar{\tilde{A}}_t}=0$; the second inequality is due to $\bar{\tilde{A}}_t\subset \bar{\tilde{A}}_t^s$; the second inequality is due to Lemma \ref{sham}; the third inequality is derived from the fact that if $u\in \tilde{A}_t\backslash \tilde{A}_t^s\subseteq \tilde{A}_t$, then $\left|f_t(x_t+\delta u_t)-f_{t-1}(x_{t-1}+\delta u_{t-1})\right|\leq D\|x_t-x_{t-1}+\delta (u_t-u_{t-1})\|$; the fourth inequality is due to $\mathbb{E}\left[\|X+Y\|\right]\leq \mathbb{E}\left[|X|\right]+\mathbb{E}\left[|Y|\right]$, $\|u_t-u_{t-1}\|\leq 4$; the last inequality is due to $\|x_t-x_{t-1}\|\leq \eta \|\tilde{g}_{t-1}^a(x_{t-1})\|$. 
\end{proof}

\subsubsection{Regret for LAZO$\mathrm{a}$}
To be self-contained, we restate the LAZOa part of Theorem \ref{thm:lazoc} as follows.

\begin{theorem}
\label{thm:lazoc_rere}
Under Assumptions \ref{as1}--\ref{as_sym1} (replacing $A_t$ with $\tilde{A}_t$), we run LAZOa for $T$ iterations with $\eta=\frac{R}{L\sqrt{dT}}$, $\delta=R\sqrt{\frac{d}{T}}$. 
If we use $D={\cal O}(\frac{L}{\sqrt{d}})<\frac{L}{\sqrt{10}}$ for LAZOa, then the regret satisfy
\begin{align*}
\mathbb{E}[{\cal R }_T(\texttt{LAZOa})]= {\cal O}(\sqrt{dT}).
\end{align*}
\end{theorem}

\begin{proof}
First, plugging $\eta=\frac{R}{L\sqrt{dT}}$, $\delta=R\sqrt{\frac{d}{T}}$ to the \eqref{bt2_bound}, we can simplify it to
\begin{align*}
    \mathbb{E}\left[\|\tilde{b}_t\|\right]\leq d(L+3D)\mathbb{P}(\tilde{A}_t\backslash \tilde{A}_t^s).
\end{align*}

Similar to the derivation for \eqref{eq:btbias}, we can get that
\begin{equation}
\label{eq:45}
\mathbb{E}\left[\sum_{t=0}^{T} f_{\delta, t}\left(x_{t}\right)-\sum_{t=0}^{T} f_{\delta,t}(x)\right] \leq \frac{1}{2 \eta}\left\|x_{0}-x\right\|^{2}+\frac{\eta}{2} \mathbb{E}\left[\sum_{t=0}^{T}\left\|\tilde{g}_{t}^a\left(x_{t}\right)\right\|^{2}\right]+d(L+3D)R\sum_{t=0}^T\mathbb{P}(\tilde{A}_t\backslash \tilde{A}_t^s). 
\end{equation}
Then similar to \eqref{eq:38}, we obtain that
\begin{align*}
\label{eq:gtc_regret}
&\mathbb{E}\Bigg[\sum_{t=0}^{T} f_{t}\left(x_{t}\right)-\sum_{t=0}^{T} f_{t}(x)\Bigg]\\
=&\mathbb{E}\left[\sum_{t=0}^{T} f_{\delta, t}\left(x_{t}\right)-\sum_{t=0}^{T} f_{\delta, t}(x) +\sum_{t=0}^{T}\left(f_{t}(x_{t})-f_{\delta, t}(x_{t})\right)-\sum_{t=0}^{T}\left(f_{t}(x)-f_{\delta, t}(x)\right)\right]\\ \leq&\frac{1}{2 \eta}\left\|x_{0}-x\right\|^{2}+\frac{\eta}{2} \mathbb{E}\left[\sum_{t=0}^{T}\left\|\tilde{g}_{t}^a\left(x_{t}\right)\right\|^{2}\right]+d(L+3D)R\sum_{t=0}^T\mathbb{P}(\tilde{A}_t\backslash \tilde{A}_t^s)+2 L \delta T.\numberthis
\end{align*}

Second, plugging $\eta=\frac{R}{L\sqrt{dT}}$, $\delta=R\sqrt{\frac{d}{T}}$ to the \eqref{eq:LAZOa_second}, we can simplify the second moment bound of LAZOa estimator to
\begin{align}
\label{lazoasc}
    \mathbb{E}\left[\|\tilde{g}_t^a(x_t)\|^2\right]\leq 10d^2D^2+cdL^2.
\end{align}

Then plugging $x^*=\arg\min_{x\in\mathcal{X}} \sum_{t=0}^T f_t(x)$ and \eqref{lazoasc} into \eqref{eq:gtc_regret}, and according to the definition of regret in \eqref{eq.regret}, we can reach the conclusion that
\begin{align*}
\label{49}
    \!\!\mathbb{E}[{\cal R }_T(\texttt{LAZOa})]&\leq\frac{R^{2}}{2 \eta}+\frac{\eta}{2}(10d^2D^2+cdL^2)T+d(L+3D)R\sum_{t=0}^T\mathbb{P}(\tilde{A}_t\backslash \tilde{A}_t^s)+2 L \delta T.\numberthis
\end{align*}
Finally, plugging $\eta=\frac{R}{L\sqrt{dT}}$ and , $\delta=R\sqrt{\frac{d}{T}}$ and $D={\cal O}(\frac{L}{\sqrt{d}})$ and $\sum_{t=0}^T\mathbb{P}(\tilde{A}_t\backslash \tilde{A}_t^s)={\cal O}(\sqrt{T}/\sqrt{d})$ into \eqref{49}, we can get that
\begin{align*}
    \!\!\mathbb{E}[{\cal R }_T(\texttt{LAZOa})]= {\cal O}(R\sqrt{dT}).\numberthis
\end{align*}
This completes the proof.
\end{proof}

\vspace{0.2cm}
\section{Proofs of the results in convex stochastic optimization}
\vspace{0.2cm}
In this section, we will present the proofs of results in convex stochastic optimization.

With the definition of $f_t(x)=F(x,\xi_t)$, instead of minimizing the regret \eqref{eq.regret}, the goal is to minimize 
\begin{equation}\label{eq.regret2}
\overline{{\cal R}}_T({\cal A})\triangleq  f(\bar{x}_T)-f(x^*)~~~{\rm with}~~~x^*\in \min _{x \in \mathcal{X}} f(x) 
\end{equation}
where the average iterate is defined as $\bar{x}_T=\frac{1}{T}\sum_{t=0}^{T-1}x_t$. 

Then under Assumption \ref{as1}--\ref{as_sym1}, we can reach the same conclusion as that in the OCO case by the similar procedure.

First, similar to the derivation of \eqref{btbound} and \eqref{bt2_bound}, we can get that 
\begin{equation}
\mathbb{E}_{\xi_t,u_t}\left[\tilde{g}_t^a(x_t)\right]=\nabla f_{\delta}(x_t)+\mathbb{E}\left[\tilde{b}_t\Big|x_t\right]~~{\rm and}~~\mathbb{E}_{\xi_t,u_t}\left[\tilde{g}_t^b(x_t)\right]=\nabla f_{\delta}(x_t)+\mathbb{E}\left[b_t\Big|x_t\right]
\end{equation}
with
\begin{align*}
    \mathbb{E}\left[\|b_t\|\right]\leq (\frac{D\eta dL}{\delta}+dL)\mathbb{P}(A_t\backslash A_t^s)~~{\rm and}~~
    \mathbb{E}[\|\tilde{b}_t\|]\leq (dL+2dD+\frac{d^2DL\eta}{\delta})\mathbb{P}(\tilde{A}_t\backslash \tilde{A}_t^s).
\end{align*}

Second, the second moment bounds for LAZOa and LAZOb in Lemma \ref{clip2bound} still hold for convex stochastic optimization if plugging $f_t(x)=F(x,\xi_t)$ into the proof. 

Then, similar to the derivation of \eqref{eq:gtc_regret0} and \eqref{eq:gtc_regret}, for \textbf{LAZOa}, we have
\begin{align*}
&\mathbb{E}\left[\sum_{t=0}^{T} f(x_{t})-\sum_{t=0}^{T} f(x)\right]\\
=&\mathbb{E}\left[\sum_{t=0}^{T} f_{\delta}(x_{t})-\sum_{t=0}^{T} f_{\delta}(x) +\sum_{t=0}^{T}\left(f(x_{t})-f_{\delta}(x_{t})\right)-\sum_{t=0}^{T}\left(f(x)-f_{\delta}(x)\right)\right]\\
\leq&\frac{1}{2\eta}\|x_0-x\|^2+\frac{\eta}{2}\mathbb{E}\left[\sum_{t=0}^T \|\tilde{g}_t^a(x_t)\|^2\right]+R\mathbb{E}[\|\tilde{b}_t\|]+2L\delta T;\numberthis
\end{align*}
for \textbf{LAZOb}, we have that
\begin{align*}
\mathbb{E}\left[\sum_{t=0}^{T} f(x_{t})-\sum_{t=0}^{T} f(x)\right]
\leq\frac{1}{2\eta}\|x_0-x\|^2+\frac{\eta}{2}\mathbb{E}\left[\sum_{t=0}^T \|\tilde{g}_t^b(x_t)\|^2\right]+R\mathbb{E}[\|b_t\|]+2L\delta T.\numberthis
\end{align*}

Thus, we can reach the same conclusion.

\section{Proofs of the results in nonconvex stochastic optimization}\label{nonconvex-proof}

In this section, we will present the proofs of results in nonconvex stochastic optimization.

\subsection{Proof of Lemmas}
Notice that without feasible bounded set assumption, LAZO update \eqref{eq:sgd} will change to
\begin{equation}
x_{t+1}=x_{t}-\eta \tilde{g}_{t}(x_{t})
\label{eq:sgd2}
\end{equation}

\begin{lemma}[{\cite{gao2018information}[Lemma 4.1(b)]}]
\label{gao}
Under Assumption \ref{as_sym1}, for any $x\in\mathbb{R}^d$, we have
\begin{align}
&\left|f_{\delta}(x)-f(x)\right| \leq \frac{\mu \delta^{2}}{2}; \\ &\left\|\nabla f_{\delta}(x)-\nabla f(x)\right\| \leq \frac{\delta d \mu}{2},
\end{align}
where $f_\delta(x)\triangleq\mathbb{E}_{v\sim U(\mathbb{B})}\left[f(x+\delta v)\right]$ is a smoothed version of function $f$. 
\end{lemma}

\begin{lemma}
\label{lm:nvx}
Under Assumption \ref{as1}, \ref{as8}, and assume $\eta=\frac{1}{L\sqrt{dT}}$, $\delta=\sqrt{\frac{d}{T}}$ and $D<\frac{L}{\sqrt{10}}$ for \textbf{LAZOa}, then we have that
\begin{align*}
\sum_{t=0}^T\mathbb{E}\left[\|\nabla f(x_t)\|^2\right]&\leq \frac{2f(x_0)}{\eta}+\frac{\delta^2\mu^2d^2T}{4}+c\mu\eta dL^2T+8d^2D^2\eta\mu T+\frac{2\mu \eta^3 d^4D^2L^2T}{\delta^2}\\
&+2 L(dL+2dD+\frac{d^2DL\eta}{\delta})\sum_{t=1}^T\mathbb{P}(\tilde{A}_t\backslash \tilde{A}_t^s). 
\end{align*}
Similarly, for \textbf{LAZOb}, under Assumption \ref{as1}, \ref{as8}, we have that
\begin{align*}
\sum_{t=0}^T\mathbb{E}\left[\|\nabla f(x_t)\|^2\right]&\leq \frac{2f(x_0)}{\eta}+\frac{\delta^2\mu^2d^2T}{4}+\frac{\mu D^2\eta^3d^2L^2T}{\delta^2} +c\mu\eta dL^2T\\
&+2L(\frac{D\eta dL}{\delta}+dL)\sum_{t=1}^T\mathbb{P}(A_t\backslash A_t^s).
\end{align*}
\end{lemma}

\begin{proof}
First, similar to the derivation of \eqref{btbound} and \eqref{bt2_bound}, we can get that 
\begin{equation}
\mathbb{E}_{\xi_t,u_t}\left[\tilde{g}_t^a(x_t)\right]=\nabla f_{\delta}(x_t)+\mathbb{E}\left[\tilde{b}_t\Big|x_t\right]~~{\rm and}~~\mathbb{E}_{\xi_t,u_t}\left[\tilde{g}_t^b(x_t)\right]=\nabla f_{\delta}(x_t)+\mathbb{E}\left[b_t\Big|x_t\right]
\end{equation}
with
\begin{align*}
    \|\mathbb{E}\left[b_t|x_t\right]\|&\leq\mathbb{E}\left[\|b_t\||x_t\right]\leq (\frac{D\eta dL}{\delta}+dL)\mathbb{P}(A_t\backslash A_t^s|x_t)
\end{align*}
and 
\begin{align*}
    \|\mathbb{E}[\tilde{b}_t|x_t]\|&\leq\mathbb{E}[\|\tilde{b}_t\||x_t]\leq (dL+2dD+\frac{d^2DL\eta}{\delta})\mathbb{P}(\tilde{A}_t\backslash \tilde{A}_t^s|x_t).
\end{align*}
\textbf{For LAZOa}, taking expectation conditioned on $x_t$, we have that 
\begin{align*}
&\mathbb{E}_{\xi_t,u_t}\left[f(x_{t+1})\right]-f(x_t)\\
\leq& -\eta\langle\nabla f(x_t),\mathbb{E}_{\xi_t,u_t}\left[\tilde{g}_t^{a}(x_t)\right]\rangle+\frac{\eta^2\mu}{2}\mathbb{E}_{\xi_t,u_t}\left[\|\tilde{g}_t^{a}(x_t)\|^2\right]\\
=&-\eta\langle\nabla f(x_t),\nabla f_\delta(x_t)\rangle-\eta\langle\nabla f(x_t),\mathbb{E}\left[\tilde{b}_t|x_t\right]\rangle+\frac{\eta^2\mu}{2}\mathbb{E}_{\xi_t,u_t}\left[\|\tilde{g}_t^{a}(x_t)\|^2\right]\\
\leq&-\eta\langle\nabla f(x_t),\nabla f(x_t)+\nabla f_\delta(x_t)-\nabla f(x_t)\rangle+\eta L\|\mathbb{E}\left[\tilde{b}_t|x_t\right]\|\\
&+\frac{\mu\eta^2cdL^2}{2} +4d^2D^2\eta^2\mu+\frac{\mu\eta^4 d^2D^2 }{\delta^2}d^2L^2\\
\leq& -\frac{\eta}{2}\|\nabla f(x_t)\|^2+\frac{\eta}{2}\|\nabla f_\delta(x_t)-\nabla f(x)\|^2+\frac{\mu\eta^2cdL^2}{2} +4d^2D^2\eta^2\mu\\
&+\eta L(dL+2dD+\frac{d^2DL\eta}{\delta})\mathbb{P}(\tilde{A}_t\backslash \tilde{A}_t^s|x_t)+\frac{\mu\eta^4 d^2D^2 }{\delta^2}d^2L^2\\
\leq& -\frac{\eta}{2}\|\nabla f(x_t)\|^2+\frac{\eta\delta^2 d^2\mu^2}{8}+\frac{\mu\eta^2cdL^2}{2} +4d^2D^2\eta^2\mu+\frac{\mu\eta^4 d^2D^2 }{\delta^2}d^2L^2\\
&+\eta L(dL+2dD+\frac{d^2DL\eta}{\delta})\mathbb{P}(\tilde{A}_t\backslash \tilde{A}_t^s|x_t)\numberthis
\label{eq:68}
\end{align*}
where the first inequality is due to the fact that $f$ is $\mu$-smooth and \eqref{eq:sgd2}; the first equality and the second inequality is due to Lemma \ref{lm:bt-bias}; the third inequality is derived from \eqref{eq:LAZOa_second}; the fifth inequality is according to Lemma \ref{gao}. Then taking the total expectation and summing it from $t=0$ to $T$, we can reach the conclusion.

\textbf{For LAZOb}, taking expectation conditioned on $x_t$, we have that 
\begin{align*}\label{eq:68-2}
&\mathbb{E}_{\xi_t,u_t}\left[f(x_{t+1})\right]-f(x_t)\\
\leq &-\eta\langle\nabla f(x_t),\mathbb{E}_{\xi_t,u_t}\left[\tilde{g}_t^{b}(x_t)\right]\rangle+\frac{\eta^2\mu}{2}\mathbb{E}_{\xi_t,u_t}\left[\|\tilde{g}_t^{b}(x_t)\|^2\right]\\
=&-\eta\langle\nabla f(x_t),\nabla f_\delta(x_t)\rangle-\eta\langle\nabla f(x_t),\mathbb{E}\left[b_t|x_t\right]\rangle+\frac{\eta^2\mu}{2}\mathbb{E}_{\xi_t,u_t}\left[\|\tilde{g}_t^{b}(x_t)\|^2\right]\\
\leq&-\eta\langle\nabla f(x_t),\nabla f(x_t)+\nabla f_\delta(x_t)-\nabla f(x_t)\rangle+\eta L\|\mathbb{E}\left[b_t|x_t\right]\|\\
&+\frac{\mu D^2\eta^4d^2L^2}{2\delta^2} +\frac{c\mu\eta^2 dL^2}{2}\\
\leq& -\frac{\eta}{2}\|\nabla f(x_t)\|^2+\frac{\eta}{2}\|\nabla f_\delta(x_t)-\nabla f(x)\|^2+\eta L(\frac{D\eta dL}{\delta}+dL)\mathbb{P}(A_t\backslash A_t^s|x_t)\\
&+\frac{\mu D^2\eta^4d^2L^2}{2\delta^2} +\frac{c\mu\eta^2 dL^2}{2}\\
\leq& -\frac{\eta}{2}\|\nabla f(x_t)\|^2+\frac{\eta\delta^2 d^2\mu^2}{8}+\frac{\mu D^2\eta^4d^2L^2}{2\delta^2} +\frac{c\mu\eta^2 dL^2}{2}\\
&+\eta L(\frac{D\eta dL}{\delta}+dL)\mathbb{P}(A_t\backslash A_t^s|x_t)\numberthis
\end{align*}
where the first inequality follows that $f$ is $\mu$-smooth and \eqref{eq:sgd2}; the first equality and the second inequality is due to Lemma \ref{lm:bt-bias}; the third inequality is derived from Lemma \ref{clip2bound}; the fourth inequality is according to Lemma \ref{gao}. Thus, taking the total expectation and summing it from $t=0$ to $T$, we can get the conclusion.
\end{proof}

\subsection{Proof of Theorem \ref{them_ncvx}}
Under additional Assumption \ref{as_sym1} (replacing $A_t$ by $\tilde{A}_t$) and based on Lemma \ref{lm:nvx}, if we choose 
$\eta=\frac{1}{L\sqrt{dT}}$, $\delta=\sqrt{\frac{d}{T}}$ and use $D={\cal O}(\frac{L}{\sqrt{d}})<\frac{L}{\sqrt{10}}$ for LAZOa, we can get that 
\begin{align*}
    \mathbb{E}[{\cal R }_T^{\rm nc}(\texttt{LAZOa})]\leq {\cal O}(\sqrt{dT}).
\end{align*}

Similarly, under additional Assumption \ref{as_sym1} and based on Lemma \ref{lm:nvx}, if we choose 
$\eta=\frac{1}{L\sqrt{dT}}$, $\delta=\sqrt{\frac{d}{T}}$ and $D={\cal O}(\sqrt{d}L)$ for LAZOb, we can also get
\begin{align*}
\mathbb{E}[{\cal R }_T^{\rm nc}(\texttt{LAZOb})]\leq {\cal O}(\sqrt{dT}).
\end{align*}

\section{Experiment details}\label{experiment_detail}
\begin{figure*}[t]
    \centering
    \includegraphics[width=1\textwidth]{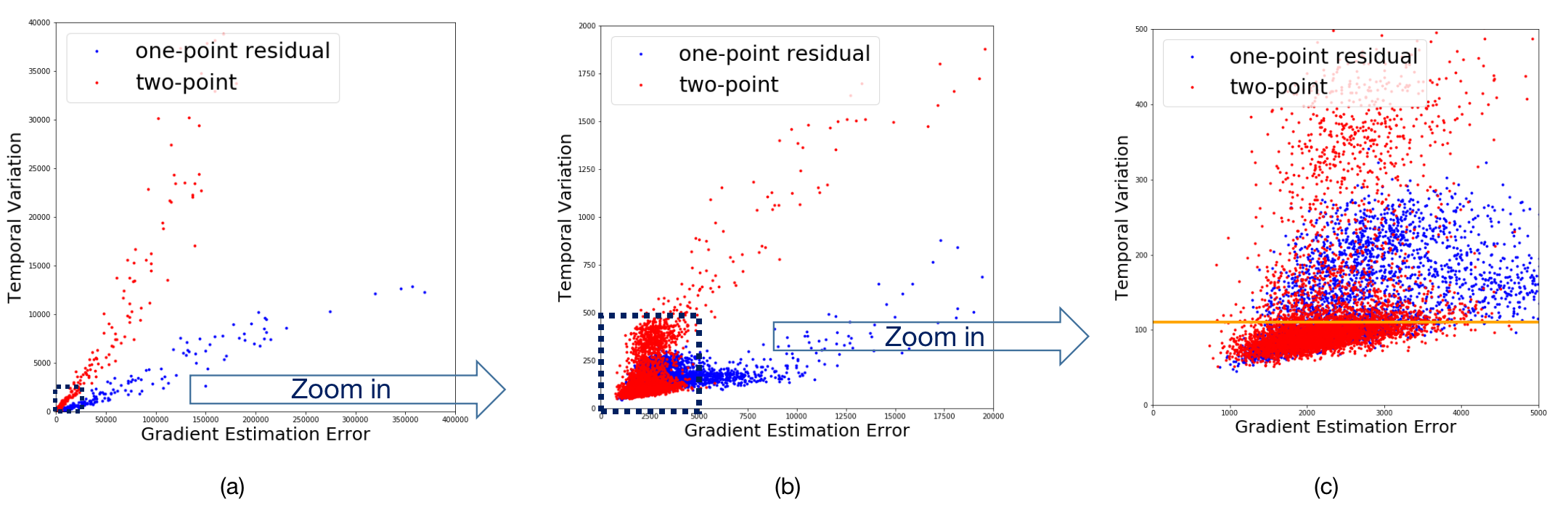}
    \vspace{-0.8cm}
    \caption{Temporal variation v.s. gradient estimation error under \textbf{one-point residual} method (\blue{blue}) and \textbf{two-point} algorithm (\red{red}) along the training trajectory for the online linear regression task. The difference of (a), (b) and (c) is the scale. The orange line in (c) indicates an ideal choice of $D$ for LAZOa in this example. 
}
    \label{fig:intuition}
\end{figure*}

\begin{figure*}[t]
    \centering
    \begin{tabular}{cccc}
\includegraphics[width=.2\textwidth]{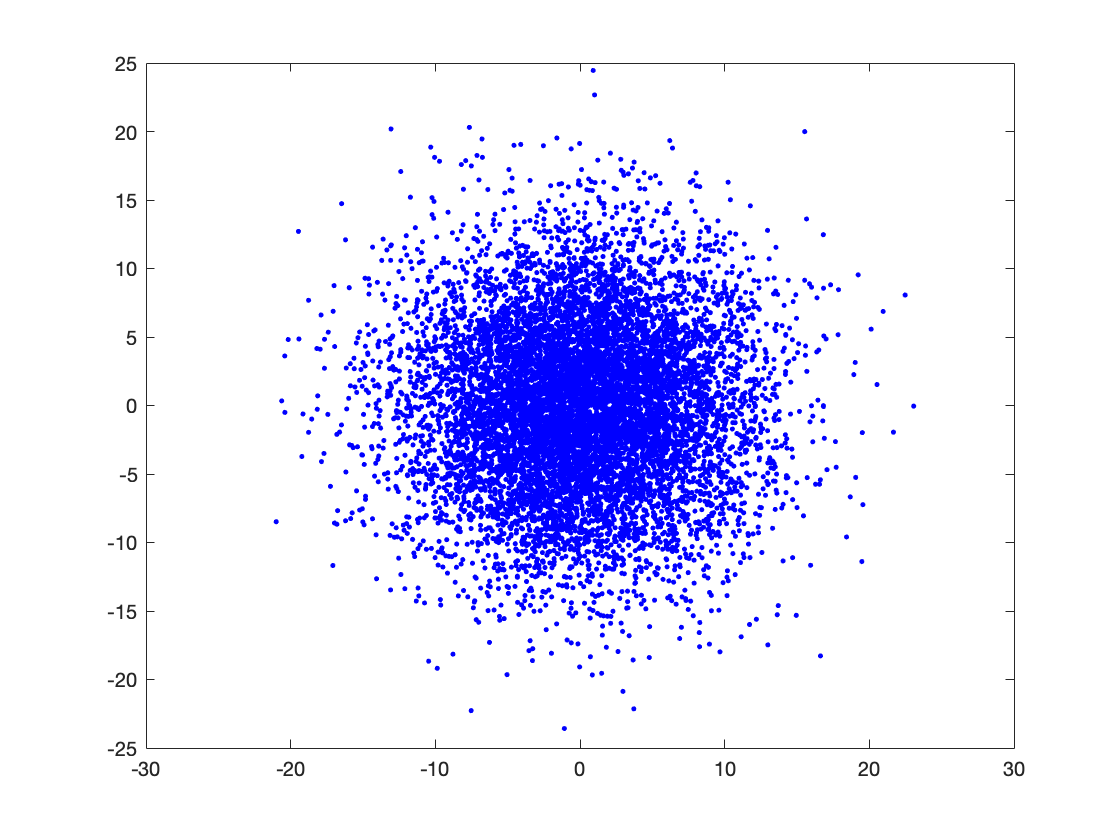} 
\hspace{-0.2cm} & \hspace{-0.2cm}
\includegraphics[width=.2\textwidth]{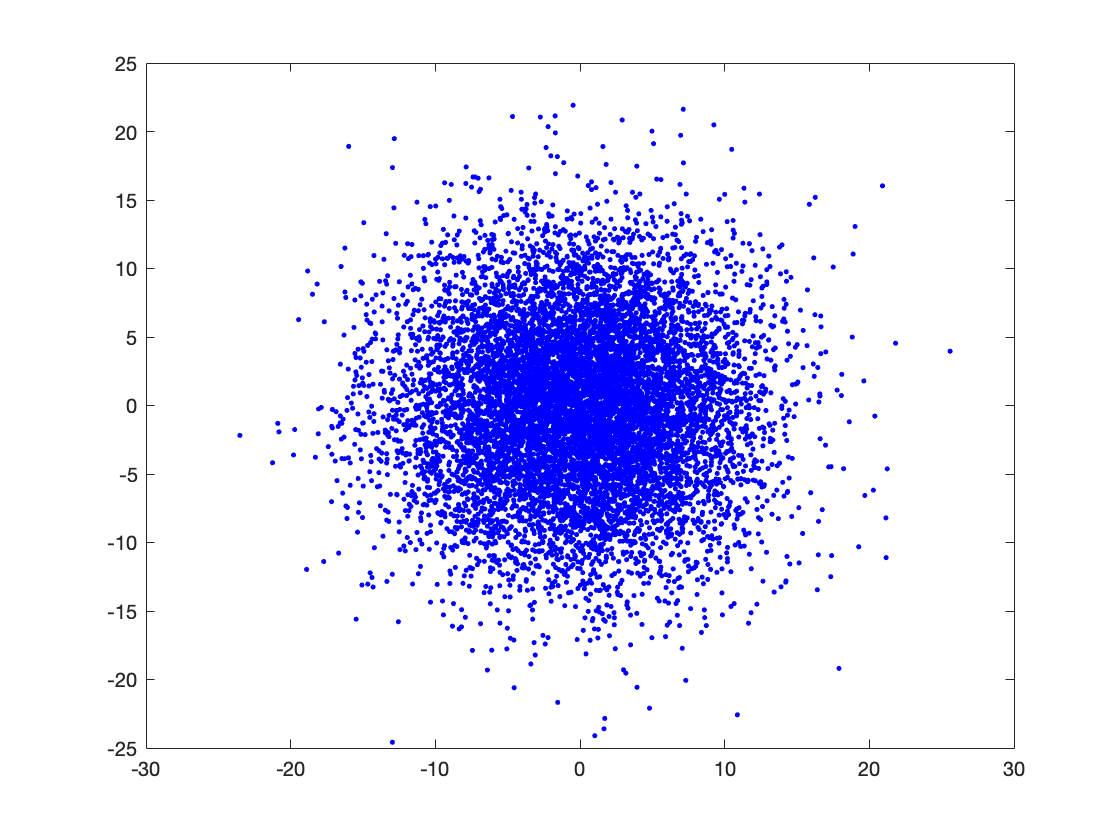} 
\hspace{-0.2cm} & \hspace{-0.2cm} \includegraphics[width=.2\textwidth]{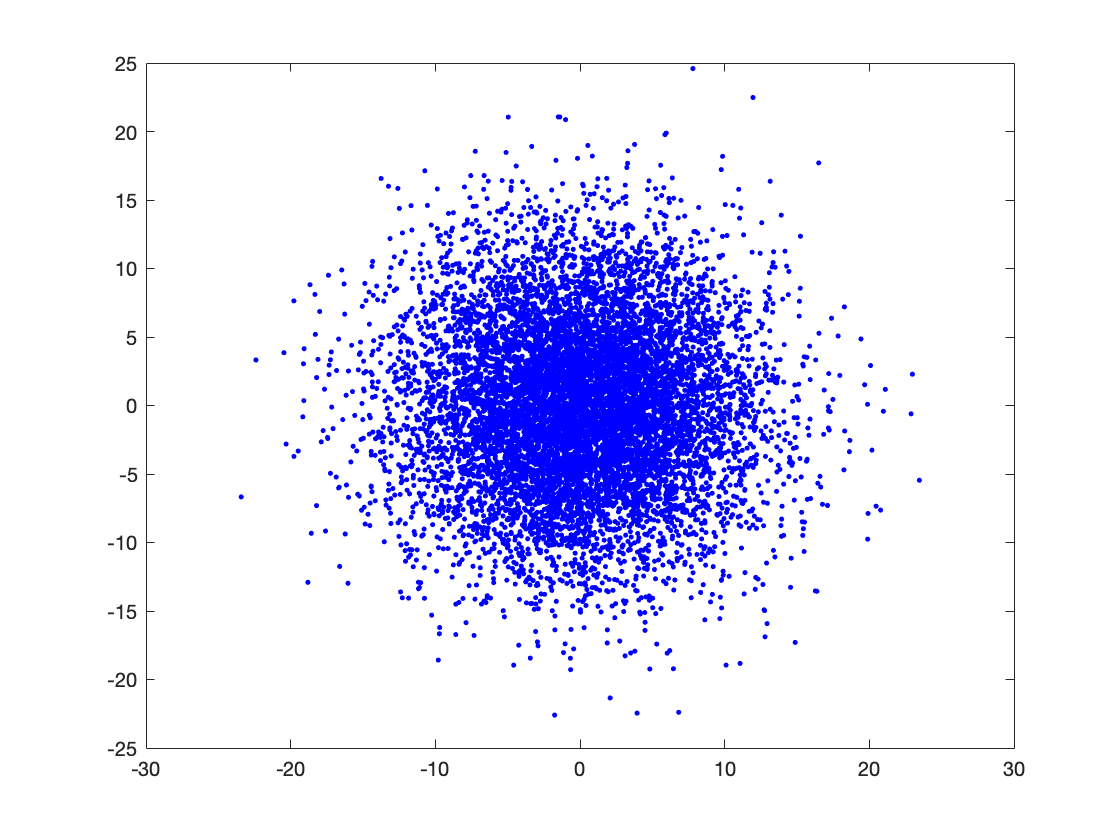}
\hspace{-0.2cm} & \hspace{-0.2cm}
\includegraphics[width=.2\textwidth]{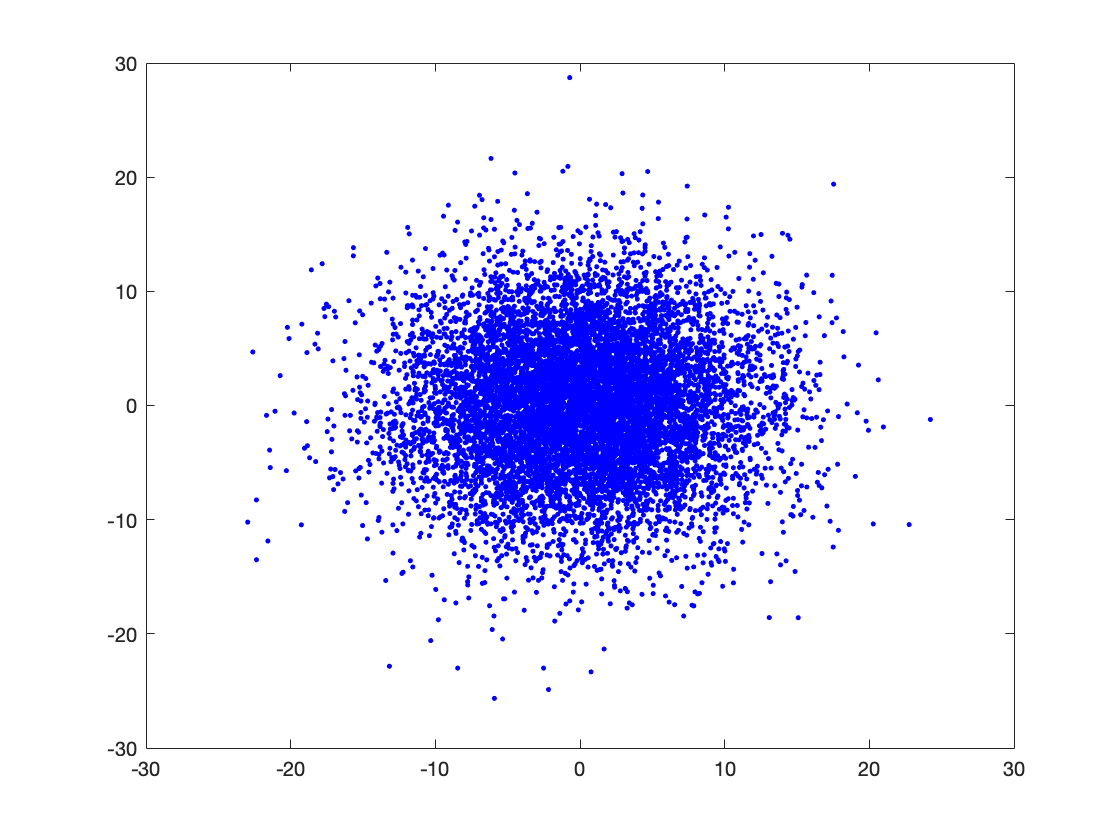}\\
    \footnotesize(a) Iteration 2, Repeat 1 &\footnotesize(b) Iteration 2, Repeat 2  &\footnotesize(c) Iteration 2, Repeat 3 &\footnotesize(d) Iteration 2, Repeat 4\\
    \includegraphics[width=.2\textwidth]{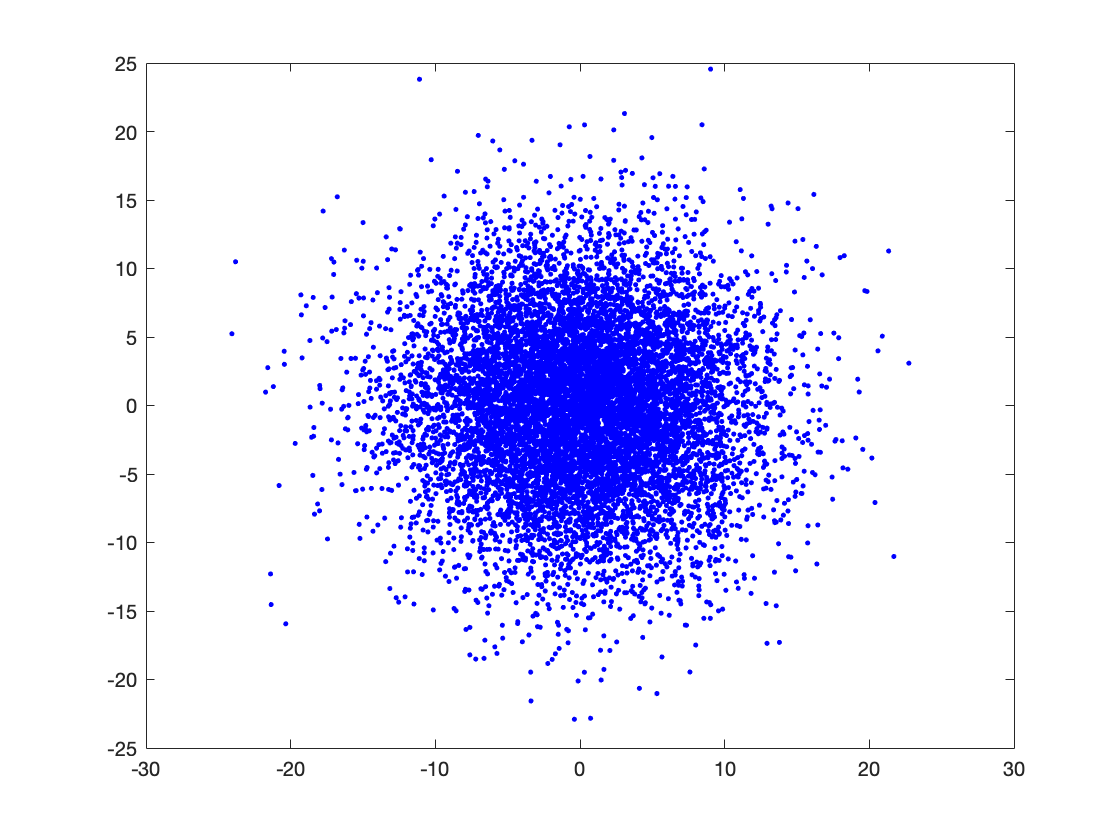} 
\hspace{-0.2cm} & \hspace{-0.2cm}
\includegraphics[width=.2\textwidth]{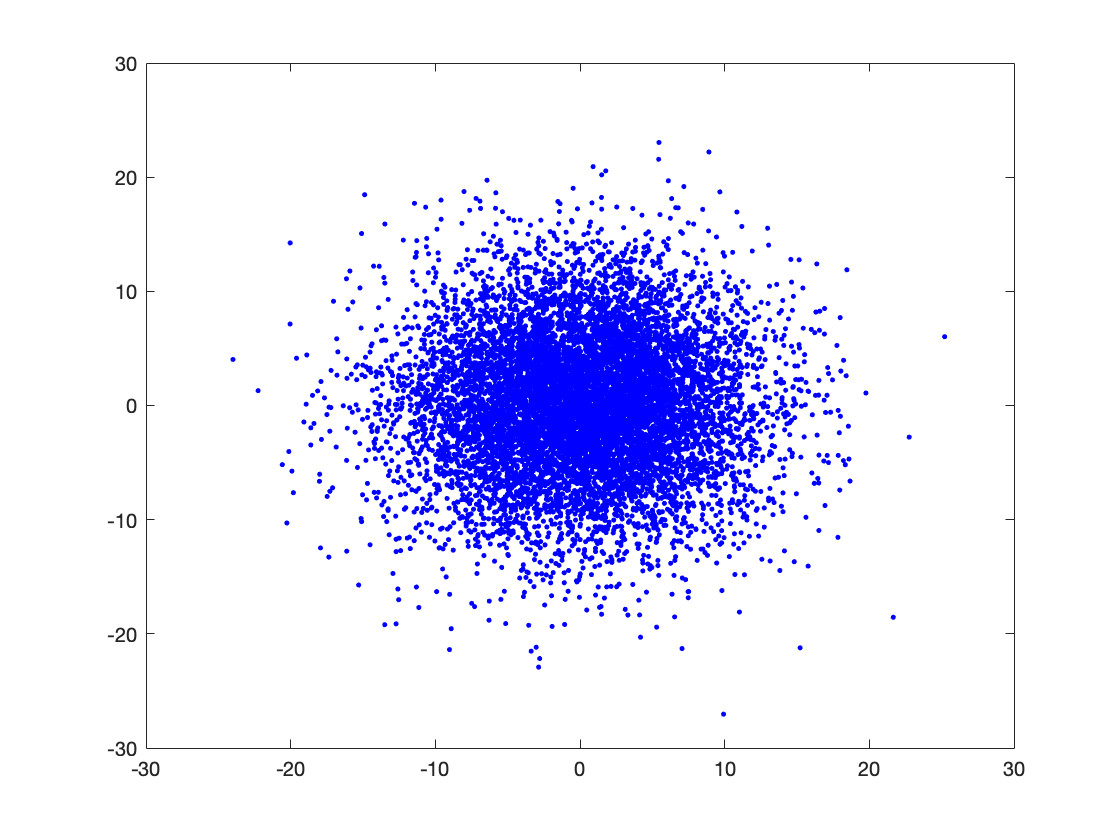} 
\hspace{-0.2cm} & \hspace{-0.2cm} \includegraphics[width=.2\textwidth]{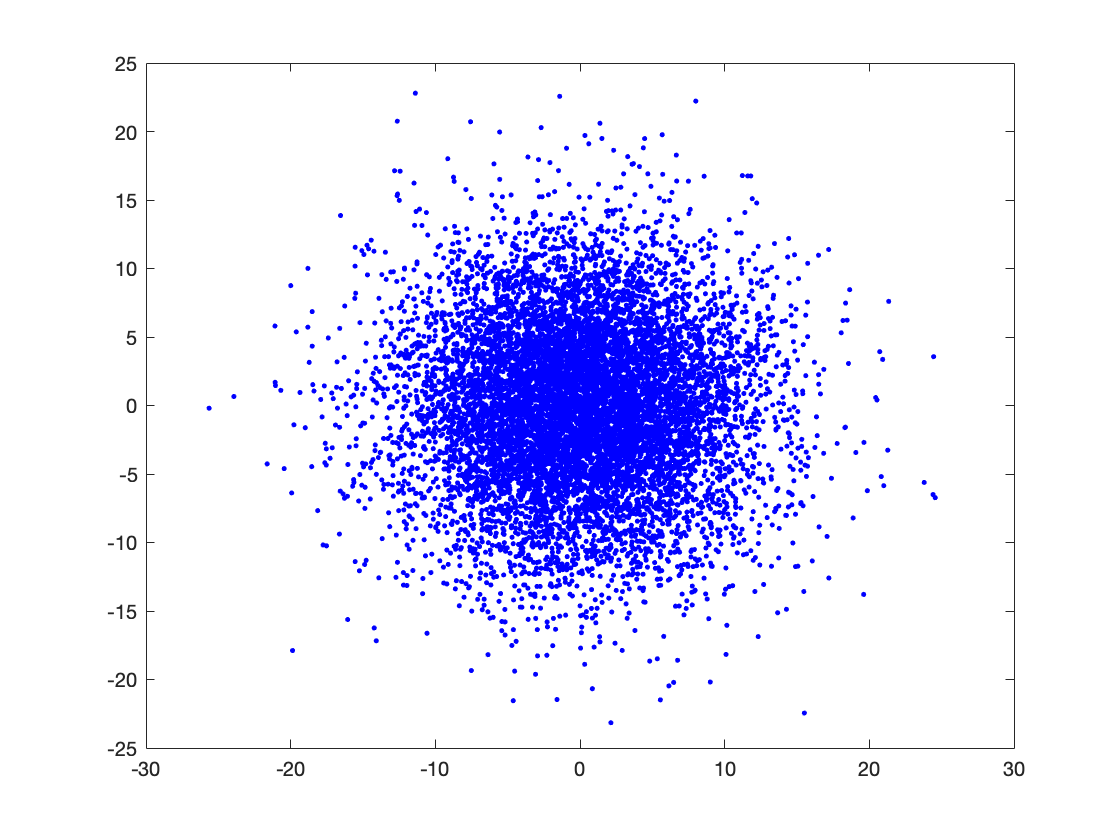}
\hspace{-0.2cm} & \hspace{-0.2cm}
\includegraphics[width=.2\textwidth]{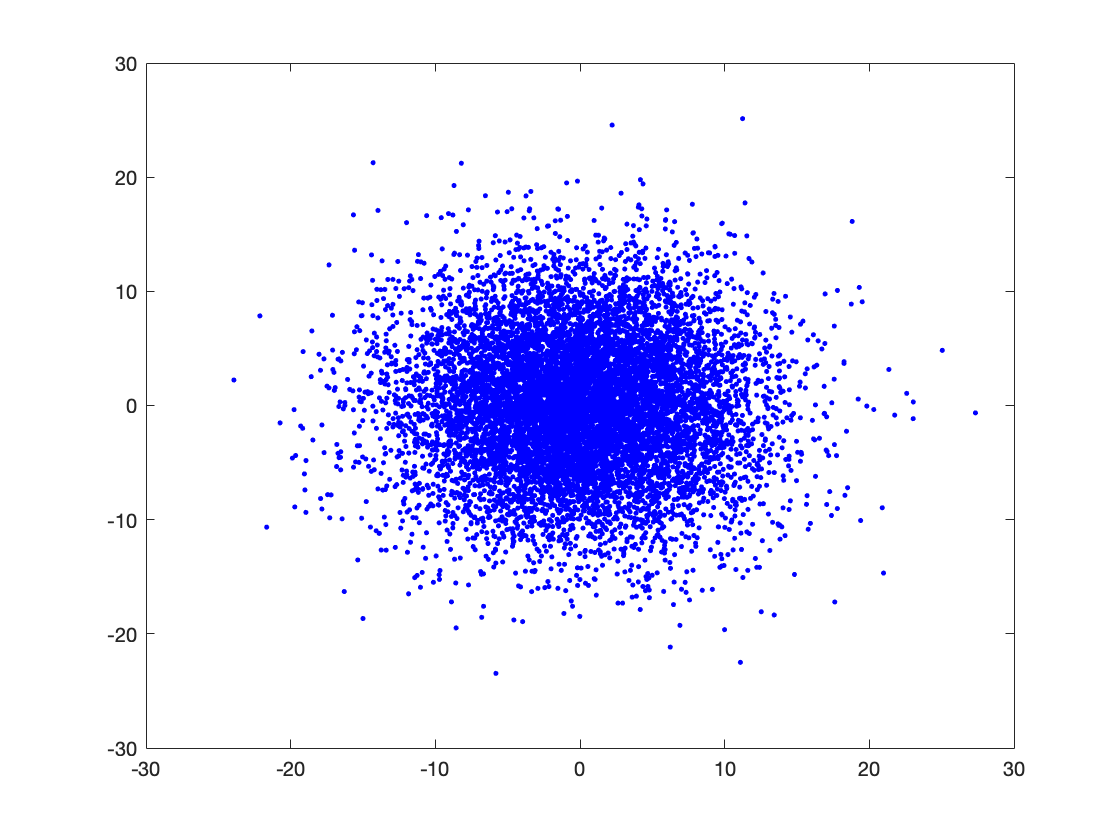}\\
    \footnotesize(e) Iteration 2, Repeat 5 &\footnotesize(f) Iteration 2, Repeat 6  &\footnotesize(g) Iteration 2, Repeat 7 &\footnotesize(h) Iteration 2, Repeat 8\\
    
    \includegraphics[width=.2\textwidth]{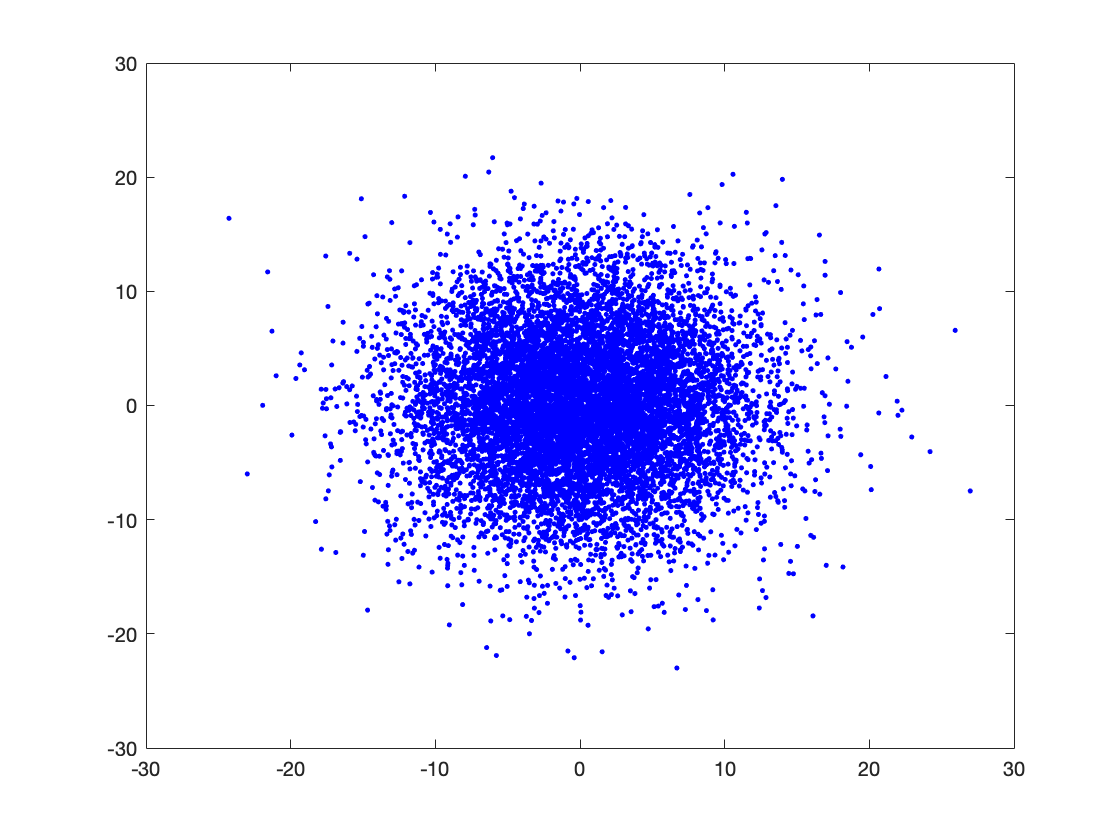} 
\hspace{-0.2cm} & \hspace{-0.2cm}
\includegraphics[width=.2\textwidth]{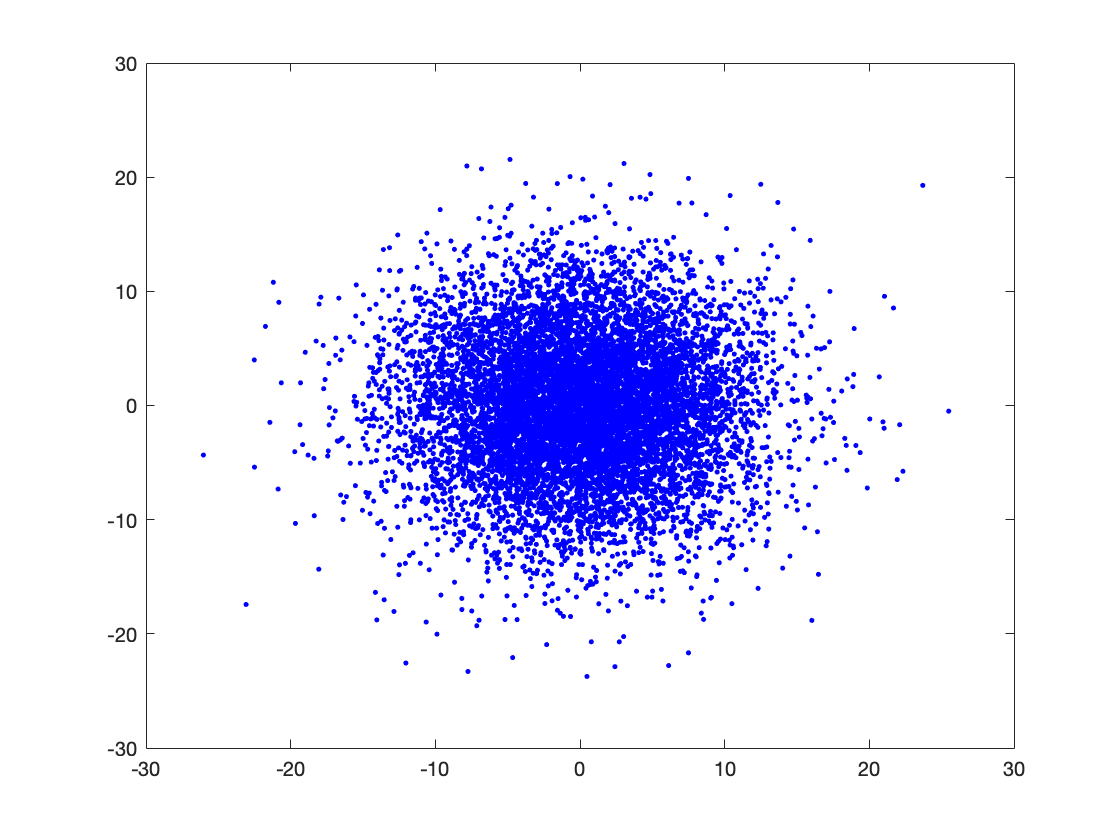} 
\hspace{-0.2cm} & \hspace{-0.2cm} \includegraphics[width=.2\textwidth]{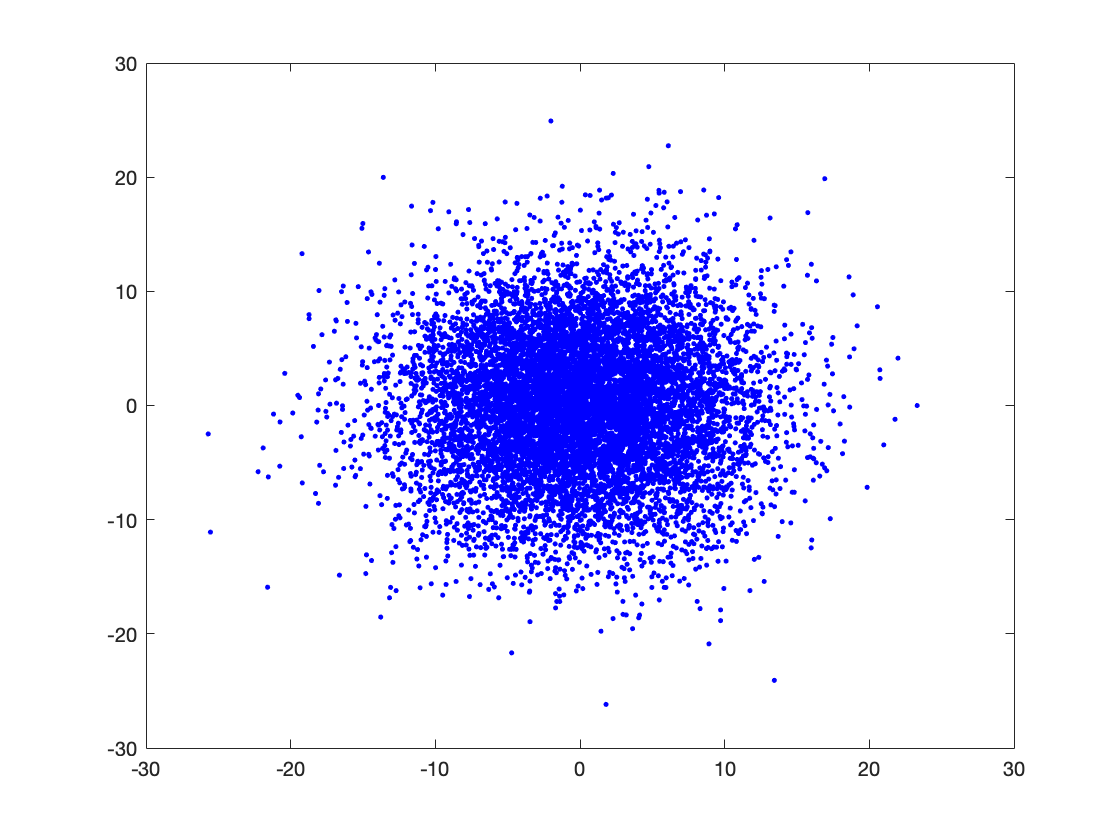}
\hspace{-0.2cm} & \hspace{-0.2cm}
\includegraphics[width=.2\textwidth]{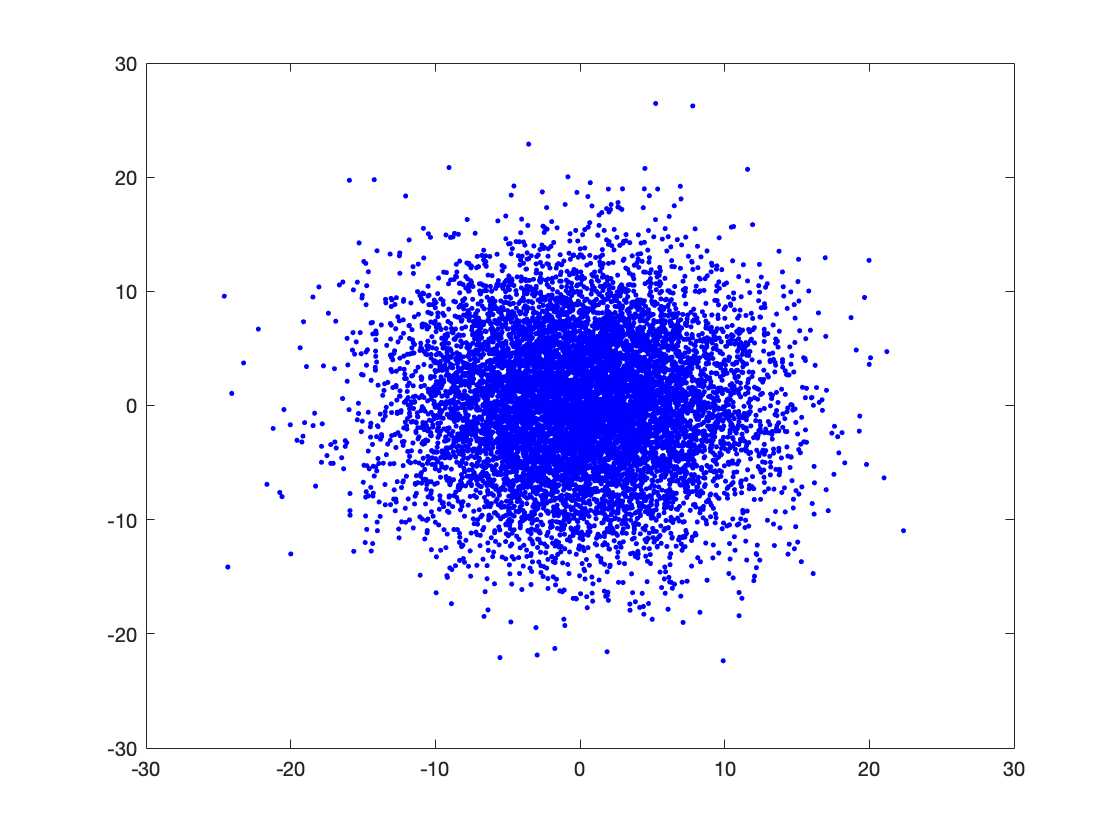}\\
    \footnotesize(i) Iteration 20, Repeat 1 &\footnotesize(j) Iteration 20, Repeat 2  &\footnotesize(k) Iteration 20, Repeat 3 &\footnotesize(l) Iteration 20, Repeat 4\\
    \includegraphics[width=.2\textwidth]{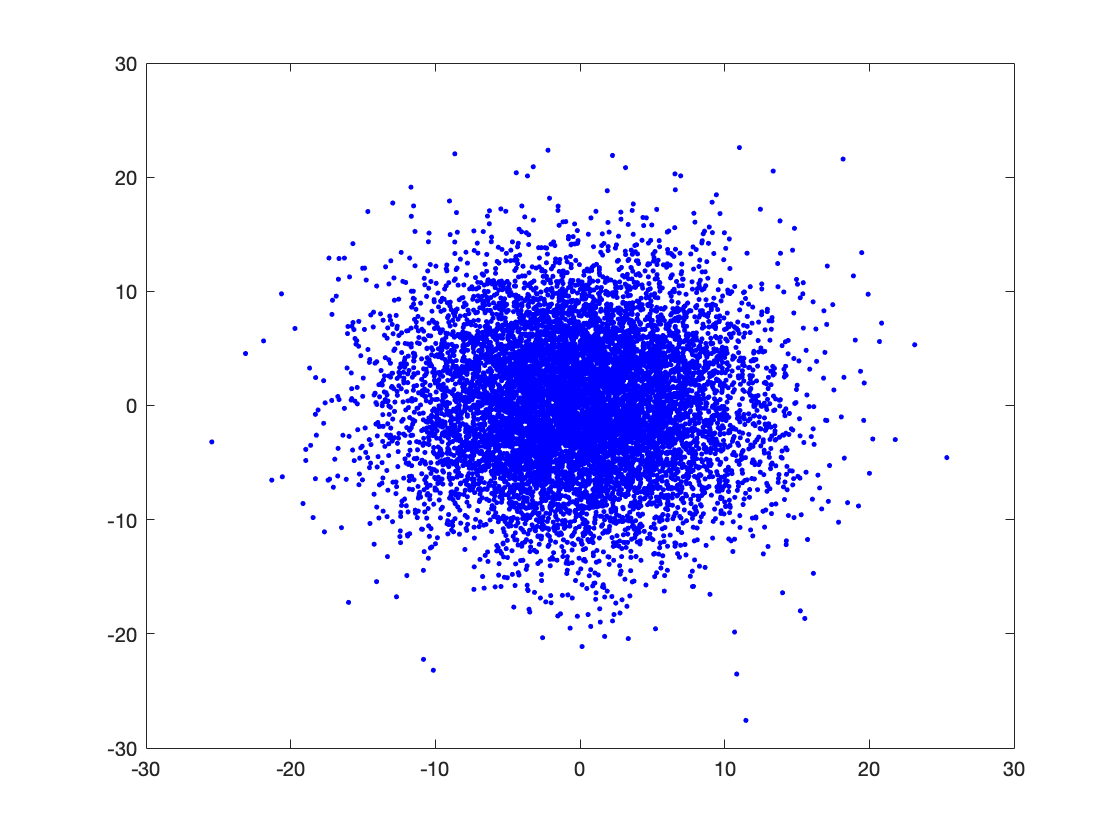} 
\hspace{-0.2cm} & \hspace{-0.2cm}
\includegraphics[width=.2\textwidth]{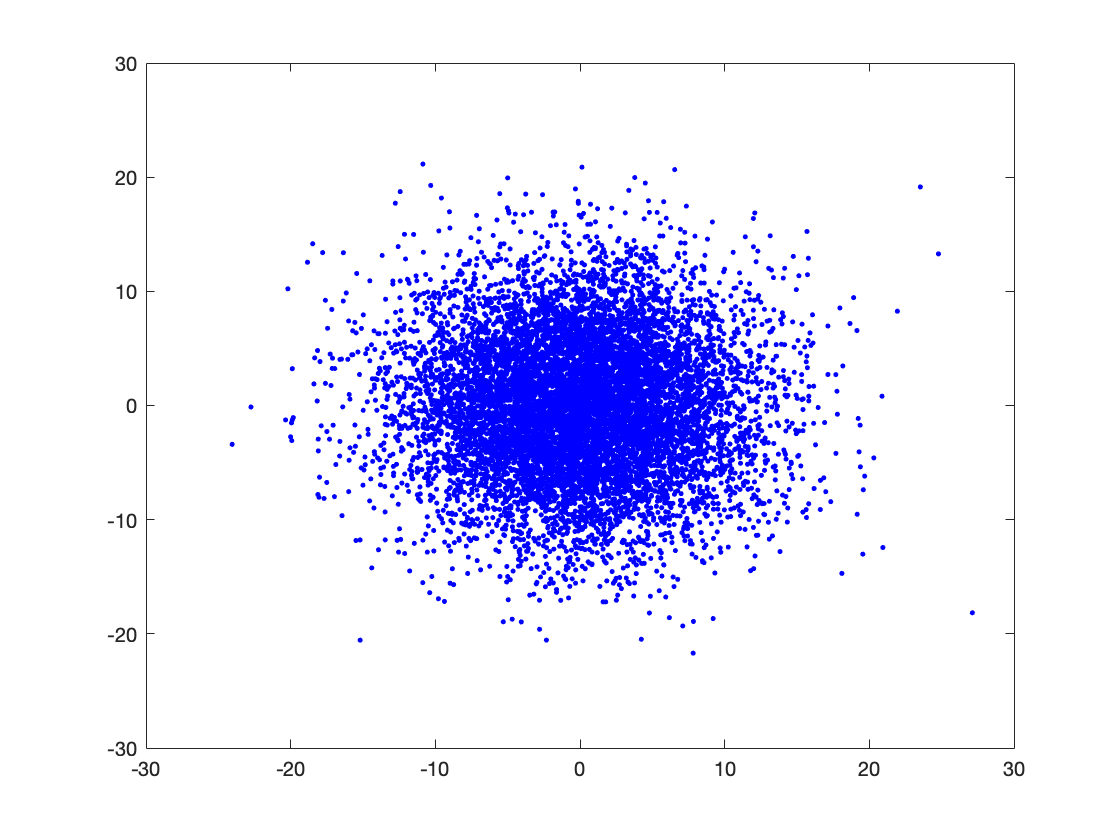} 
\hspace{-0.2cm} & \hspace{-0.2cm} \includegraphics[width=.2\textwidth]{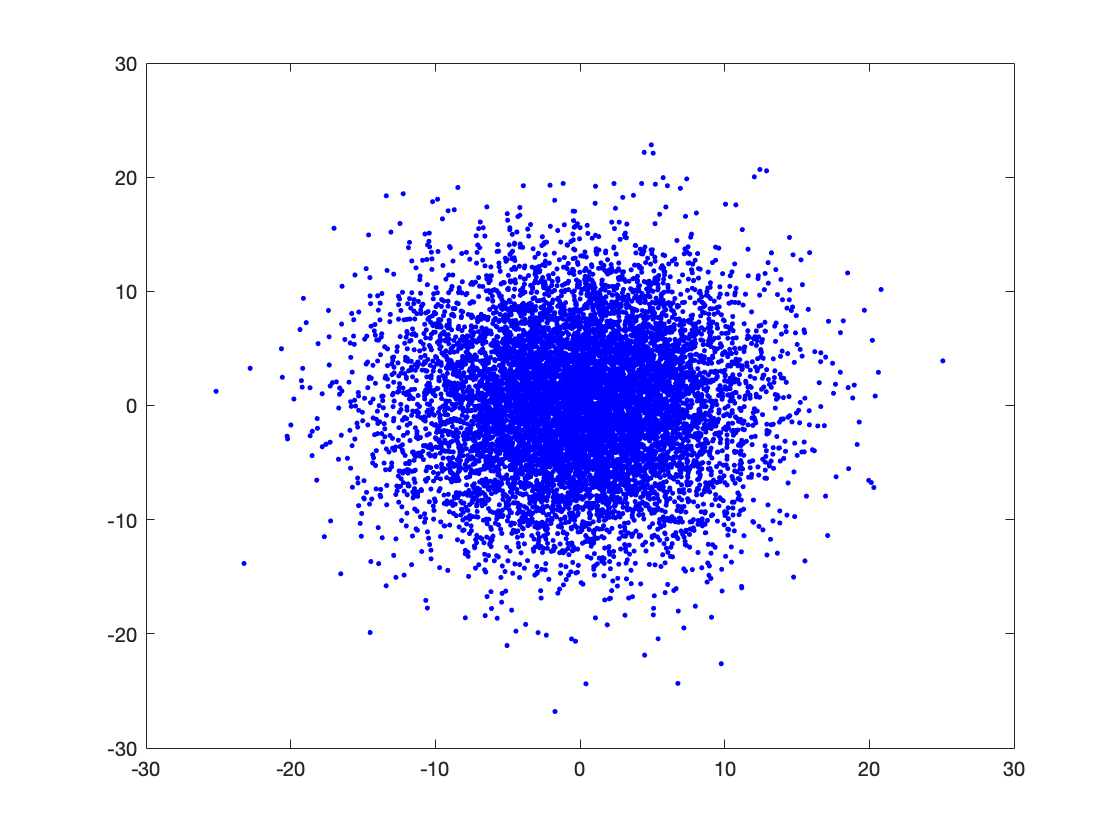}
\hspace{-0.2cm} & \hspace{-0.2cm}
\includegraphics[width=.2\textwidth]{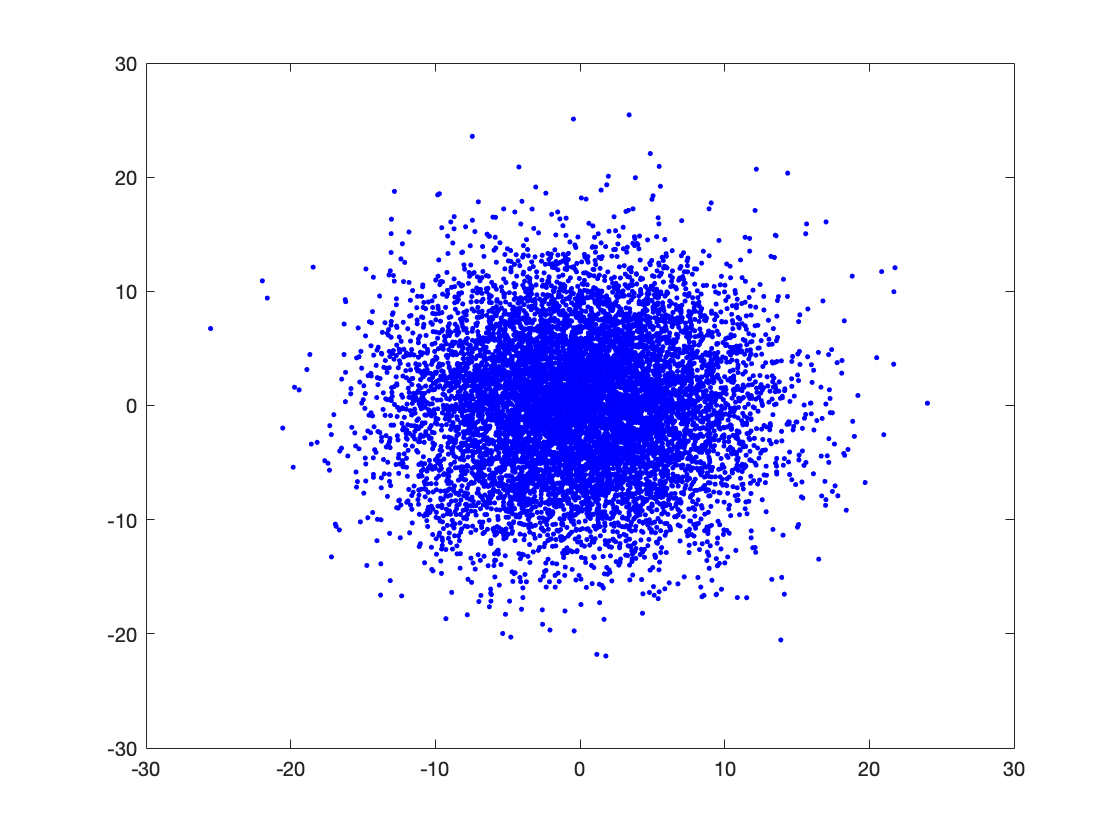}\\
    \footnotesize(m) Iteration 20, Repeat 5 &\footnotesize(n) Iteration 20, Repeat 6  &\footnotesize(o) Iteration 20, Repeat 7 &\footnotesize(p) Iteration 20, Repeat 8\\
    
    \includegraphics[width=.2\textwidth]{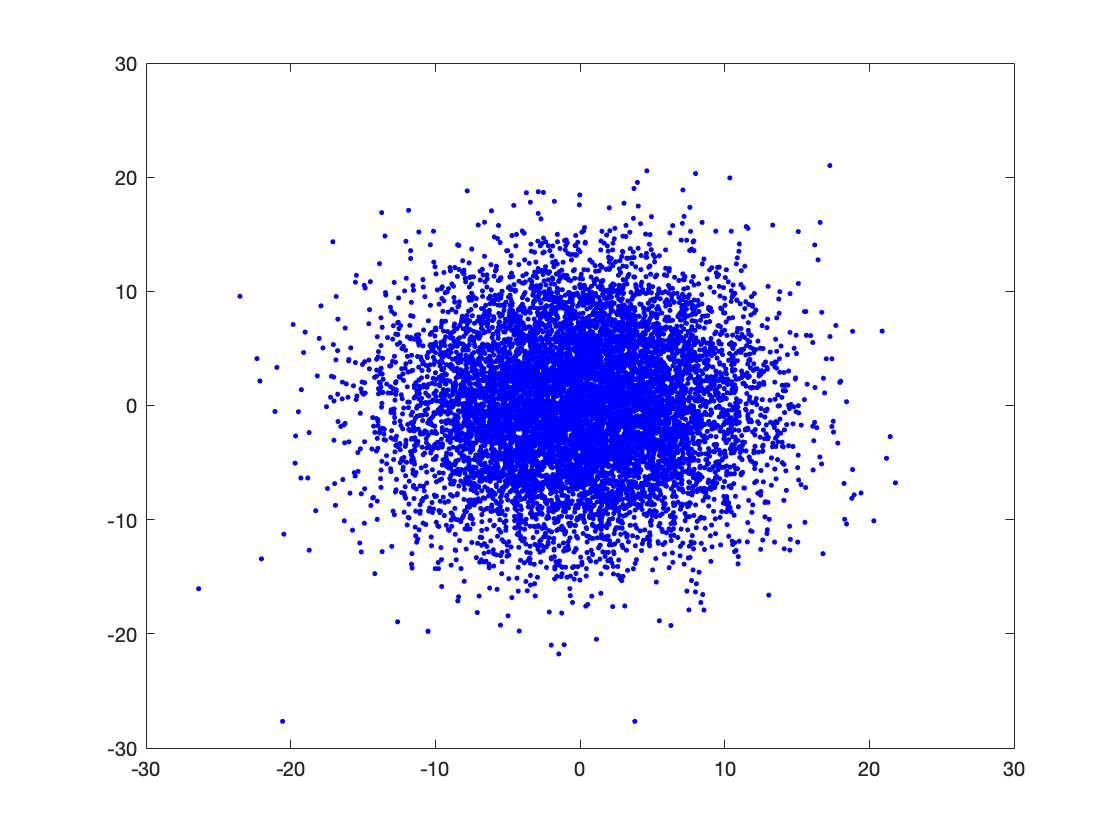} 
\hspace{-0.2cm} & \hspace{-0.2cm}
\includegraphics[width=.2\textwidth]{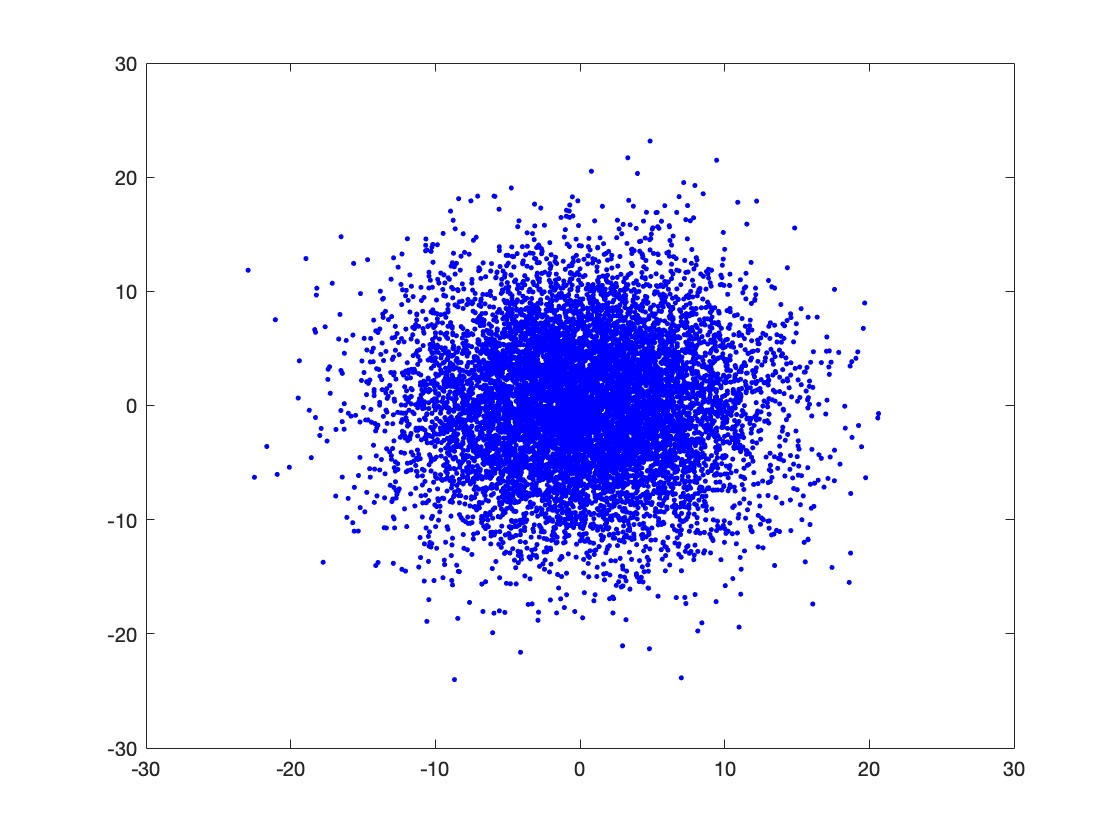} 
\hspace{-0.2cm} & \hspace{-0.2cm} \includegraphics[width=.2\textwidth]{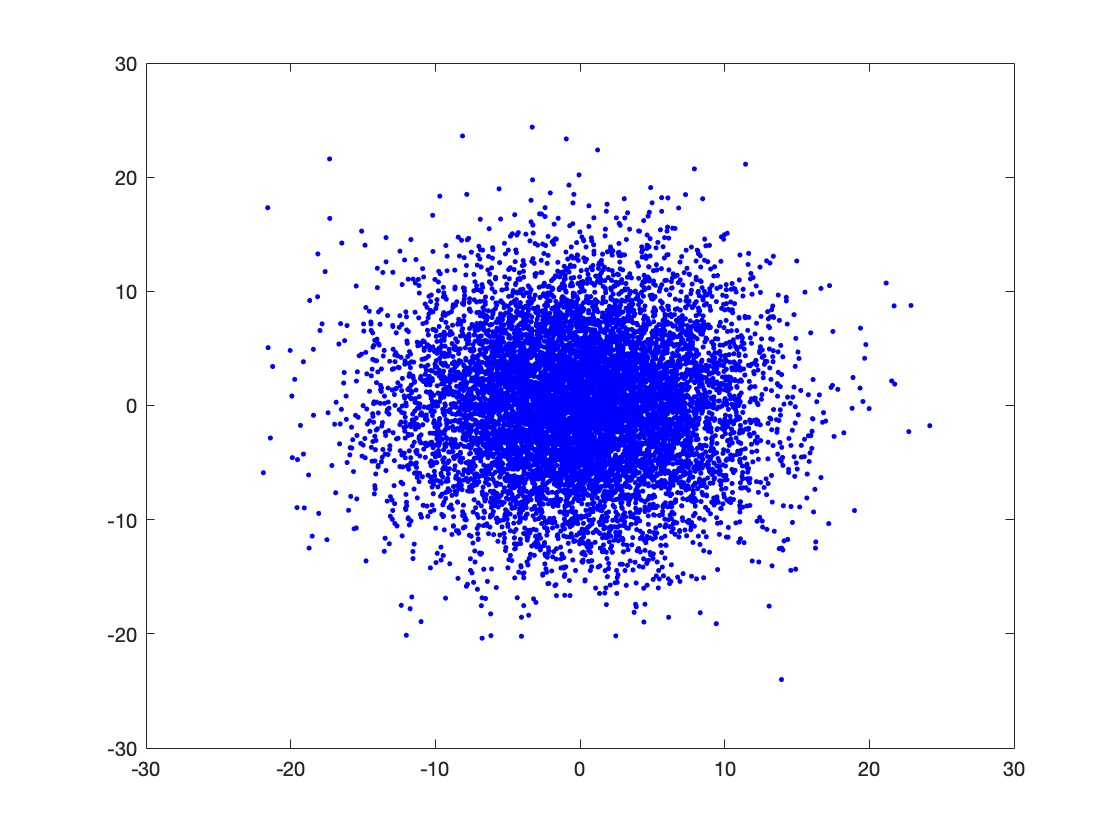}
\hspace{-0.2cm} & \hspace{-0.2cm}
\includegraphics[width=.2\textwidth]{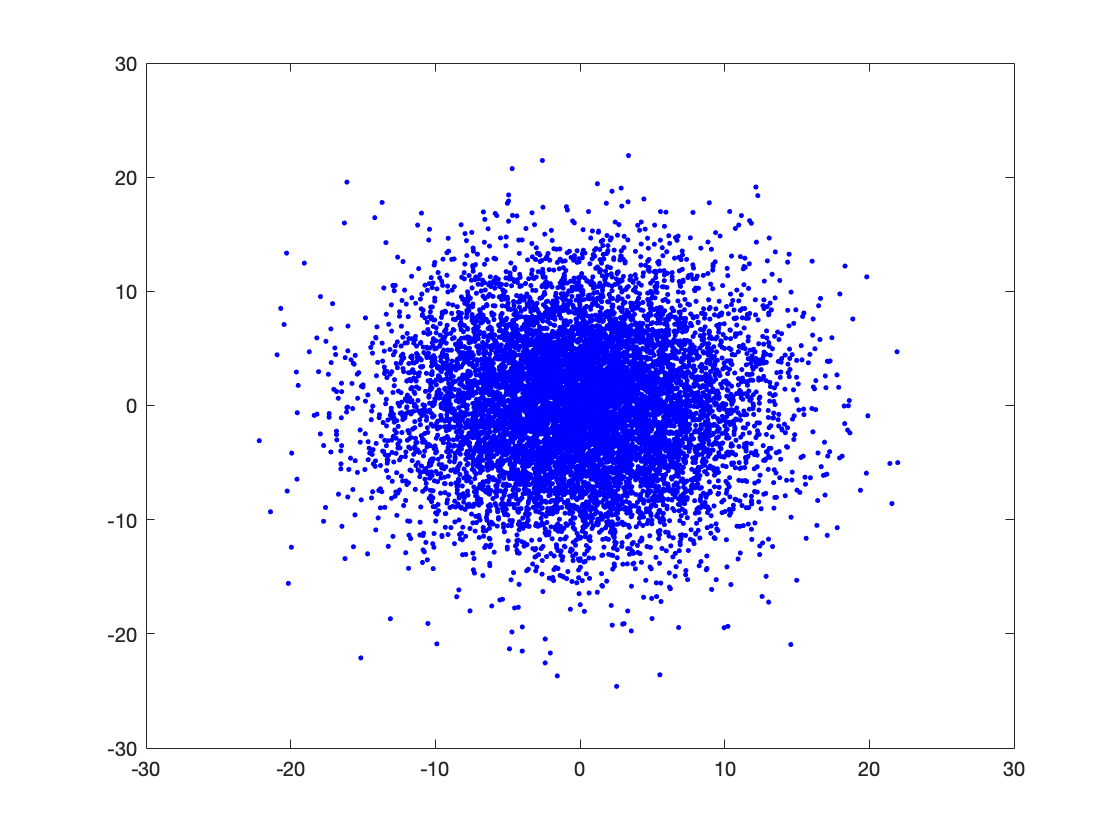}\\
    \footnotesize(q) Iteration 100, Repeat 1 &\footnotesize(r) Iteration 100, Repeat 2  &\footnotesize(s) Iteration 100, Repeat 3 &\footnotesize(t) Iteration 100, Repeat 4\\
    \includegraphics[width=.2\textwidth]{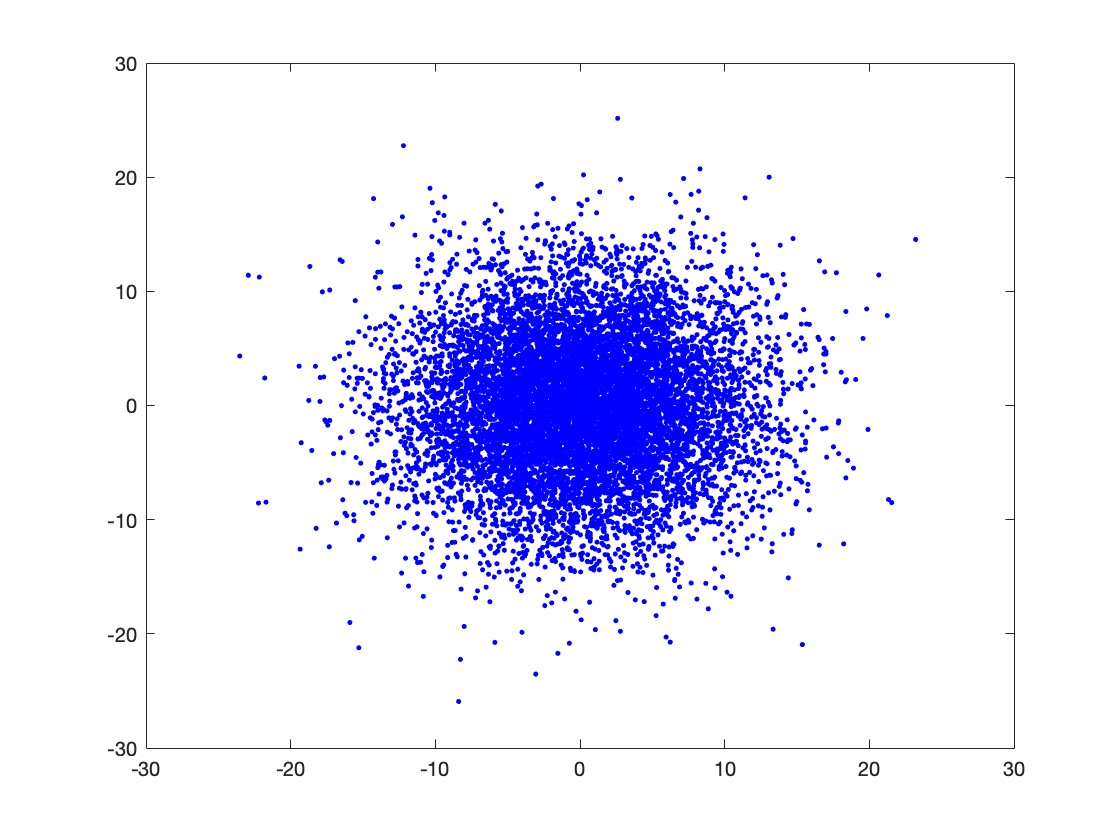} 
\hspace{-0.2cm} & \hspace{-0.2cm}
\includegraphics[width=.2\textwidth]{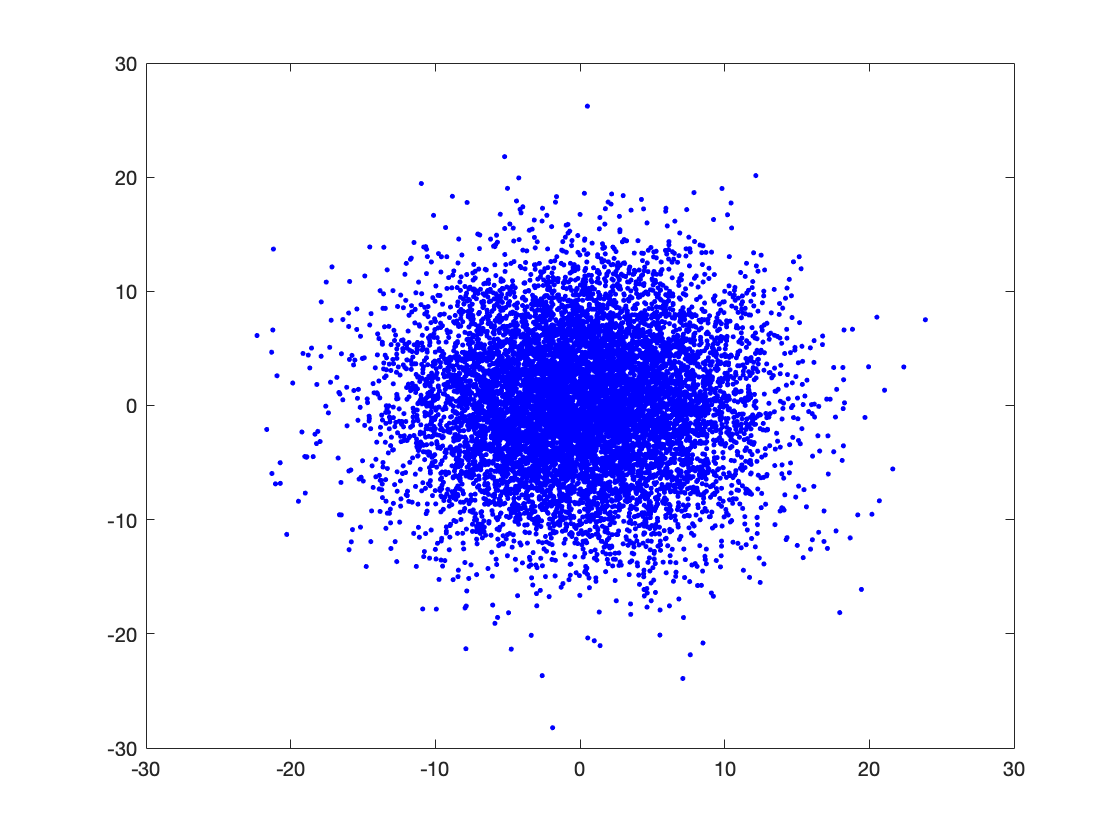} 
\hspace{-0.2cm} & \hspace{-0.2cm} \includegraphics[width=.2\textwidth]{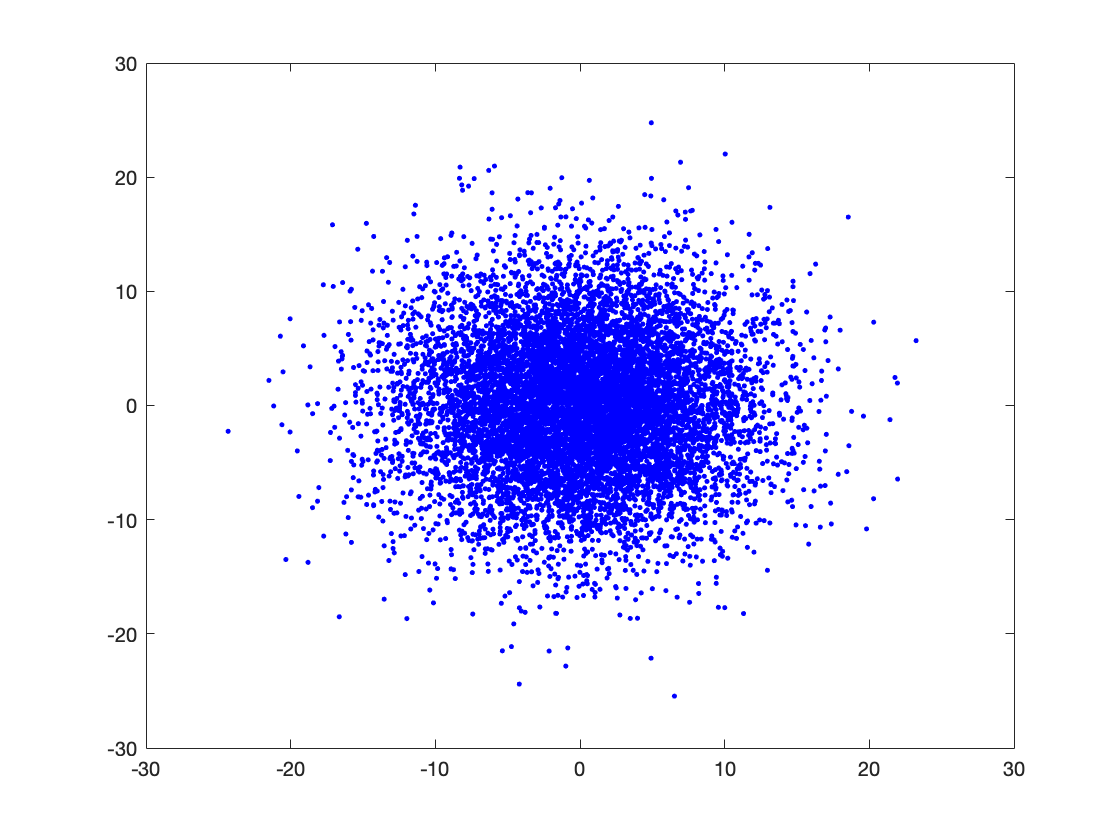}
\hspace{-0.2cm} & \hspace{-0.2cm}
\includegraphics[width=.2\textwidth]{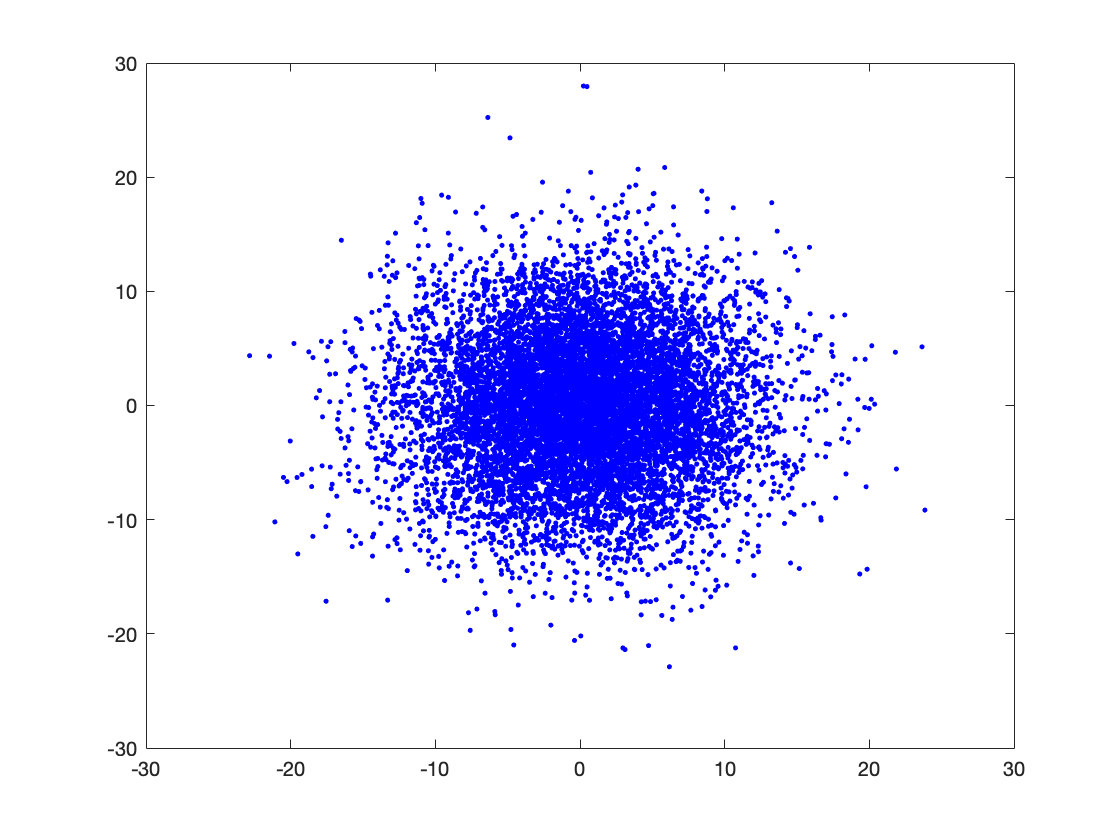}\\
    \footnotesize(u) Iteration 100, Repeat 5 &\footnotesize(v) Iteration 100, Repeat 6  &\footnotesize(w) Iteration 100, Repeat 7 &\footnotesize(x) Iteration 100, Repeat 8\\
    \end{tabular}
    \caption{Distribution of $u_t\in \tilde{A}_t$ projected using 24 random matrices for iterations 2, 20 and 100 of LAZOa in LQR example.}
    \label{fig:symm_appendix2}
 \vspace{-0.2cm}
\end{figure*}

\begin{figure*}[t]
    \centering
    \begin{tabular}{cccc}
\includegraphics[width=.2\textwidth]{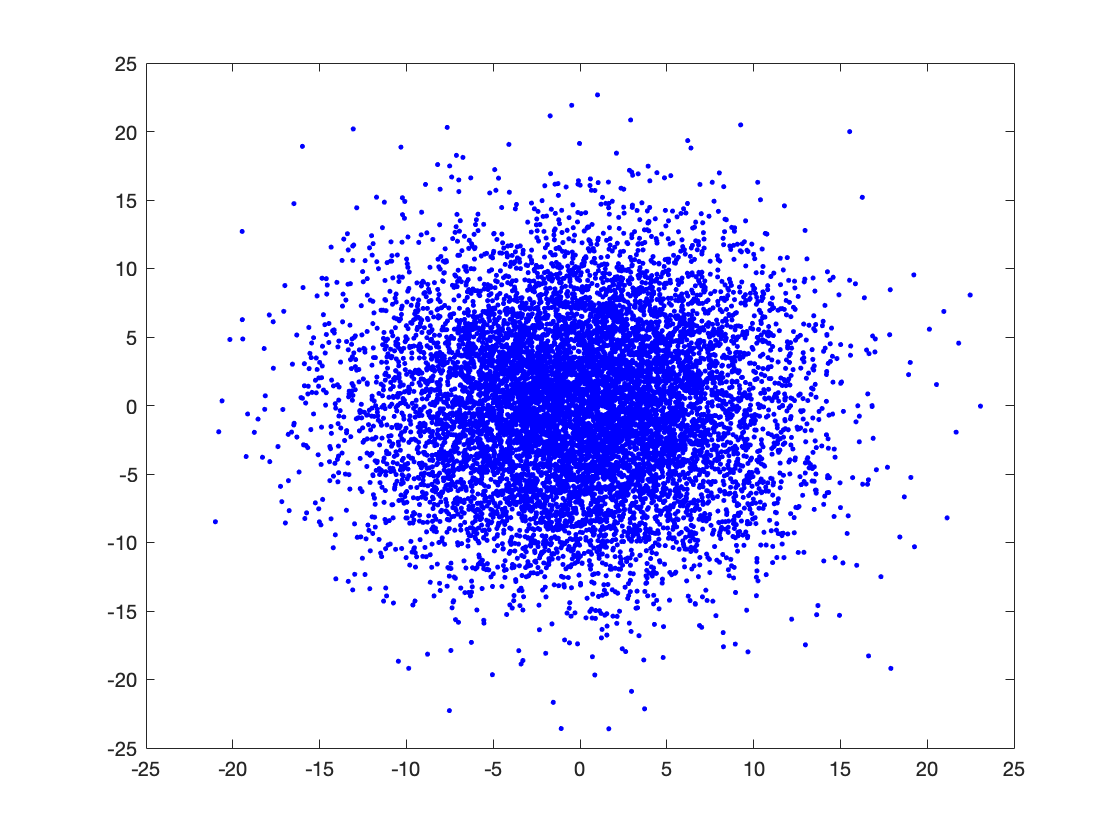} 
\hspace{-0.2cm} & \hspace{-0.2cm}
\includegraphics[width=.2\textwidth]{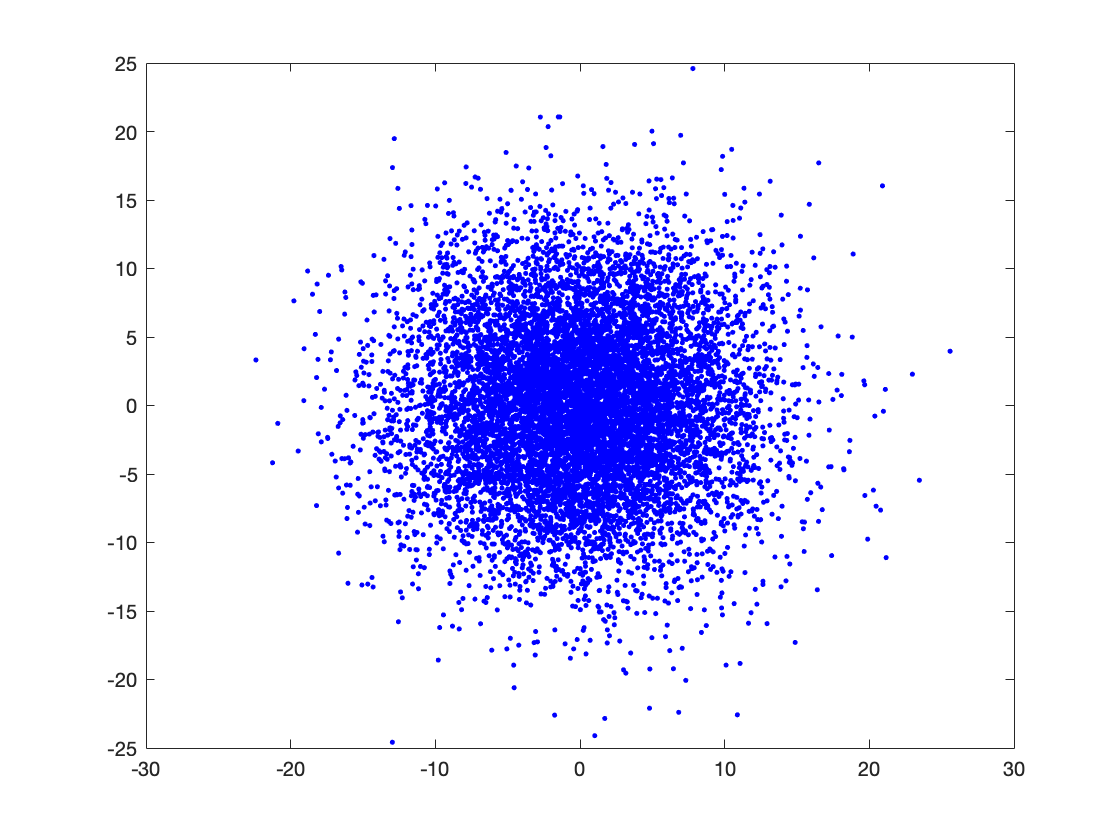} 
\hspace{-0.2cm} & \hspace{-0.2cm} \includegraphics[width=.2\textwidth]{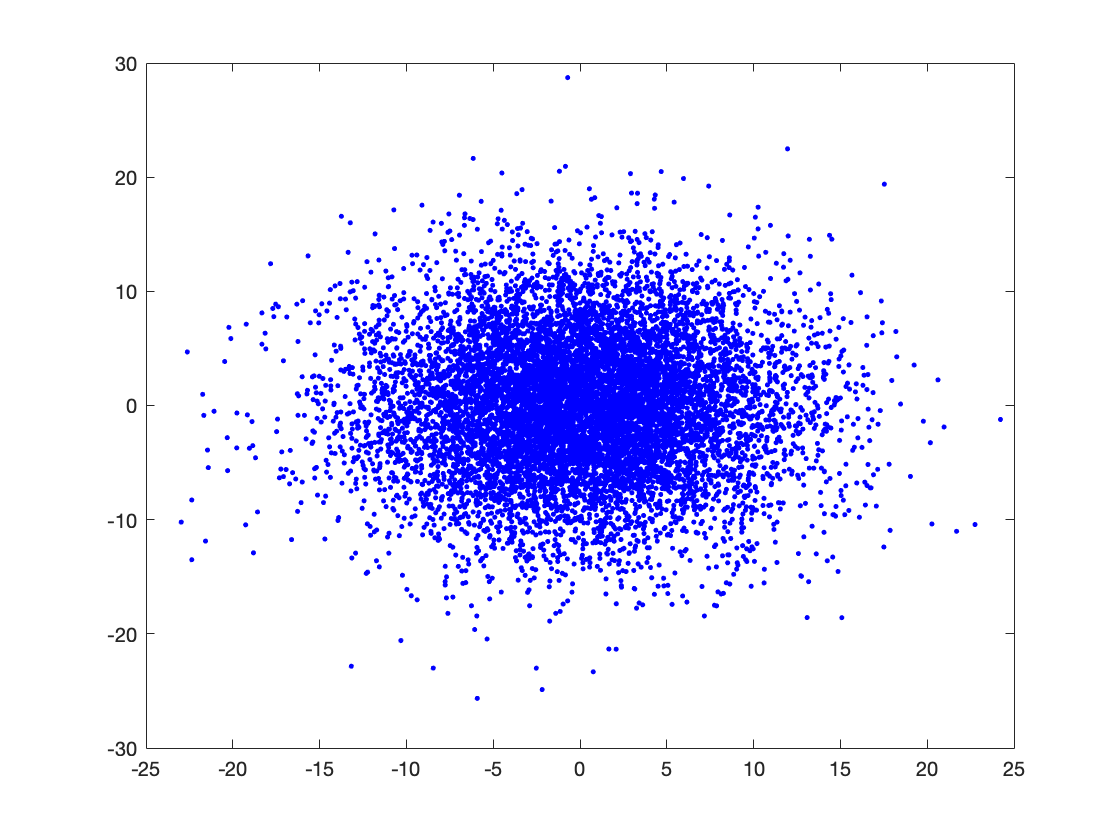}
\hspace{-0.2cm} & \hspace{-0.2cm}
\includegraphics[width=.2\textwidth]{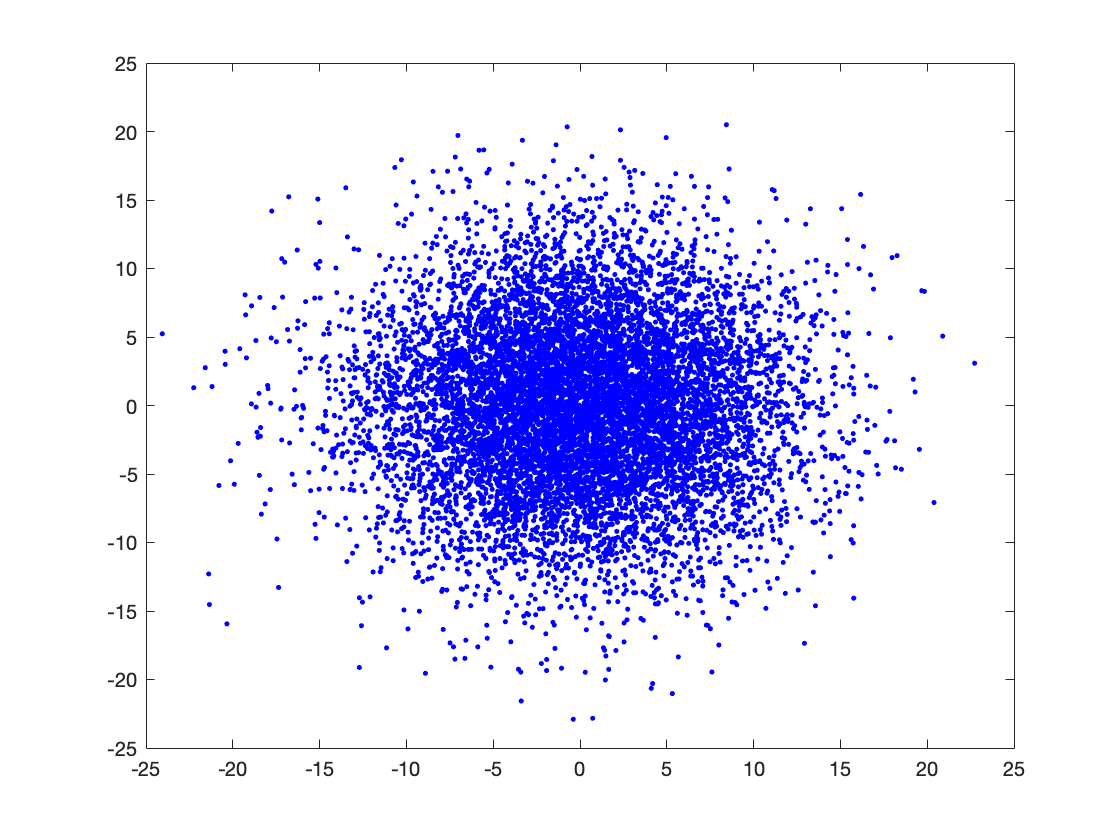}\\
    \footnotesize(a) Iteration 2, Repeat 1 &\footnotesize(b) Iteration 2, Repeat 2  &\footnotesize(c) Iteration 2, Repeat 3 &\footnotesize(d) Iteration 2, Repeat 4\\
    \includegraphics[width=.2\textwidth]{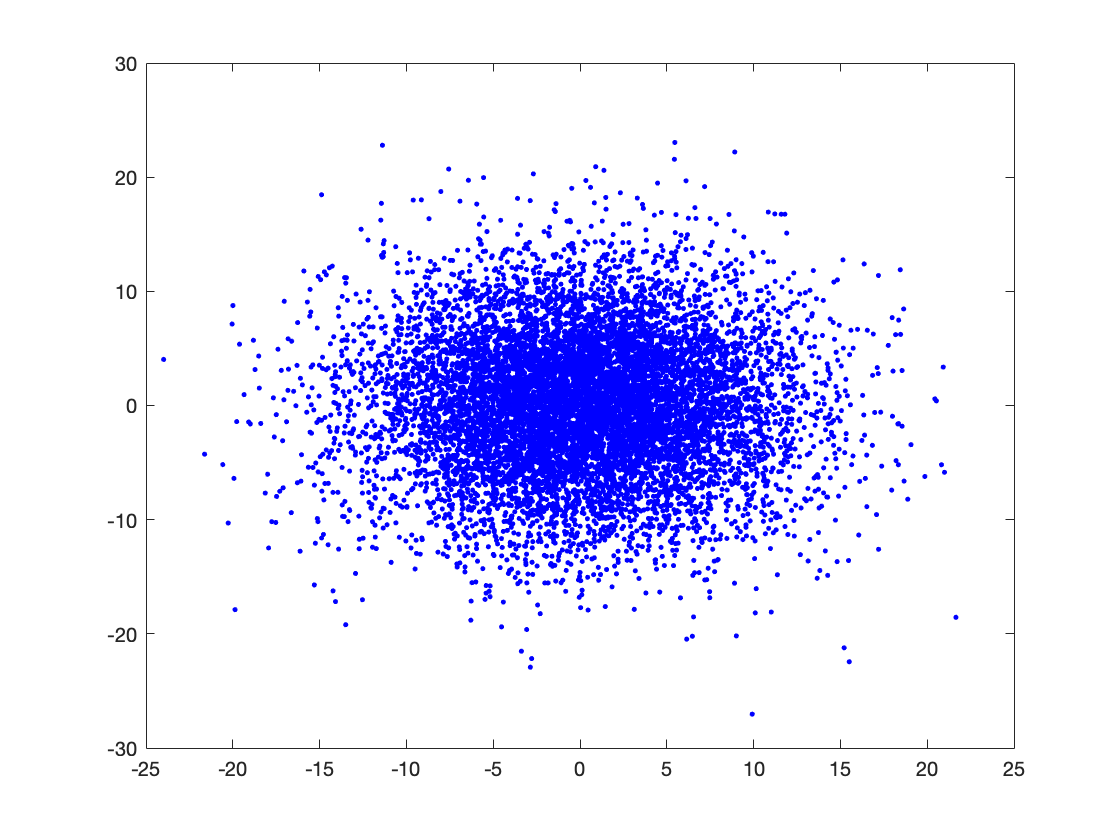} 
\hspace{-0.2cm} & \hspace{-0.2cm}
\includegraphics[width=.2\textwidth]{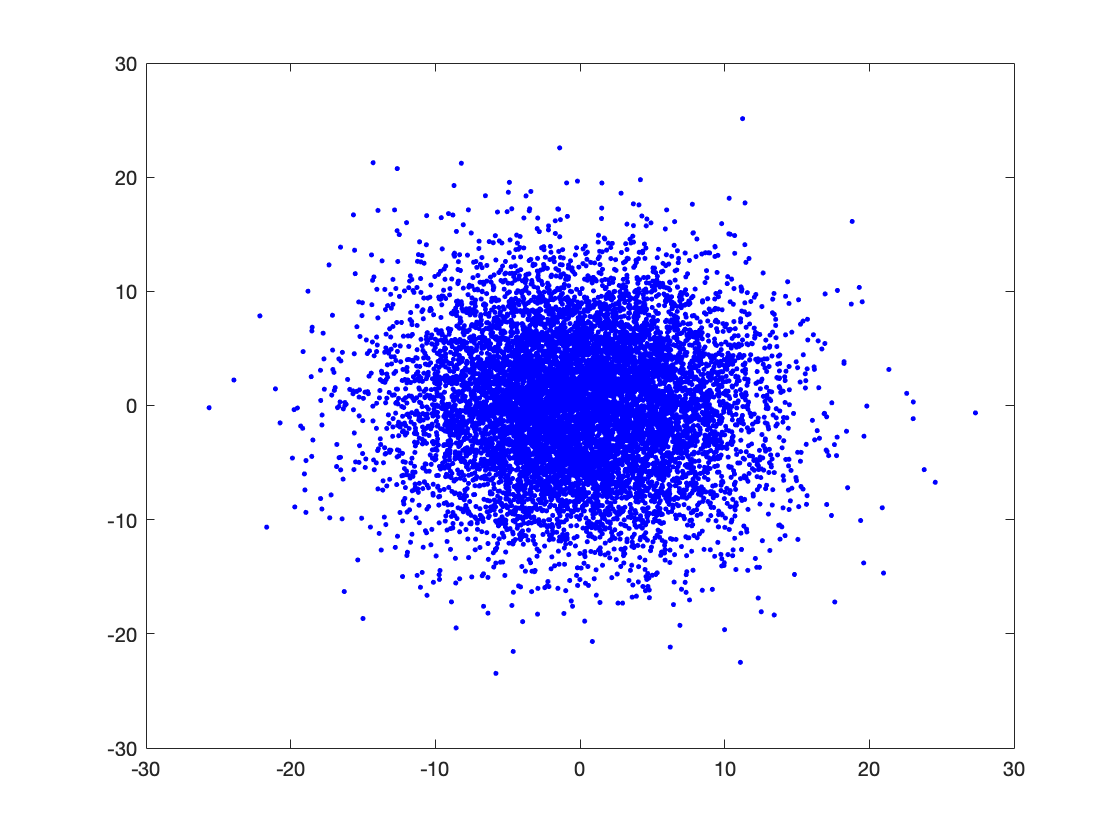} 
\hspace{-0.2cm} & \hspace{-0.2cm} \includegraphics[width=.2\textwidth]{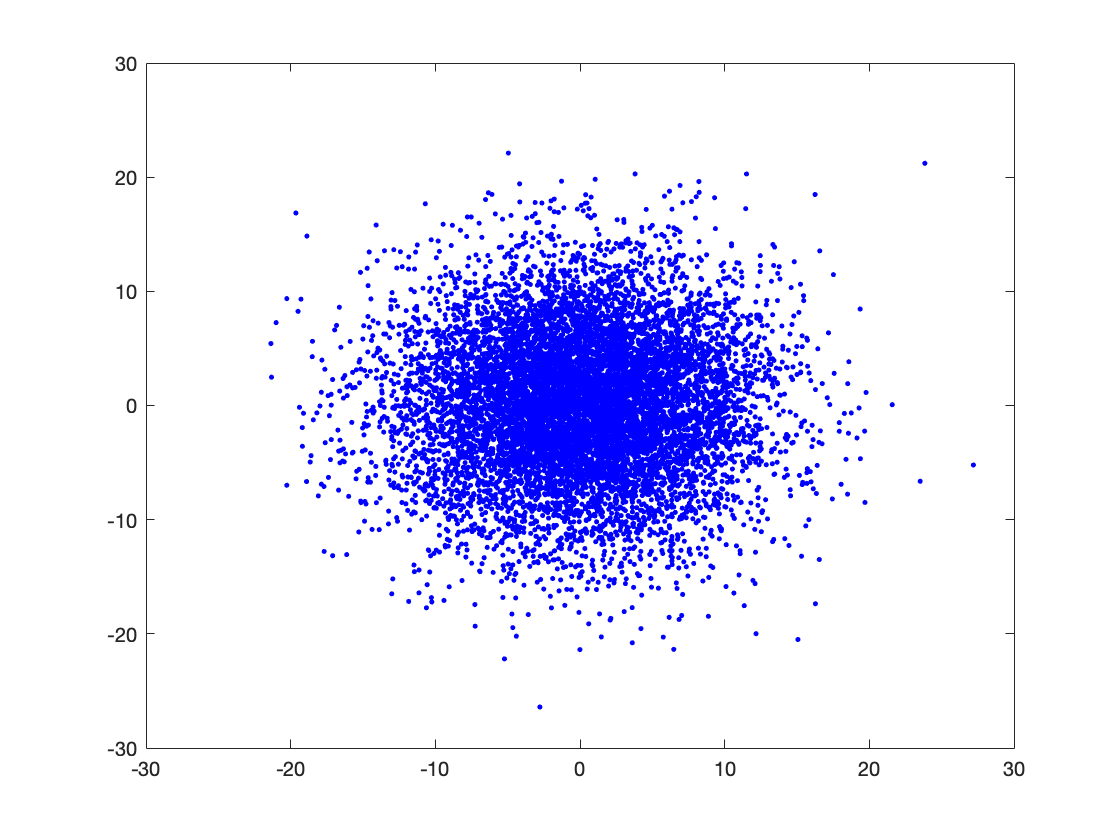}
\hspace{-0.2cm} & \hspace{-0.2cm}
\includegraphics[width=.2\textwidth]{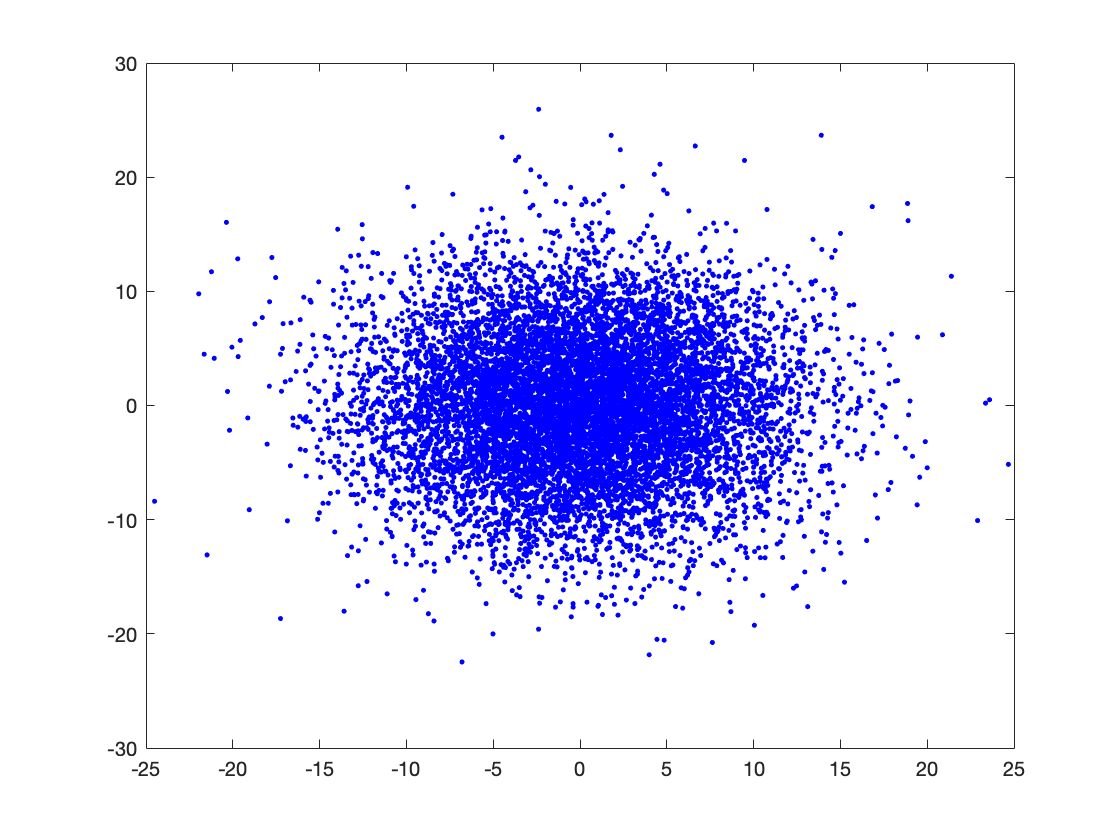}\\
    \footnotesize(e) Iteration 2, Repeat 5 &\footnotesize(f) Iteration 2, Repeat 6  &\footnotesize(g) Iteration 2, Repeat 7 &\footnotesize(h) Iteration 2, Repeat 8\\
    
    \includegraphics[width=.2\textwidth]{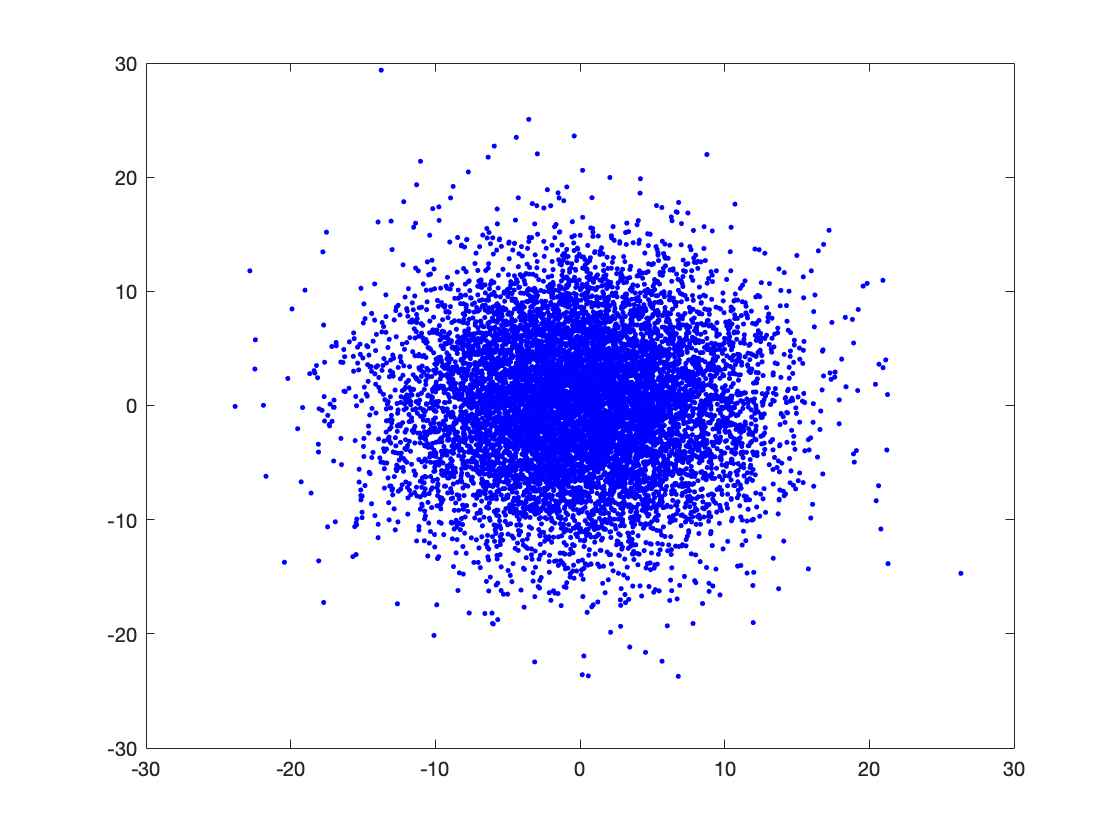} 
\hspace{-0.2cm} & \hspace{-0.2cm}
\includegraphics[width=.2\textwidth]{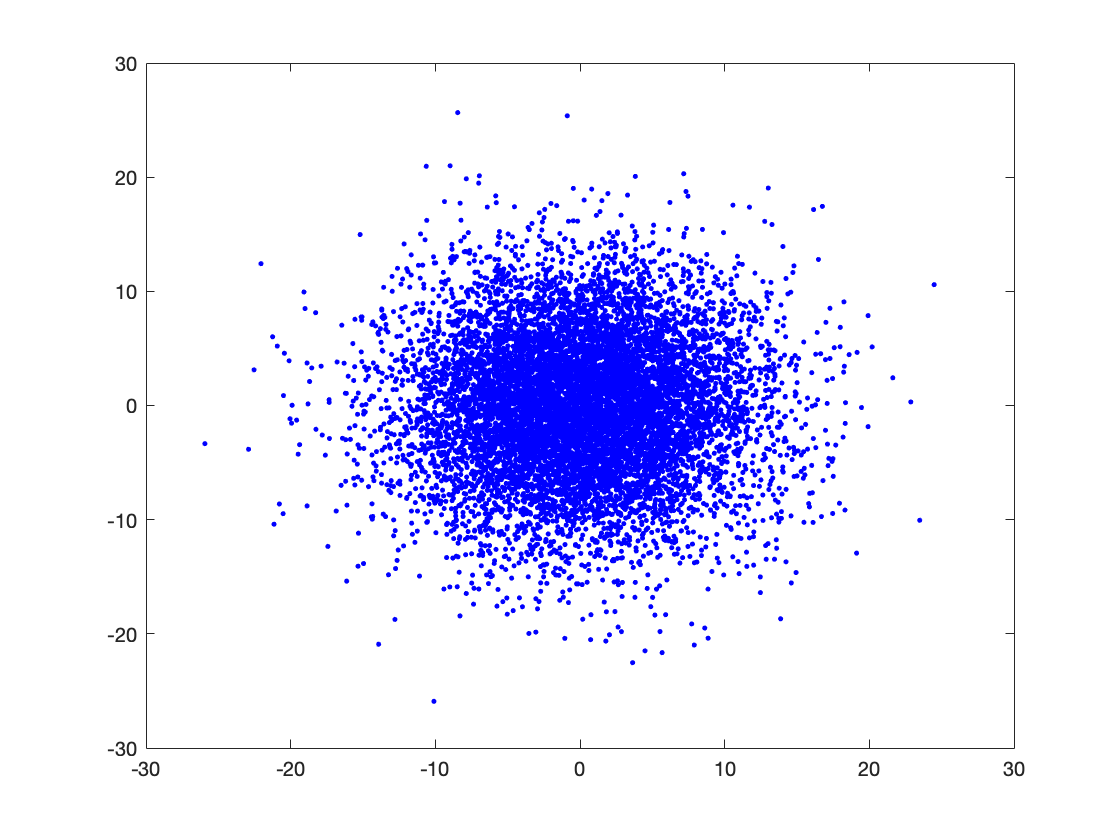} 
\hspace{-0.2cm} & \hspace{-0.2cm} \includegraphics[width=.2\textwidth]{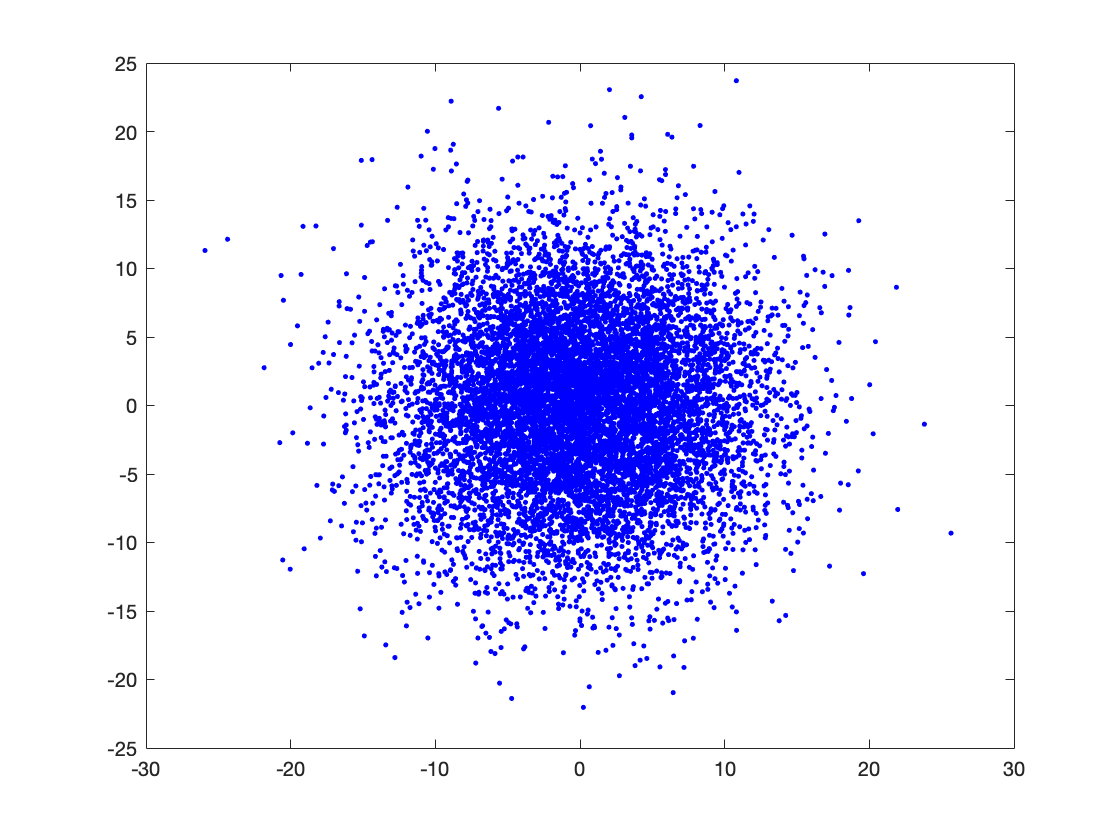}
\hspace{-0.2cm} & \hspace{-0.2cm}
\includegraphics[width=.2\textwidth]{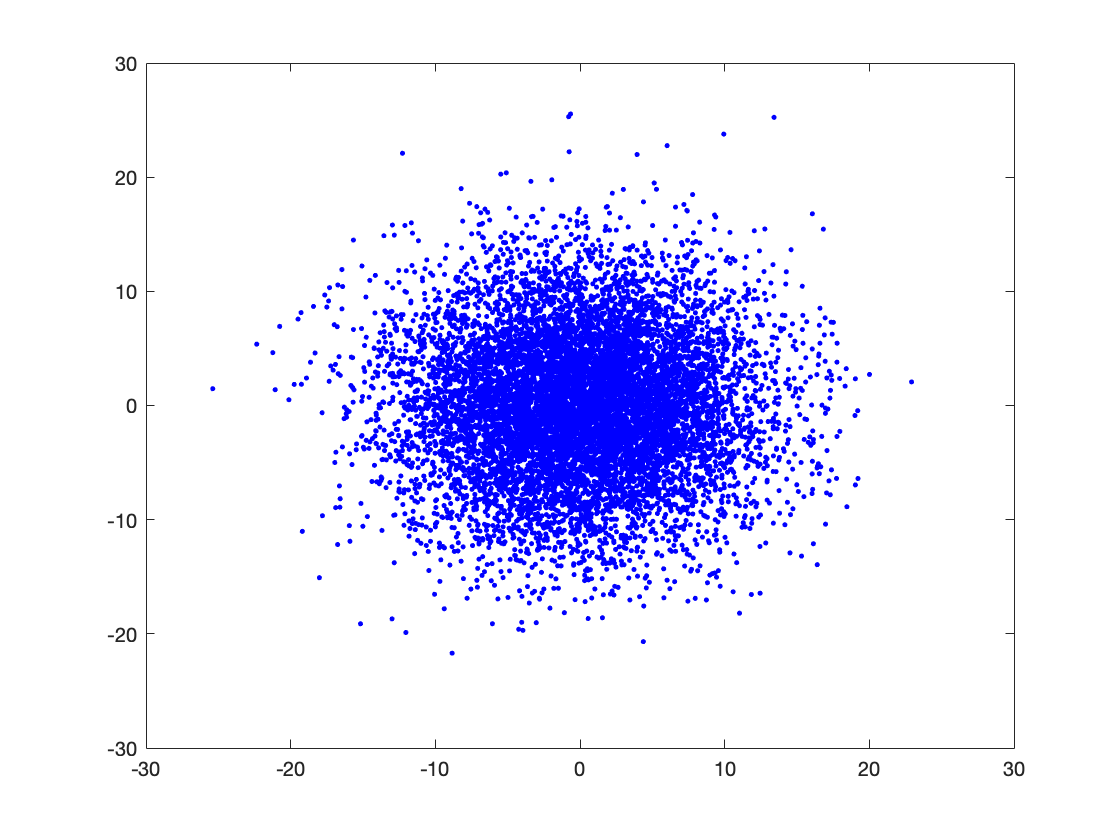}\\
    \footnotesize(i) Iteration 20, Repeat 1 &\footnotesize(j) Iteration 20, Repeat 2  &\footnotesize(k) Iteration 20, Repeat 3 &\footnotesize(l) Iteration 20, Repeat 4\\
    \includegraphics[width=.2\textwidth]{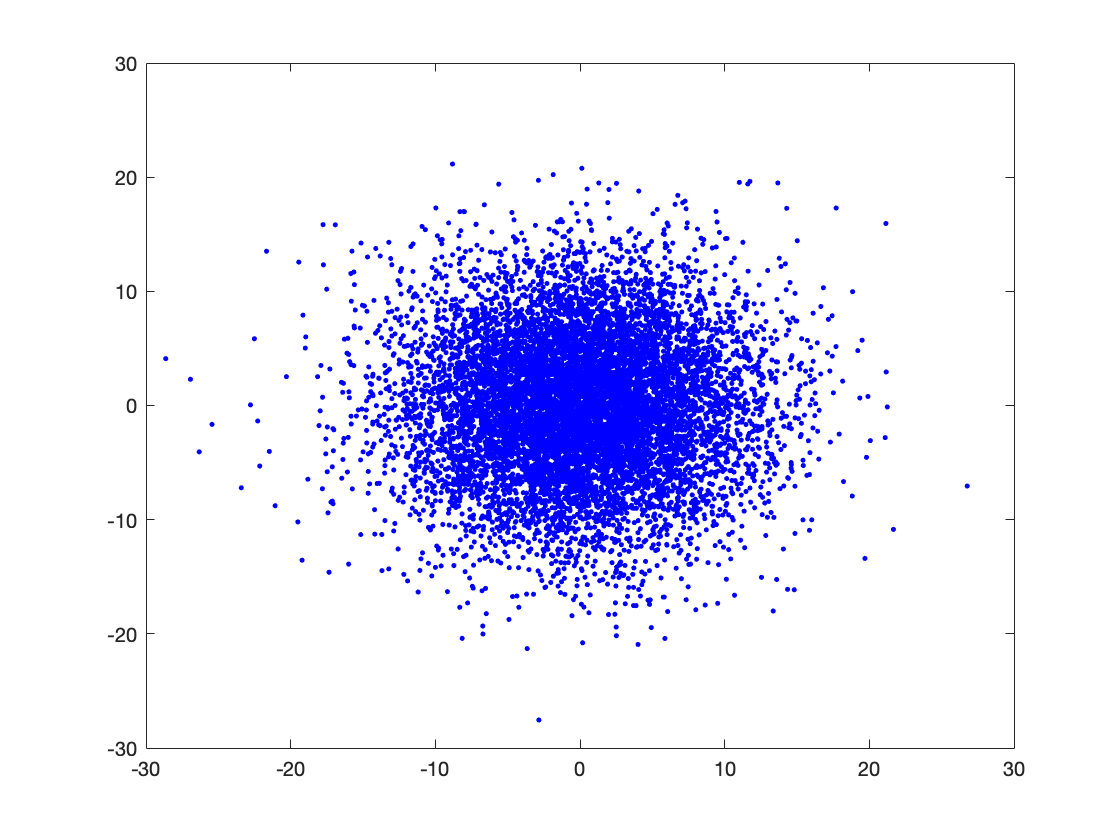} 
\hspace{-0.2cm} & \hspace{-0.2cm}
\includegraphics[width=.2\textwidth]{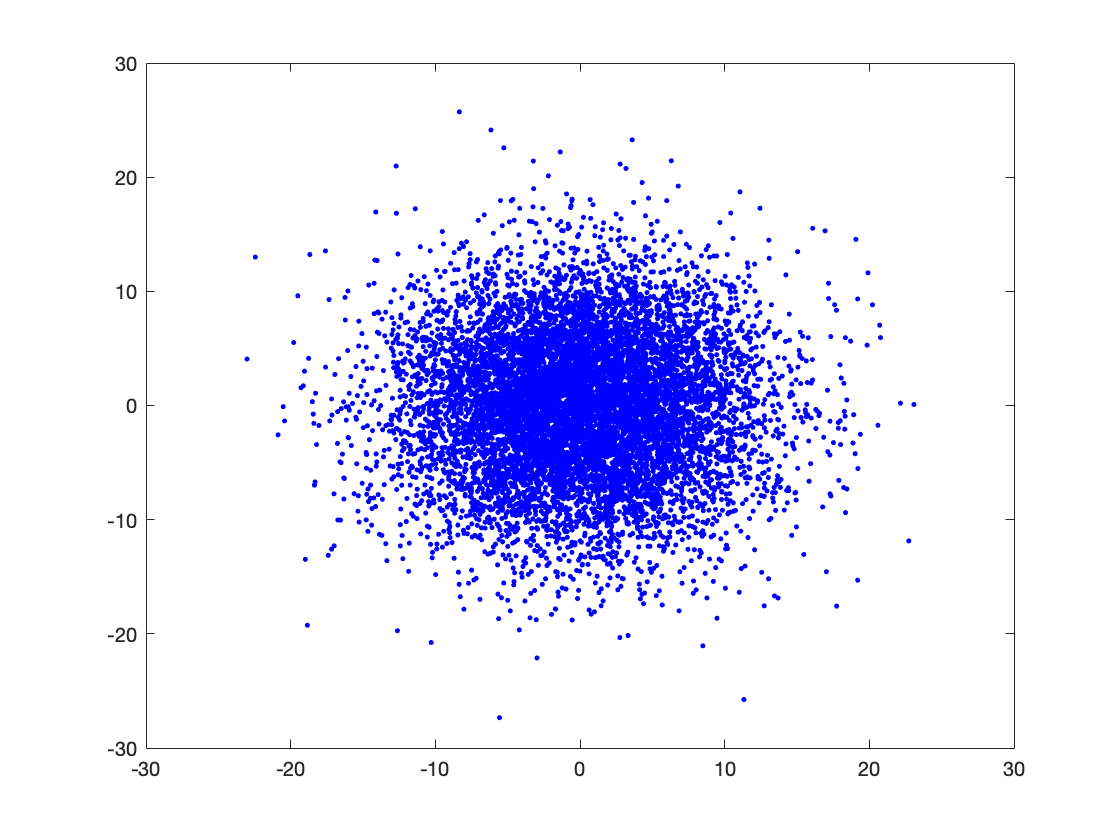} 
\hspace{-0.2cm} & \hspace{-0.2cm} \includegraphics[width=.2\textwidth]{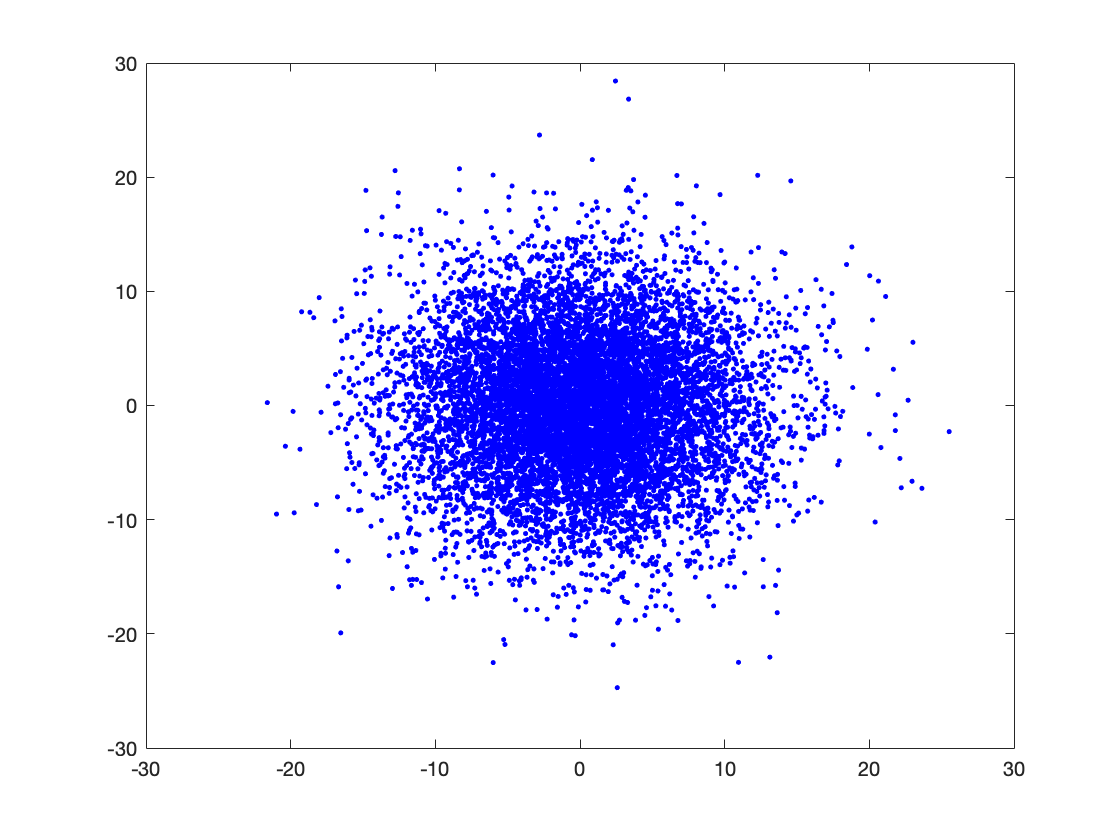}
\hspace{-0.2cm} & \hspace{-0.2cm}
\includegraphics[width=.2\textwidth]{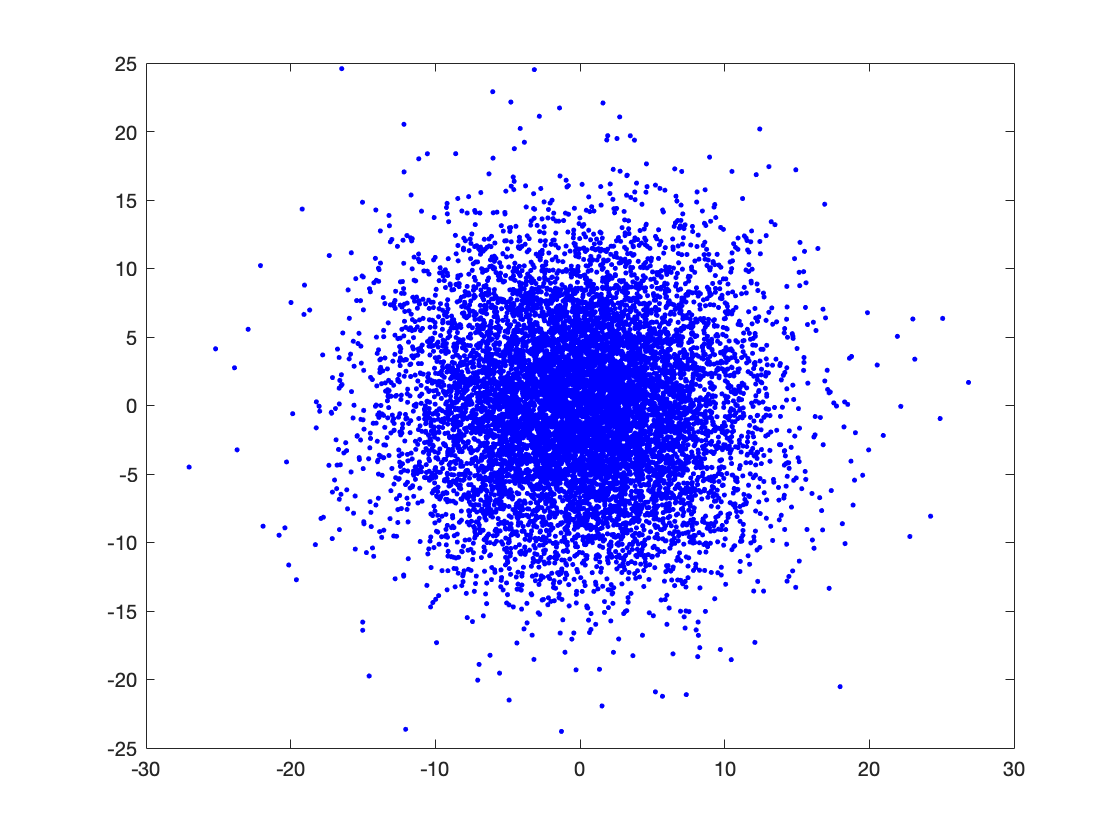}\\
    \footnotesize(m) Iteration 20, Repeat 5 &\footnotesize(n) Iteration 20, Repeat 6  &\footnotesize(o) Iteration 20, Repeat 7 &\footnotesize(p) Iteration 20, Repeat 8\\
    
    \includegraphics[width=.2\textwidth]{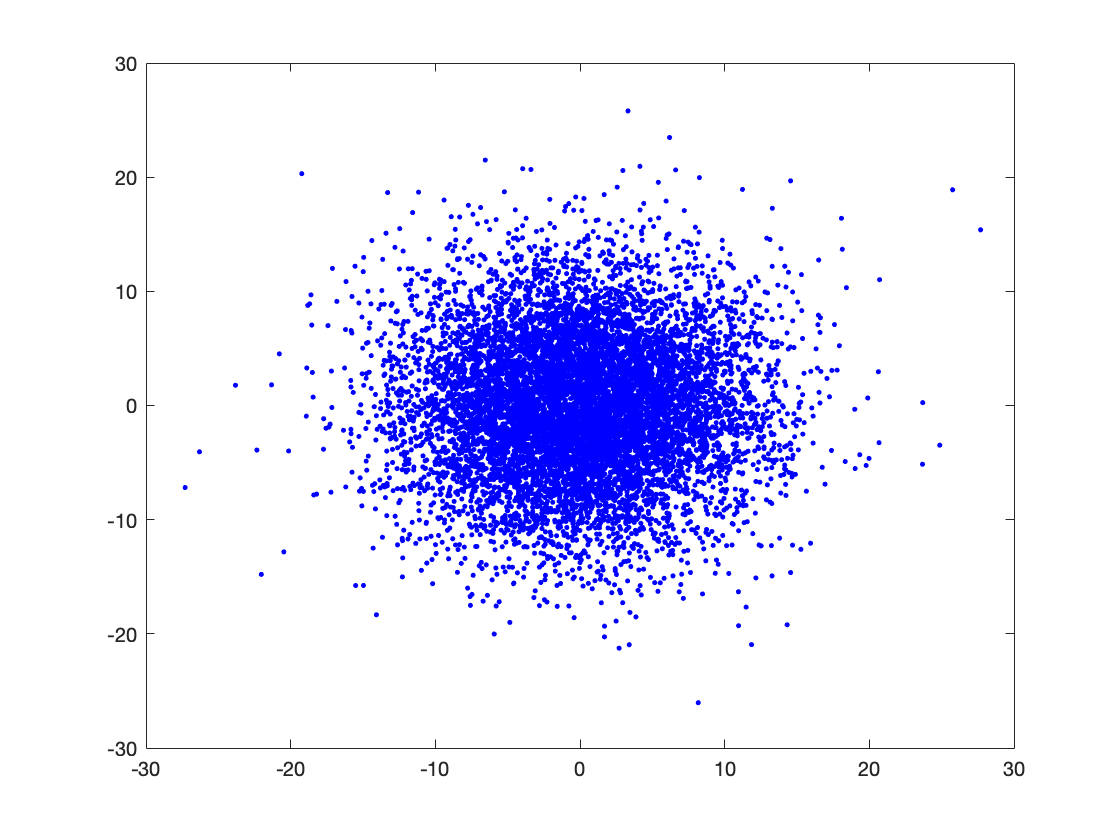} 
\hspace{-0.2cm} & \hspace{-0.2cm}
\includegraphics[width=.2\textwidth]{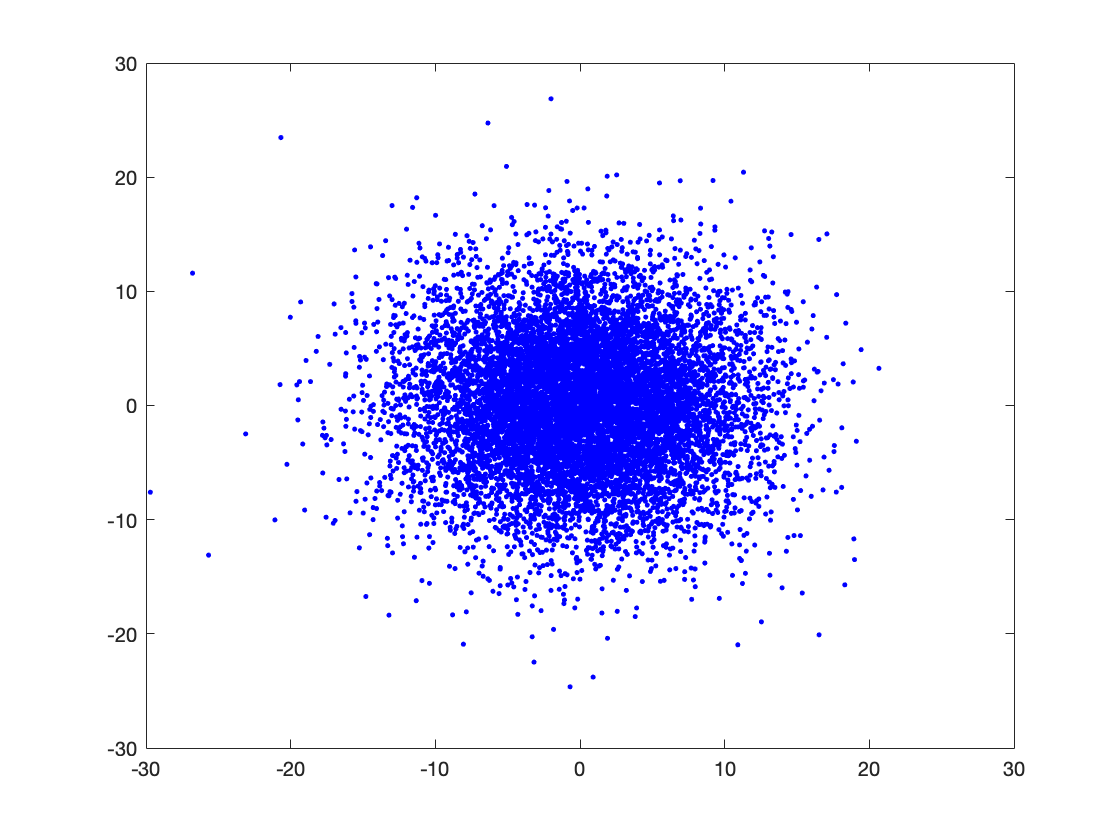} 
\hspace{-0.2cm} & \hspace{-0.2cm} \includegraphics[width=.2\textwidth]{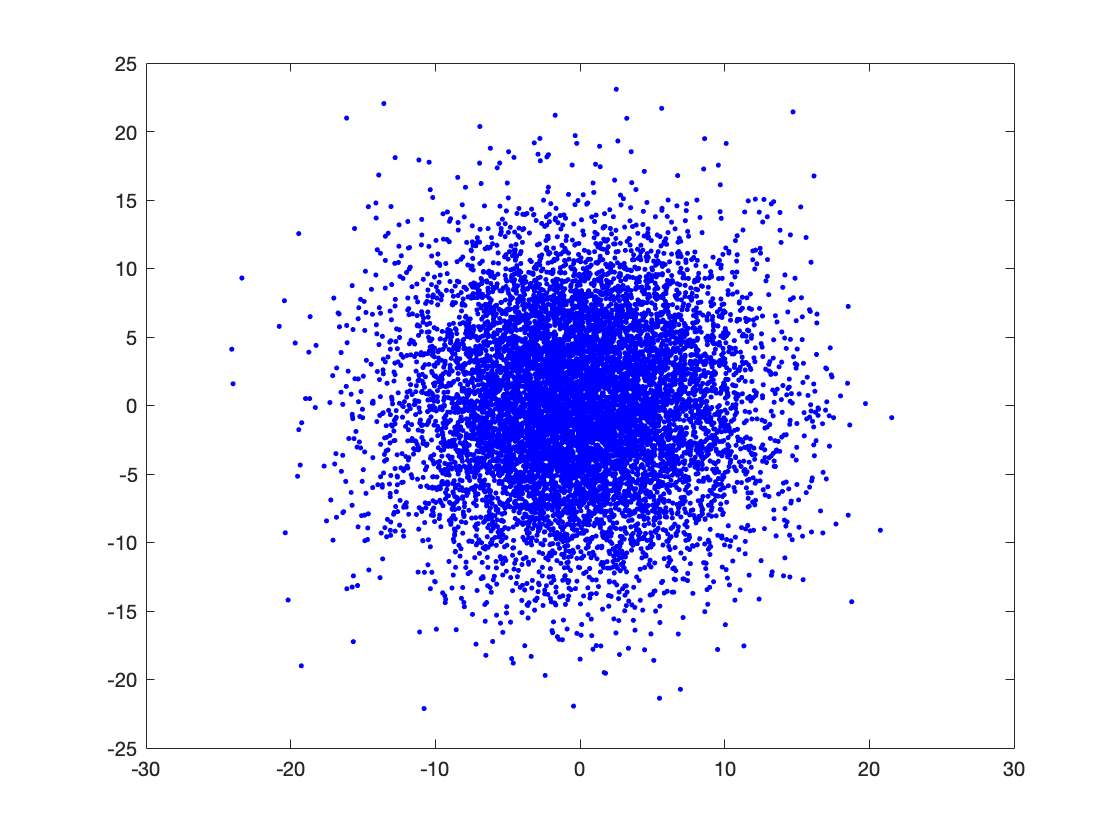}
\hspace{-0.2cm} & \hspace{-0.2cm}
\includegraphics[width=.2\textwidth]{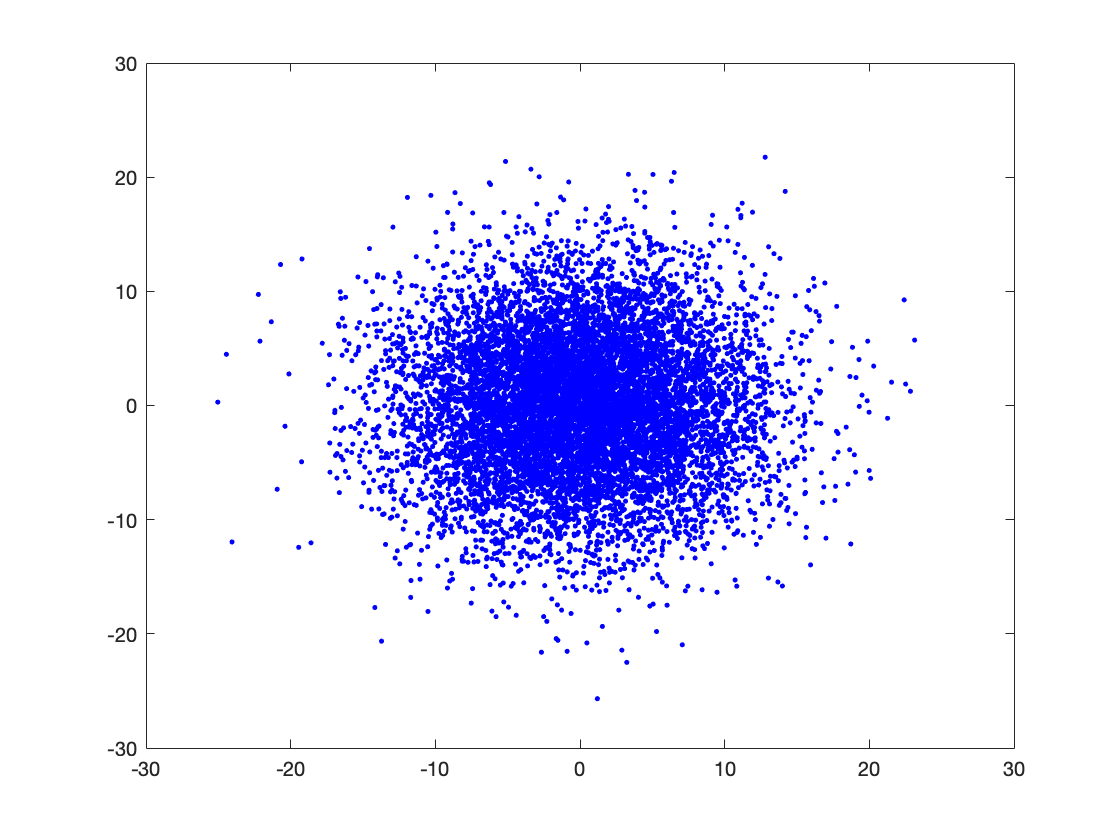}\\
    \footnotesize(q) Iteration 100, Repeat 1 &\footnotesize(r) Iteration 100, Repeat 2  &\footnotesize(s) Iteration 100, Repeat 3 &\footnotesize(t) Iteration 100, Repeat 4\\
    \includegraphics[width=.2\textwidth]{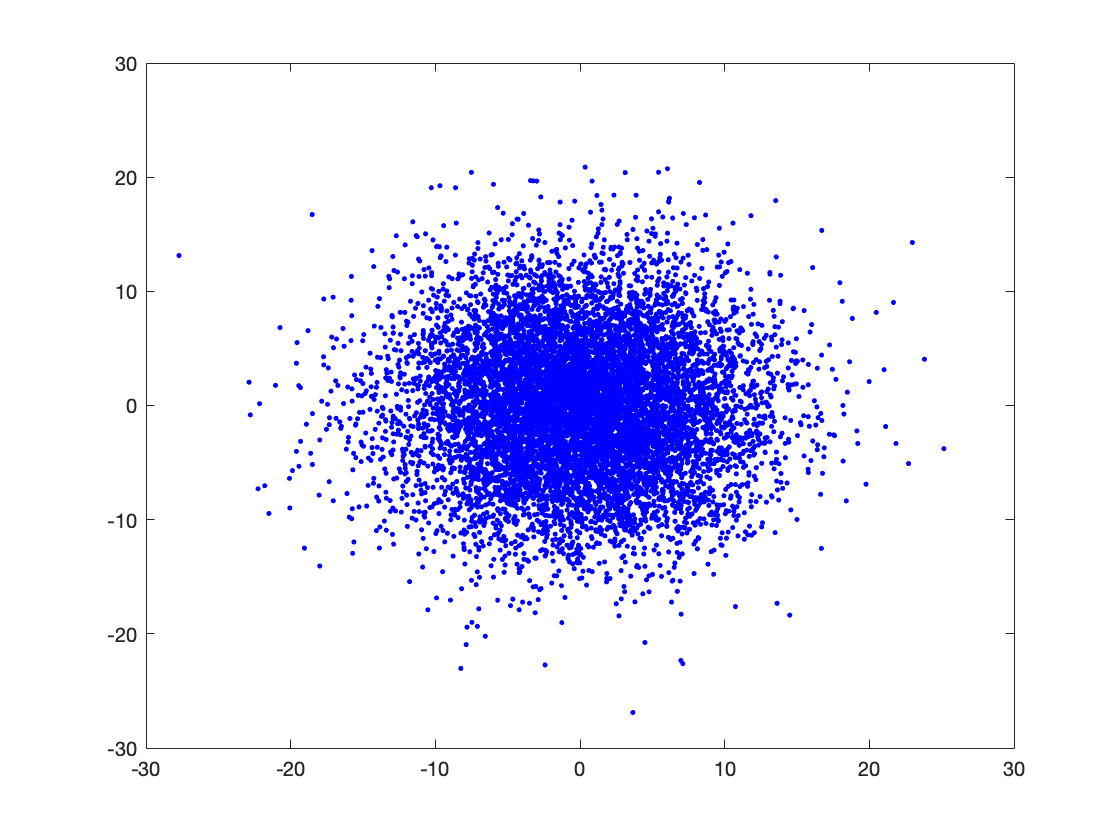} 
\hspace{-0.2cm} & \hspace{-0.2cm}
\includegraphics[width=.2\textwidth]{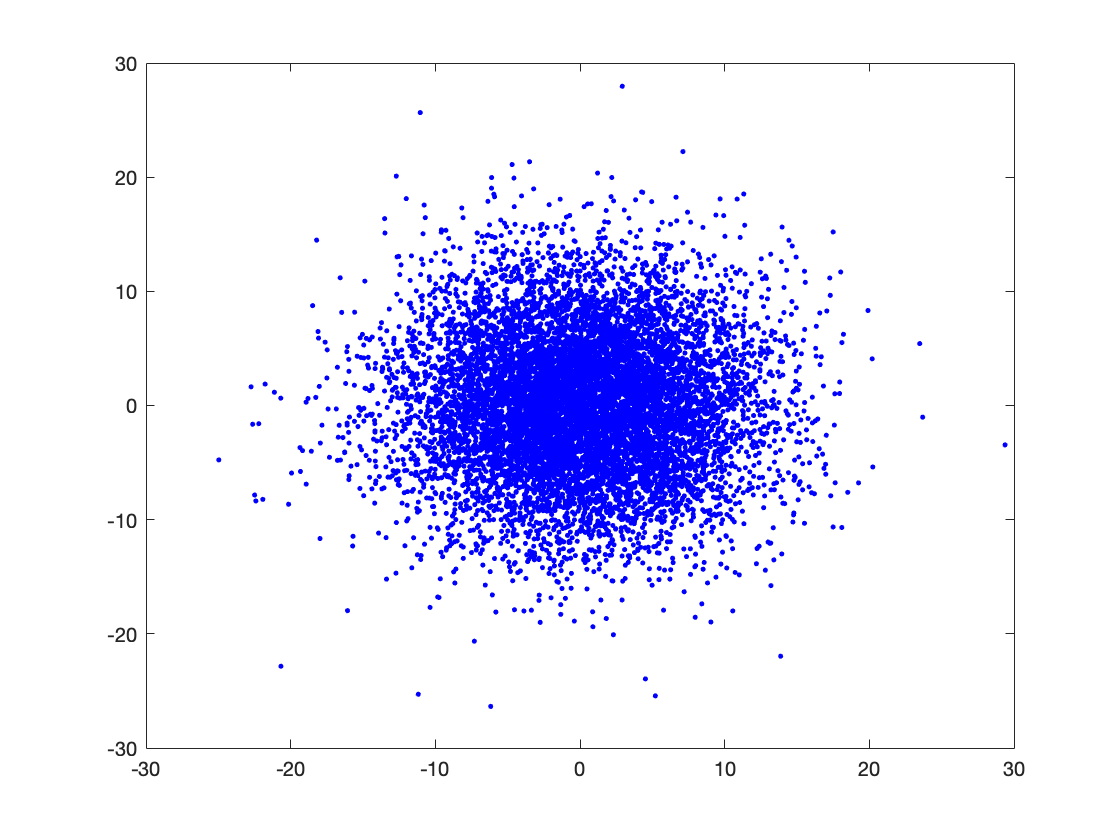} 
\hspace{-0.2cm} & \hspace{-0.2cm} \includegraphics[width=.2\textwidth]{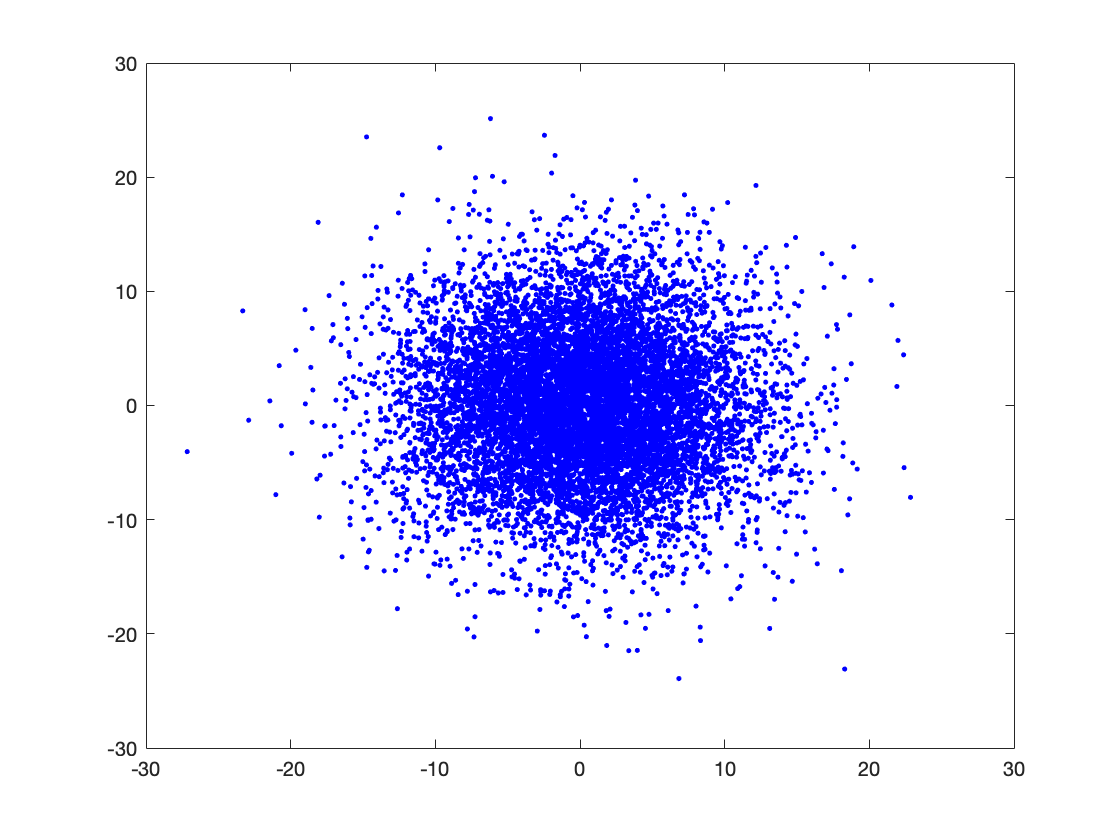}
\hspace{-0.2cm} & \hspace{-0.2cm}
\includegraphics[width=.2\textwidth]{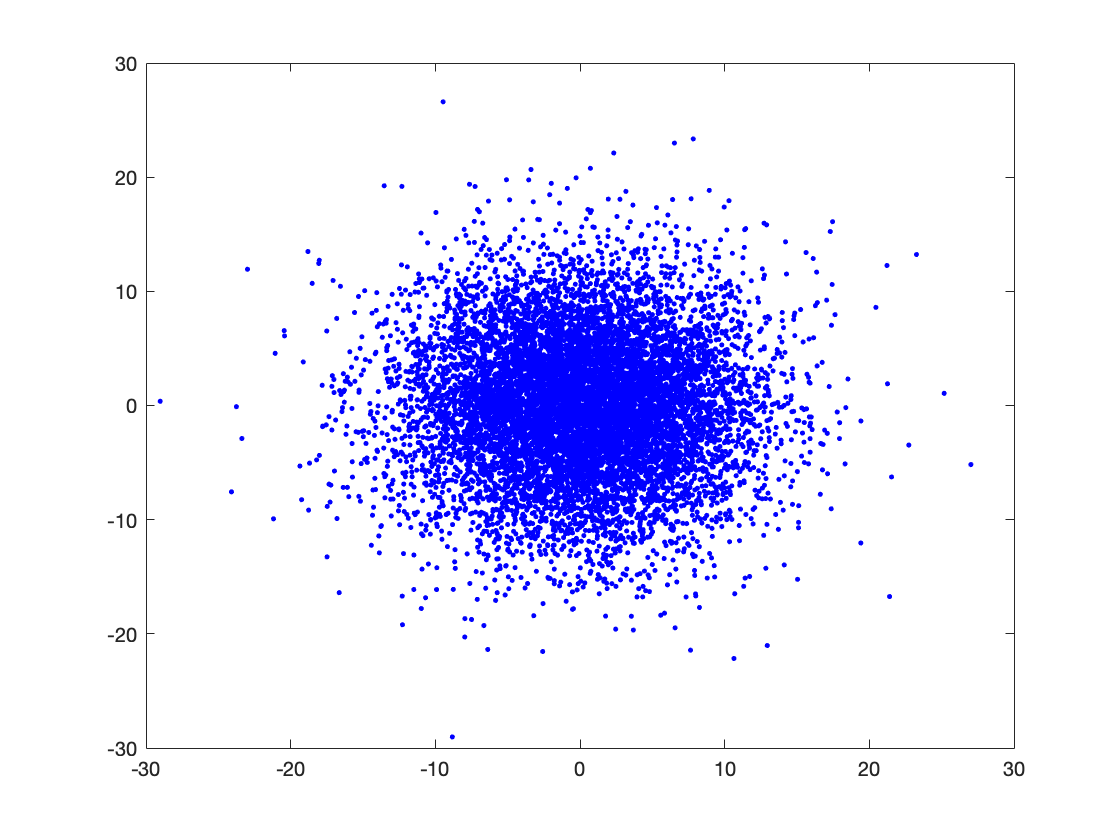}\\
    \footnotesize(u) Iteration 100, Repeat 5 &\footnotesize(v) Iteration 100, Repeat 6  &\footnotesize(w) Iteration 100, Repeat 7 &\footnotesize(x) Iteration 100, Repeat 8\\
    \end{tabular}
    \caption{Distribution of $u_t\in A_t$ projected using 24 random matrices for iterations 2, 20 and 100 of LAZOb in LQR example. }
    \label{fig:symm_appendix}
 \vspace{-0.2cm}
\end{figure*}

\subsection{Justification of asymmetric of $\tilde{A}_t$ for LAZOa}
In Figure \ref{fig:symm_appendix2}, we put similar symmetricity results of $\tilde{A}_t$ for LAZOa in LQR example. We can see that for iteration 2, 20, and 100, all of 24 random Gaussian projection are also almost symmetric, which can further verify the asymptotic symmetricity assumption for LAZOa. 

\subsection{Justification of asymmetric of $A_t$ for LAZOb}
In Figure \ref{fig:symm_appendix}, we put more symmetricity results of $A_t$ for LAZOb in LQR example. We can see that for iteration 2, 20, and 100, all of 24 random Gaussian projection are almost symmetric, which can further verify the asymptotic symmetricity assumption in Assumption \ref{as_sym1}. 

\subsection{Estimate of asymmetric probability}
We also report the estimate of asymmetric probability 
\begin{align*}
    &\hat p_t^a=\frac{1}{N}\sum_{i=1}^N \mathbf{1}\{u_{t}^i\in \tilde A_t\},~~~\hat p_t^b=\frac{1}{N}\sum_{i=1}^N \mathbf{1}\{u_{t}^i\in A_t\}
\end{align*}

in Table \ref{fig:p}, which decays with $t$. We can see that the asymmetric rate estimator for LAZOa decays quicker while both LAZOa and LAZOb are almost symmetric at around $1000$ iteration. Considering the total iteration number is $10000$, the decaying speed for LAZOb is reasonable. 

\begin{table}[htb]
	\centering
		\caption{Estimate $\hat p_t^{\red{a}/\blue{b}}$ of \red{$\mathbb{P}(\bar{A}_t\backslash \bar {A}_t^s)$}~/~\blue{$\mathbb{P}(A_t\backslash A_t^s)$} with $N=10^4$. }
	\begin{tabular}{ccccc}
		\toprule  
		 Iteration $t$ &\red{2}/\blue{2}&\red{10}/\blue{100}&\red{50}/\blue{200}&\red{100}/\blue{1000} \\ 
		 \midrule
	$\hat p_t^{\red{a}/\blue{b}}$ for LAZO\red{a}/\blue{b} &\!\!\red{0.02}/\blue{0.26}\!\!&\!\!\red{0.01}/\blue{0.15}\!\!&\!\!\red{0.01}/\blue{0.04}\!\!&\!\!\red{0.00}/\blue{0.02} \!\!\\
		\bottomrule  
	\end{tabular}
	\label{fig:p}
\end{table}

\subsection{Example in Figure \ref{fig:intuition}: online linear regression}\label{sec:intuition}
In this subsection, we give additional insights on why LAZOa may work in practice. 

We consider an online linear regression model with noisy loss function evaluation $f_t(x)=\frac{1}{2p}\|y-\theta x\|^2+z_t$, where $y\in\mathbb{R}^{p}, \theta=[\theta_{1},\cdots, \theta_{d}]\in\mathbb{R}^{p\times d}$, $x\in\mathbb{R}^{d}$ and $z_t\in\mathbb{R}$. Here we set $p=100$, $d=2$ and generate $y$, $\theta$ and $z_t$ by
\begin{align*}
\theta_{1}=\mathbf{1}_p, \quad \theta_{2}\sim\mathcal{N}^p(0,2^2), \quad y=4+3\theta_{2}+s,\quad s\sim\mathcal{N}^p(0,1), \quad z_t\sim\mathcal{N}(0,1).  
\end{align*}
where $\mathcal{N}^p(\mu,\sigma^2)$ denotes the i.i.d. Gaussian distribution with mean value $\mu$ and standard deviation $\sigma$ in dimension $p$; and $\mathbf{1}_p$ denotes the vector with all $1$ in dimension $p$. 

In this simulation, we choose $\eta=1e^{-7}$ and $\delta=0.1$ for both one-point residual (\blue{blue}) and two-point (\red{red}) methods. We repeat the experiment for $50$ random trials and 
get the temporal variation v.s. gradient estimation error in Figure \ref{fig:intuition}, where the gradient estimation error is computed by $\|\tilde{g}_t(x_t)-\frac{1}{p}\theta^T(y-\theta x_t)\|$. 

In Figure \ref{fig:intuition}, we study how temporal variation $D_t^a(w_t,w_{t-1})$ affects the gradient estimation error (i.e. $\|\tilde{g}_t(x_t)-\nabla f_t(x_t)\|$) under the one-point residual and two-point gradient estimators. 
In Figure \ref{fig:intuition}(a), the gradient estimation error is generally positive correlated to the temporal variation. With the same temporal variation, the two-point estimator has a smaller error than the one-point residual method. However, when zooming into the regime where the temporal variation is small (see Figures \ref{fig:intuition}(b) and \ref{fig:intuition}(c)), we observe that for the same temporal variation,
the gradient estimation error of one-point residual ZO is no longer larger than that of two-point ZO method. Thus in this regime (below the orange line), querying the second new point does not bring much innovation.

\subsection{Simulation setting of non-stationary LQR control}
We study a non-stationary version of the classic LQR  problem \cite{fazel2018global} with the time-varying dynamics. 
At iteration $t$, consider the linear dynamic system described by the dynamic $x_{k+1}=A_tx_k+B_tq_k$, where $x_k\in \mathbb{R}^n$ is the state, $q_k\in\mathbb{R}^p$ is the control variable at step $k$, $A_t\in \mathbb{R}^{n\times n}$ and $B_t\in \mathbb{R}^{n\times p}$ are the dynamic matrices for iteration $t$. Our goal is to minimize the cost which is a fixed quadratic function of state and control given by
\begin{align*}
\begin{array}{ll}
\min_{u_k} & \mathbb{E}\left[\frac{1}{H} \sum_{k=1}^{H} \beta^k(x_{k}^{\top} Q x_{k}+q_{k}^{\top} R q_{k})\right] 
\end{array}
\end{align*}
where $\beta\in(0, 1)$ is a discount factor, $Q\in\mathbb{R}^{n\times n}$ and $R\in\mathbb{R}^{p\times p}$ are the positive definite matrices, and $H$ is the length of step horizon. We search the control $q_k=K_t^*x_k$ that linearly depends on the current state $x_k$, where $K_t^*\in \mathbb{R}^{p\times n}$ is the optimal policy in iteration $t$. Thus our optimization variable will be $K$ and the loss function will be 
\begin{align*}
\begin{array}{ll}\min_{K} & f_t(K)=\mathbb{E}\left[\frac{1}{H} \sum_{k=1}^{H} \beta^k x_{k}^{\top} (Q+K^{\top} RK) x_{k}\right]\\ \text {s.t. } & x_{k+1}=(A_t+B_t K)x_k. \end{array}
\end{align*}

We set $n=p=6$, $\beta=0.5$, $\delta=0.01$, and stepsize $\eta=10^{-5}$ for all four methods. We generate $A_t$, $B_t$ to mimic the situations where the loss functions encounter intermittent changes, given by
\begin{equation*}
\!A_t\!=\!\left\{\begin{array}{lll}s_t, & t\!\!\!\!\!\mod\! 100=0\\A_{t-1}+s_t+7\sin[7(\tiny{t\!\!\!\!\!\mod\! 100})], &t\!\!\!\!\!\mod\! 100\in[35, 65] \\ A_{t-1}+s_t, &\text{else} \end{array}\right.
\label{A_t}
\end{equation*}
where the noise $s_t\sim\mathcal{N}(0,1)$ if $t\!\!\mod\! 100=0$, and $s_t\sim\mathcal{N}(0,0.1^2)$, otherwise. We generate $B_t$ the same way as $A_t$ but replacing the $\sin$ function with the $\cos$ function.

\subsection{Simulation setting of non-stationary resource allocation}
We consider a resource allocation problem with $16$ agents connected by a ring graph. For each iteration $t$, per step $k$, each node $i$ receives an exogenous data request $b_k^i=\psi_{i} \sin (\omega_{i} k+\phi_{i})$, stores $y_k^i$ amount of resources and forwards $a_k^{ij}$ fraction of resources to its neighbor node $j\in \mathcal{N}_i$. Then the aggregate (endogenous plus exogenous) workload of each node $i$ evolves by $$y_{k+1}^i=y_{k}^i-\sum_{j \in \mathcal{N}_{i}} a_k^{i j} y_k^i+\sum_{j \in \mathcal{N}_{i}} a_k^{ji} y_k^j-b_k^j.$$
Per iteration $t$, for each node $i$, at each step $k$, the power cost $r_{k,t}^i$ depends on a varying parameter $p_t^j$ as
\begin{equation*}
r_{k,t}^i=\left\{\begin{array}{ll}0, & \text{If }~~ y_k^i\geq 0 \\ p_t^i (y_k^i)^2, & \text{else}. \end{array}\right.
\label{eq:RA}
\end{equation*}
Defining $x_k^i=[y_k^i,b_k^i]^T$ and the policy $\pi_t^i(x_k^i,\theta_t):x_k^i\rightarrow [0,1]^{|\mathcal{N}_i|}$ at iteration $t$, our goal is to find the optimal policy to allocate $a_k^{ij}$, and thus to minimize the instantaneous accumulated cost $f_t(\theta_t)=\sum_{i=1}^{16}\sum_{k=1}^H \beta^k r_{k,t}^i$, where $H$ is the time step length and $\beta$ is the discount factor. The time-varying parameter $p_t^j$ is generated according to $p_t^j=\sin(\frac{\pi t}{12})+s_t^j$ where $s_t^j$ is uniformly distributed over $[0,1]$.

\subsection{Simulation setting of generation of adversarial examples}
We study generating adversarial examples from a image classifier given by a black-box DNN on the MNIST dataset. The DNN model is seen as the zeroth-order oracle. Let $(y_0,l_0)$ be the image $y_0$ with true label $l_0\in\{1,\cdots n\}$ in $n$ different classes. Assume the target DNN classifier $H(y)=\left[H_1(y),\cdots H_n(y)\right]$ is a well-trained classifier, where $H_l(y)$ means the probability of $y$ being class $l$. Given $H$, an adversarial examples $y$ of $y_0$ means that it is visually similar to $y_0$ but $H$ gives a different prediction class to it. Since the pixel value range of images is always bounded, without loose of generality, we can assume $y\in[-0.5,0.5]^d$. Since the black-box attack is nonconvex stochastic problem, which we only have access to solve the unconstrained setting, we need to apply the $\tanh/2$ transformation to an unbounded variable $x\in\mathbb{R}^d$ to represent $y$. Then we can adopt the black-box attacking loss function defined in \cite{chen2017zoo}, which is given by
\begin{equation}
\underset{\mathbf{x} \in \mathbb{R}^{d}}{\min} \beta \cdot \max \left\{\log H_{l_{0}}(\tanh (\mathbf{x}) / 2)-\max _{l \neq l_{0}} \log H_{l}(\tanh (\mathbf{x}) / 2), 0\right\}+\left\|\tanh (\mathbf{x}) / 2-\mathbf{y}_{0}\right\|_{2}^{2}
\end{equation}
where the first term represents maximum difference between probability of being classified to the true class $l_0$ and the most possible predicted class other than $l_0$, the second term is the $l_2$ distortion, and $\beta$ is the penalty parameter. Here we choose $\beta=0.5$.

\subsection{Parameter tuning details}
Our general procedure is to use grid search to find the best step sizes $\eta$ and delta sequences $\delta$ for the one-point residual and two-point methods to optimize the loss versus iteration plot. Then we use these values in the LAZO method, and use grid search to find the best parameter $D$ in LAZO. 

\begin{table}[htbp]
	\centering
		\caption{Search grid of parameters for LQR control and resource allocation tasks}
	\begin{tabular}{ccc}
		\toprule  
		 &LQR control&Resource allocation \\ 
		 \midrule
	$\eta$  &$\{10^{-4},10^{-5},5\times10^{-6},10^{-6}\}$&$\{10^{-4},10^{-5},5\times10^{-6},10^{-6}\}$ \\
			 \midrule
	$\delta$  &$\{0.1,0.05, 0.01,0.005, 0.001\}$ & $\{1, 0.5, 0.1,0.05, 0.01\}$\\
		\bottomrule  
	\end{tabular}
\end{table}
In the simulations we find that the best values for the one-point residual and two-point methods are the same. Thus, we apply the best values $\eta=10^{-5}$ and $\delta=0.01$ for LQR control and $\eta=10^{-5}$ and $\delta=0.1$ for resource allocation to LAZOa and LAZOb and then use grid search to find the best parameter $D$ from $\{0.1, 0.5, 1, 10, 50\}$. 

For black-box adversarial attack, the search grids for $\eta$ and $\delta$ of the one-point residual and two-point methods are mentioned in Section 4.3. We find the optimal choice for them to minimize the iterations needed for the first successful attack (see Section 4.3 for its definition). Similar to the procedure for LQR control and resource allocation, we apply the optimal $\delta=0.01$ for the two-point method to LAZOa/b. However, the optimal $\eta=3$ for the two-point method cannot be applied to LAZOa/b in the black-box adversarial attack task. The reason is that LAZOa/b are significantly damped by the one-point steps in this task. Only by setting $D$ small enough, they can avoid exploding but will sacrifice in query saving, or even degenerate to the two-point method. Thus, we set smaller $\eta$ for LAZOa/b and then tune $D$ to optimize the iterations needed for the first successful attack.

\begin{table}[htbp]
	\centering
	\caption{Runtime for one-point residual method, two-point method and LAZOa/b in LQR control}
	\begin{tabular}{ccccc}
		\toprule  
		 Method&one-point residual&two-point&LAZOa&LAZOb \\ 
		 \midrule
	Time (sec)  &1548&1653&1651&1640 \\
		\bottomrule  
	\end{tabular}
\end{table}

\subsection{Computational overhead}\label{comp}
Regarding the overhead, the adaptive rules of LAZOa/b actually rely on simple calculations, and only have little influence on the computation. In fact, we observed that LAZOa/b can even slightly save runtime compared to the two-point ZO method because querying function value usually takes more time than computing temporal variation. As an example, in Table 2, we report the runtime for running the four methods over *the same number of iterations* in LQR control.

\begin{algorithm}[tb]
\setstretch{1.4}
   \caption{Multi-point LAZO: Multiple point Lazy query for ZO gradient method: \colorbox{red!30}{red part} is run only by \red{\bf LAZOa}; \colorbox{blue!30}{blue part} is implemented only by \blue{\bf LAZOb}; not both at the same time.}
   \label{alg:LAZO_multi}
\begin{algorithmic}[1]
   \State {\bfseries Input:} $x_0\in \mathbb{R}^d$; $H,K,T, \delta,\eta, D>0$.
   \For{$t=1$ {\bfseries to} $H$}
   \State Sample $u_{t}^1,\cdots, u_t^K\sim U(\mathbb{S})$ independently.
   \State Query $f_t(w_t^1),f_t(x_t-\delta u_t^1),\cdots,f_t(w_t^K),f_t(x_t-\delta u_t^K)$.~~~~~~~~~~~~~~~~~~~~\Comment{$ w_t^k=x_t+\delta u_t^k$}
   \State Compute $\tilde{g}_{t}(x_{t})=\frac{d}{2\delta K}\sum_{k=1}^K(f_t(w_t^k)-f_{t}(x_{t}-\delta u_t^k))$
   \State Update $x_{t+1}=\Pi_{\mathcal{X}}\left(x_{t}-\eta \tilde{g}_{t}(x_{t})\right)$.
   \EndFor

   \For{$t=H+1$ {\bfseries to} $T$}
   \State Set $k=0,\tilde{g}_t(x_t)=0$.
   \While{$k\leq K$}
   \State Sample $u_t^k\sim U(\mathbb{S})$. 
   \State Query $f_t(w_t^k)$. 
   \State Set $m=0$.
   \For{$\tau=1$ {\bfseries to} $H$} 
   \For{$j=1$ {\bfseries to} $K$}
   \If{\colorbox{red!30}{$D_t^a(w_t^k,w_{t-\tau}^j)$} or \colorbox{blue!30}{$D_t^b(w_t^k,w_{t-\tau}^j)\!\leq\! D$}} ~~~~~~~~~~~\Comment{Exert old queries}
   \State Compute $\tilde{g}_t(x_t)=\tilde{g}_t(x_t)+\frac{du_t^k}{\delta}\left(f_t(w_t)-f_{t-\tau}(w_{t-\tau}^j)\right)$.
   \State Set $m=m+1,k=k+1$ and $u_t^k=u_t^{k-1}$.
   \EndIf 
   \EndFor
   \EndFor
   \If{$m=0$}~~~~~~~~~~~~~~~~~~~~~~~~~~~~~~~~~~~~~~~~~~\Comment{No informative old queries, query new point}
   \State Query $f_{t}(x_{t}-\delta u_t^k)$. 
   \State Compute $\tilde{g}_t(x_t)=\tilde{g}_t(x_t)+\frac{du_t^k}{2\delta}\left(f_t(w_t^k)-f_{t}(x_{t}-\delta u_t^k)\right)$. 
   \State Set $k=k+1$. 
   \EndIf
   \EndWhile
   \State Compute $\tilde{g}_t(x_t)=\tilde{g}_t(x_t)/K$ and update $x_{t+1}=\Pi_{\mathcal{X}}\left(x_{t}-\eta \tilde{g}_{t}(x_{t})\right)$.
   \EndFor
   \State {\bfseries Output:} $x_{t+1}$. 
\end{algorithmic}
\end{algorithm}
\section{Extension to multi-point query rules}\label{exten_app}
In previous analysis, we only consider reusing one previous step. In this section, we aim to enlarge the reusing horizon from $1$ to $H>1$. 

\subsection{Algorithm development}
Correspondingly, the temporal variation in Definition \ref{df:temp} need to be extended to multiple previous horizons case.

\begin{defi}[Temporal variation]
    \label{df:temp_ex}
  For  $\forall x,y\in\mathcal{X},\tau\in N$, define the two temporal variation between $t$ and $t-\tau$ as
    \begin{align}\label{eq.df:temp_ex}
&D_{t,\tau}^a(x,y)\triangleq\frac{|f_t(x)-f_{t-\tau}(y)|}{\|x-y\|}, ~~~~~D_{t,\tau}^b(x,y)\triangleq\frac{|f_t(x)-f_{t-\tau}(y)|}{\eta L}.
    \end{align}
\end{defi}
Likewise, the value of two temporal variation in Definition \ref{df:temp_ex} can be served as the indicator of whether the queries  at time $t-\tau$ is informative. 

Moreover, in \cite{shamir2017optimal}, two-point symmetric ZO gradient estimator can be extended to $2K$-points estimator to reduce estimation variance as follows ($K>1$):
\begin{equation}
\label{multi-shamir}
    \tilde{g}_t^{(2K)}(x_t)=\frac{d}{2\delta K}\sum_{k=1}^{K}u_{t}^k(f_t(x_t+\delta u_{t}^k)-f_t(x_t-\delta u_{t}^k))
\end{equation}
where $u_t^k$ are i.i.d randomly sampled from the unit sphere. We can also use the same idea in LAZO to construct multiple points gradient estimator for robustness.

To attain a $2K$-points LAZO estimator using queries in $H>1$ previous steps, at time $t$, we first search on the previous $H$ steps for valuable queries and reuse them, or if no old query is useful, we query a new point. Continuing this procedure until we obtain a $2K$-points estimator. Specifically, we can define $2K$-points \textbf{LAZOa} estimator as follows:
\begin{align}
\tilde{g}_{t}^{a}(x_{t})&=\frac{d}{\delta K}\sum_{k=1}^{K_t^a}u_{t}^k\left\{\mathbf{1}_{\mathcal{T}_{t,k}^a\neq\emptyset}\left(\sum_{(\tau,l)\in\mathcal{T}_{t,k}^a}\!\!\!\left(f_t(w_{t}^k)-f_{t-\tau}(w_{t-\tau}^l)\right)\right)+\frac{\mathbf{1}_{\mathcal{T}_{t,k}^a=\emptyset}}{2}\left(f_t(w_{t}^k)-f_{t}(x_t-\delta u_{t}^k)\right)\right\}
\label{eq:clipping_ex}
\end{align}
where $w_{t}^k\triangleq x_{t}+\delta u_{t}^k$, $\mathcal{T}_{t,k}^a\triangleq\{(\tau,l)|1\leq\tau\leq H,1\leq l\leq K_{t-\tau}^a, D_{t,\tau}^a(w_{t}^k,w_{t-\tau}^l)\leq D\}$ and $K_t^a$ is the integer such that $\sum_{k=1}^{K_t^a}\left(|\mathcal{T}_{t,k}^a|+\mathbf{1}_{\mathcal{T}_{t,k}^a\neq\emptyset}\right)= K$. The set $\mathcal{T}_{t,K_t^a}^a$ only contains $K-\sum_{k=1}^{K_t^a-1}\left(|\mathcal{T}_{t,k}^a|+\mathbf{1}_{\mathcal{T}_{t,k}^a\neq\emptyset}\right)$ elements if its size is bigger. 

Similarly, the $2K$-points \textbf{LAZOb} gradient estimator as
\begin{align}
\tilde{g}_{t}^{b}(x_{t})&=\frac{d}{\delta K}\sum_{k=1}^{K_t^b}u_{t}^k\left\{\mathbf{1}_{\mathcal{T}_{t,k}^b\neq\emptyset}\left(\sum_{(\tau,l)\in\mathcal{T}_{t,k}^b}\!\!\!\left(f_t(w_{t}^k)-f_{t-\tau}(w_{t-\tau}^l)\right)\right)+\frac{\mathbf{1}_{\mathcal{T}_{t,k}^b=\emptyset}}{2}\left(f_t(w_{t}^k)-f_{t}(x_t-\delta u_{t}^k)\right)\right\}
\label{eq:lazob_ex}
\end{align}
where $\mathcal{T}_{t,k}^b\triangleq\{(\tau,l)|1\leq\tau\leq H,1\leq l\leq K_{t-\tau}^b, D_{t,\tau}^b(w_{t}^k,w_{t-\tau}^l)\leq D\}$ and $K_t^b$ is the integer such that $\sum_{k=1}^{K_t^b}\left(|\mathcal{T}_{t,k}^b|+\mathbf{1}_{\mathcal{T}_{t,k}^b\neq\emptyset}\right)= K$. Again if the set $\mathcal{T}_{t,K_t^b}^b$ only contains $K-\sum_{k=1}^{K_t^b-1}\left(|\mathcal{T}_{t,k}^b|+\mathbf{1}_{\mathcal{T}_{t,k}^b\neq\emptyset}\right)$ elements if its size is bigger. 

The complete multi-point LAZOa and LAZOb are summarized in Algorithm \ref{alg:LAZO_multi}.

\subsection{Parameters search grid}
We pick the optimal $\eta\in \{0.1, 1,2,\cdots,10\}$, $\delta\in\{0.5,0.1,0.05,0.01\}$ for multi-point LAZOa, LAZOb and multi-point symmetric ZO gradient estimator \cite{shamir2017optimal} and $D\in \{1\times 10^{-5},5\times 10^{-5},1\times 10^{-4},5\times 10^{-4}\}$ for multi-point LAZOa, $DL\in \{10,50,100,500,1000\}$ for multi-point LAZOb in the black-box adversarial attack application. 
\end{document}